\documentclass{article}
  
\usepackage{comment,url,algorithm,algorithmic,graphicx,subcaption,relsize}
\usepackage{amssymb,amsfonts,amsmath,amsthm,amscd,dsfont,mathrsfs,mathtools,nicefrac}
\usepackage{float,psfrag,epsfig,color,xcolor,url,hyperref,wrapfig}
\usepackage{epstopdf,bbm,mathtools,enumitem}
\usepackage[toc,page]{appendix}
\usepackage[mathscr]{euscript}

\usepackage[top=1in, bottom=1in, left=1in, right=1in]{geometry}
\newcommand{\dee}{\mathop{\mathrm{d}\!}}
\newcommand{\dt}{\,\dee t}

\def\balign#1\ealign{\begin{align}#1\end{align}}
\def\baligns#1\ealigns{\begin{align*}#1\end{align*}}
\def\balignat#1\ealign{\begin{alignat}#1\end{alignat}}
\def\balignats#1\ealigns{\begin{alignat*}#1\end{alignat*}}
\def\bitemize#1\eitemize{\begin{itemize}#1\end{itemize}}
\def\benumerate#1\eenumerate{\begin{enumerate}#1\end{enumerate}}

\newenvironment{talign*}
 {\csname align*\endcsname}
 {\endalign}
\newenvironment{talign}
 {\csname align\endcsname}
 {\endalign}

\def\balignst#1\ealignst{\begin{talign*}#1\end{talign*}}
\def\balignt#1\ealignt{\begin{talign}#1\end{talign}}
\let\originalleft\left
\let\originalright\right
\renewcommand{\left}{\mathopen{}\mathclose\bgroup\originalleft}
\renewcommand{\right}{\aftergroup\egroup\originalright}

\def\tinycitep*#1{{\tiny\citep*{#1}}}
\def\tinycitealt*#1{{\tiny\citealt*{#1}}}
\def\tinycite*#1{{\tiny\cite*{#1}}}
\def\smallcitep*#1{{\scriptsize\citep*{#1}}}
\def\smallcitealt*#1{{\scriptsize\citealt*{#1}}}
\def\smallcite*#1{{\scriptsize\cite*{#1}}}

\def\mbb#1{\mathbb{#1}}

\def\mrm#1{\mathrm{#1}}

\newcommand{\norm}[1]{\left\lVert#1\right\rVert}

\theoremstyle{plain}
\newtheorem*{remark}{\textbf{Remark}}

\newcommand{\E}{\mathbb{E}}

\newcommand{\R}{\mathbb{R}}

\newcommand{\boldA}{{\boldsymbol{A}}}
\newcommand{\boldB}{{\boldsymbol{B}}}

\newcommand{\boldF}{{\boldsymbol{F}}}

\newcommand{\boldH}{{\boldsymbol{H}}}
\newcommand{\boldI}{{\boldsymbol{I}}}
\newcommand{\boldJ}{{\boldsymbol{J}}}

\newcommand{\boldP}{{\boldsymbol{P}}}

\newcommand{\boldS}{{\boldsymbol{S}}}

\newcommand{\boldU}{{\boldsymbol{U}}}

\newcommand{\bW}{{\boldsymbol{W}}}
\newcommand{\bX}{{\boldsymbol{X}}}

\newcommand{\bK}{{\boldsymbol{K}}}

\newcommand{\bolde}{{\boldsymbol{e}}}

\newcommand{\bx}{{\boldsymbol{x}}}
\newcommand{\by}{{\boldsymbol{y}}}
\newcommand{\bz}{{\boldsymbol{z}}}

\newcommand{\boldtheta}{{\boldsymbol{\theta}}}

\newcommand{\boldTheta}{{\boldsymbol{\Theta}}}

\newcommand{\boldvarepsilon}{{\boldsymbol{\varepsilon}}}

\newcommand{\calE}{{\mathcal{E}}}
\newcommand{\calF}{{\mathcal{F}}}

\newcommand{\calH}{{\mathcal{H}}}

\newcommand{\calN}{{\mathcal{N}}}

\def\R{\mathbb{R}}
\def\<{\left\langle} % Angle brackets
\def\>{\right\rangle}

\def\norm#1{\left\|{#1}\right\|} % A norm with 1 argument
\def\E{\mbb{E}} % Expectation symbol

\newcommand{\Tr}[1]{\mathrm{tr}\LL(#1\RR)}
\def\Cov{\mrm{Cov}} % Covariance symbol

\DeclareSymbolFont{rsfs}{U}{rsfs}{m}{n}
\DeclareSymbolFontAlphabet{\mathscrsfs}{rsfs}
\providecommand{\argmin}{\mathop\mathrm{arg min}}

\providecommand{\diag}{\mathop\mathrm{diag}}
\def\supp#1{\mathrm{supp}({#1})}

\ifdefined\nonewproofenvironments\else
\ifdefined\ispres\else
\newtheorem{theo}{Theorem}
\newtheorem*{theo*}{Theorem}
\newtheorem{lemm}[theo]{Lemma}
\newtheorem{coro}[theo]{Corollary}

\newtheorem{prop}[theo]{Proposition}  
\newtheorem*{prop*}{Proposition}

\renewenvironment{proof}{\noindent\textbf{Proof.}\hspace*{.3em}}{\qed\\}
\newenvironment{proof-sketch}{\noindent\textbf{Proof Sketch}
  \hspace*{1em}}{\qed\bigskip\\}
\newenvironment{proof-idea}{\noindent\textbf{Proof Idea}
  \hspace*{1em}}{\qed\bigskip\\}
\newenvironment{proof-of-lemma}[1][{}]{\noindent\textbf{Proof of Lemma {#1}}
  \hspace*{1em}}{\qed\\}
\newenvironment{proof-of-theorem}[1][{}]{\noindent\textbf{Proof of Theorem {#1}}
  \hspace*{1em}}{\qed\\}
\newenvironment{proof-attempt}{\noindent\textbf{Proof Attempt}
  \hspace*{1em}}{\qed\bigskip\\}

\fi

\fi
\makeatletter
\@addtoreset{equation}{section}
\makeatother
\def\theequation{\thesection.\arabic{equation}}

\hypersetup{
  colorlinks,
  linkcolor={red!50!black},
  citecolor={blue!50!black},
  urlcolor={blue!80!black}
}

\newcommand\numberthis{\addtocounter{equation}{1}\tag{\theequation}}

\mathtoolsset{showonlyrefs}

\usepackage{caption}
\usepackage{authblk}
\usepackage{etoc}
\usepackage{wrapfig}

\newcommand{\bt}{\boldsymbol{\theta}}
\newcommand{\hbt}{\hat{\boldsymbol{\theta}}}
\newcommand{\dpar}[2]{\frac{\partial #1}{\partial #2}}

\newcommand*{\LL}{\left}
\newcommand*{\RR}{\right}
\newcommand{\bP}{\boldP}
\newcommand{\bSigma}{{\boldsymbol{\Sigma}}}
\newcommand{\btheta}{{\boldsymbol{\theta}}}
\newcommand{\bI}{{\boldsymbol{I}}}
\newcommand{\LPiPx}{{L_2(P_X)}}

\newcommand*\samethanks[1][\value{footnote}]{\footnotemark[#1]}

\title{\vspace{-4mm} When Does Preconditioning Help or Hurt Generalization?}
 
\author{\thanks{Alphabetical ordering. Correspondence to: Denny Wu (\texttt{dennywu@cs.toronto.edu}).}
Shun-ichi Amari\thanks{RIKEN Center for Brain Science. \texttt{amari@brain.riken.jp}.} ,\, 
Jimmy Ba\thanks{University of Toronto and Vector Institute for Artificial Intelligence. \texttt{\{jba,rgrosse,dennywu\}@cs.toronto.edu}.} ,\,
Roger Grosse\samethanks[3] ,\,
Xuechen Li\thanks{Google Research, Brain Team. Member of the Google AI Residency Program. \texttt{lxuechen@cs.toronto.edu}.}
\\ 
\!Atsushi Nitanda\thanks{University of Tokyo and RIKEN Center for Advanced Intelligence Project. \texttt{\{nitanda,taiji\}@mist.i.u-tokyo.ac.jp}.} ,\,  
Taiji Suzuki\samethanks[5] ,\, 
Denny Wu\samethanks[3] ,\,
Ji Xu\thanks{Columbia University. \texttt{jixu@cs.columbia.edu}.\vspace{-2mm}}
}

\begin{document}
\etocdepthtag.toc{mtchapter}
\etocsettagdepth{mtchapter}{subsection}
\etocsettagdepth{mtappendix}{none}

\maketitle

\vspace{-3.5mm}  
\begin{abstract}
While second order optimizers such as natural gradient descent (NGD) often speed up optimization, their effect on generalization has been called into question. 
% For instance, it has been pointed out that gradient descent (GD), in contrast to many preconditioned updates, converges to small Euclidean norm solutions in overparameterized models, leading to favorable generalization properties. 
This work presents a more nuanced view on how the \textit{implicit bias} of first- and second-order methods affects the comparison of generalization properties.
We provide an exact asymptotic bias-variance decomposition of the generalization error of overparameterized ridgeless regression under a general class of preconditioner $\boldsymbol{P}$, and consider the inverse population Fisher information matrix (used in NGD) as a particular example. 
We determine the optimal $\boldsymbol{P}$ for both the bias and variance, and find that the relative generalization performance of different optimizers depends on the label noise and the ``shape'' of the signal (true parameters): when the labels are noisy, the model is misspecified, or the signal is misaligned with the features, NGD can achieve lower risk; conversely, GD generalizes better than NGD under clean labels, a well-specified model, or aligned signal. 
Based on this analysis, we discuss several approaches to manage the bias-variance tradeoff, and the potential benefit of interpolating between GD and NGD. We then extend our analysis to regression in the reproducing kernel Hilbert space and demonstrate that preconditioned GD can decrease the population risk faster than GD.
Lastly, we empirically compare the generalization error of first- and second-order optimizers in neural network experiments, and observe robust trends matching our theoretical analysis. 
\end{abstract}

\section{Introduction}
\label{sec:intro}
{

Due to the significant and growing cost of training large-scale machine learning systems (e.g.~neural networks \cite{brown2020language}), there has been much interest in algorithms that speed up optimization.
Many such algorithms make use of various types of second-order information, and can be interpreted as minimizing the empirical risk (or the training error) $L(f_\btheta)$ via a preconditioned gradient descent update: 
\begin{align}
    \btheta_{t+1} = \btheta_{t} -\eta\bP(t)\nabla_{\btheta_t} L(f_{\btheta_t}), \quad t = 0,1,\ldots
    \label{eq:preconditioned_gradient}
\end{align}
Setting $\bP \!=\! \bI$ recovers ordinary gradient descent (GD). Choices of $\bP$ which exploit second-order information include the inverse Fisher information matrix, which gives the natural gradient descent (NGD) \cite{amari1998natural}; the inverse Hessian, which gives Newton's method \cite{lecun2012efficient}; and diagonal matrices estimated from past gradients, corresponding to various adaptive gradient methods \cite{duchi2011adaptive,kingma2014adam}. By using second-order information, these preconditioners often alleviate the effect of pathological curvature and speed up optimization.
 
However, the typical goal of learning is not to fit a finite training set, but to construct predictors that generalize beyond the training data.
Although second-order methods can lead to faster optimization, their effect on generalization has been largely under debate.
NGD \cite{amari1998natural}, as well as Adagrad \cite{duchi2011adaptive} and its successors \cite{kingma2014adam}, was originally justified in online learning, where efficient optimization directly translates to good generalization.
Nonetheless, there remains the possibility that, in the finite data setting, these preconditioned updates are more (or less) prone to overfitting than GD. For example, several works reported that in neural network optimization, adaptive or second-order methods generalize worse than GD and its stochastic variants \cite{wilson2017marginal,keskar2017improving}, while other empirical studies suggested that second-order methods can achieve comparable, if not better generalization \cite{xu2020second,zhang2018three}. 
We aim to understand when preconditioning using second-order information can help or hurt generalization under fixed training data.

\noindent
\parbox{0.63\linewidth}{
\vspace{1.2mm}
~~~~The generalization property of different optimizers relates to the discussion of \textit{implicit bias} \cite{gunasekar2018characterizing,zhang2016understanding}, i.e.~preconditioning may lead to a different converged solution (even under the same training loss), as shown in Figure~\ref{fig:implicit_bias_illustration}. While many explanations have been proposed, our starting point is the well-known observation that GD often implicitly regularizes the parameter $\ell_2$ norm. For instance in overparameterized least squares regression, GD and many first-order methods find the minimum $\ell_2$ norm solution from zero initialization (without explicit regularization), but preconditioned updates often do not. This being said, while the minimum norm interpolant may generalize well ~\cite{bartlett2019benign}, it is unclear whether preconditioning always leads to inferior solutions -- even in the simple setting of overparameterized linear regression, \textit{quantitative} understanding of how preconditioning affects generalization is largely lacking.
\vspace{1.2mm} 
} 
\parbox{0.01\linewidth}{\hspace{0.1cm}}
\parbox{0.35\linewidth}{ 
% \vspace{-0.1cm} 
{ 
\vspace{0.5mm}
\begin{minipage}[t]{1.0\linewidth}
\centering
{\includegraphics[width=0.98\linewidth]{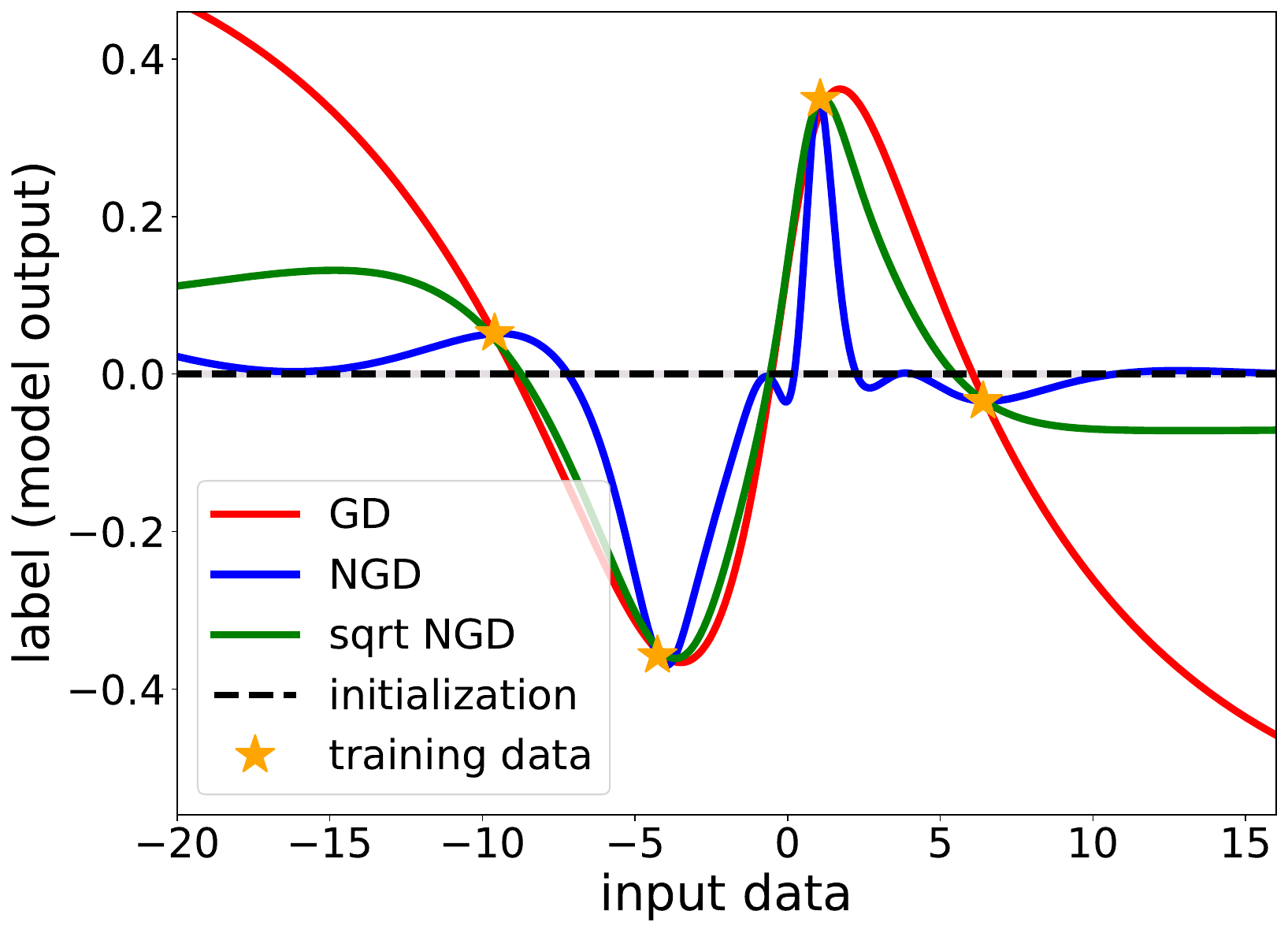}} \\ \vspace{-2.5mm}
\captionof{figure}{\small 1D illustration of different implicit biases: function output of interpolating two-layer sigmoid networks trained with preconditioned gradient descent.}
\label{fig:implicit_bias_illustration}
\vspace{0.5mm}
\end{minipage} 
}
} 
 
Motivated by the observations above, in this work we analyze the least squares regression model, which is convenient and also interesting for several reasons: $(i)$ the Hessian and Fisher matrix coincide and are not time varying, $(ii)$ the optimization trajectory and stationary solution admit an analytical form both with and without preconditioning, $(iii)$ due to overparameterization, different $\bP$ may give solutions with contrasting generalization properties.
Despite its simplicity, linear regression often yields insights and testable predictions for more complex problems such as neural network optimization; indeed we validate the conclusions of our analysis in neural network experiment (although a rigorous connection is not established). 

Our results are organized as follows. In Section~\ref{sec:risk}, we compute the stationary ($t\!\to\!\infty$) generalization error of update~\eqref{eq:preconditioned_gradient} for overparameterized linear regression (unregularized) under time-invariant preconditioner. Extending previous analysis in the proportional limit \cite{hastie2019surprises,dobriban2018high}, we consider a more general random effects model and derive the exact population risk in its \textit{bias-variance decomposition} via random matrix theory. 
We then characterize choices of $\bP$ that achieve the optimal bias or variance within a general class of preconditioners. 
Our analysis focuses on the comparison between GD, for which $\bP$ is identity, and NGD, for which $\bP$ is the inverse population Fisher information matrix\footnote{From now on we use NGD to denote the preconditioned update with the inverse \textit{population} Fisher, and we write ``sample NGD'' when $\bP$ is the inverse or pseudo-inverse of the sample Fisher; see Section~\ref{sec:related} for discussion.}.
Our characterization reveals that the comparison of generalization performance is affected by the following factors:  
\begin{enumerate}[leftmargin=*,topsep=0.5mm,itemsep=0.5mm] 
    \item \textbf{Label Noise:} Additive noise in the labels leads to the \textit{variance} term in the risk. We prove that NGD achieves the optimal variance among a general class of preconditioned updates. 
    \item \textbf{Model Misspecification:} 
    Under misspecification, there does not exist a perfect $f_\btheta$ that recovers the true function (target). We argue that this factor is similar to additional label noise, and thus NGD may also be beneficial when the model is misspecified. 
    \item \textbf{Data-Signal-Alignment:} 
    Alignment describes how the target signal distributes among the input features.
    % \footnote{Our notion of alignment relates to the \textit{source condition} in RKHS \cite{cucker2002mathematical} (see Section~\ref{subsec:RKHS}).} and affects the \textit{bias} term. 
    We show that GD achieves lower \textit{bias} when signal is isotropic, whereas NGD is preferred under ``misalignment''
    --- when the target function focuses on small feature directions. 
\end{enumerate}

\noindent
In addition to the decomposition of stationary risk, our findings in Section~\ref{sec:bias-variance} and \ref{sec:experiment} are summarized as: 
\begin{itemize}[leftmargin=*,topsep=0.75mm,itemsep=0.6mm] 
    \item In Section~\ref{subsec:interpolate} and \ref{subsec:early-stop} we discuss how the bias-variance tradeoff can be realized by different choices of preconditioner $\bP$ (e.g.~interpolating between GD and NGD) or early stopping.
    \item In Section~\ref{subsec:RKHS} we extend our analysis to regression in reproducing kernel Hilbert spaces (RKHS) and show that under early stopping, a preconditioned update interpolating between GD and NGD achieves minimax optimal convergence rate in much fewer steps, and thus reduces the population risk faster than GD.
    \item In Section \ref{sec:experiment} we empirically test how well our predictions from the linear regression setting carry over to neural networks: under a student-teacher setup, we compare the generalization of GD with preconditioned updates and illustrate the influence of all aforementioned factors. The performance of neural networks under a variety of manipulations results in trends that align with our analysis of linear model. 
\end{itemize}

}
\section{Background and Related Works}
\label{sec:related}
{

\paragraph{Natural Gradient Descent.}
NGD is a second-order optimization method originally proposed in~\cite{amari1997neural}. 
Consider a data distribution $p(\bx)$ on the space $\mathcal{X}$, a function $f_\boldtheta: \mathcal{X} \to \mathcal{Z}$ parameterized by $\boldtheta$, and
a loss function $L(\bX, f_\boldtheta) = \frac{1}{n} \sum_{i=1}^n l( y_i, f_\boldtheta(\bx_i))$, where $l: \mathcal{Y} \times \mathcal{Z} \to \R$.
Also suppose a probability distribution $p(y|\bz) = p(y| f_\boldtheta(\bx))$ is defined on the space of labels as part of the model.
Then, the natural gradient is the direction of steepest ascent in the Fisher information norm given by
$
    \tilde{\nabla}_\theta L(\bX, f_\boldtheta) = {\boldF}^{-1} \nabla_\theta L(\bX, f_\boldtheta),
$
where 
\vspace{-0.03cm}
\begin{align}
    \boldF = \mathbb{E} [
        \nabla_\boldtheta \log p(\bx, y | \boldtheta) \nabla_\boldtheta \log p(\bx, y | \boldtheta)^\top
        ]
      = -\mathbb{E}[ 
        \nabla_\boldtheta^2 \log p(\bx, y | \boldtheta)
        ]
\label{eq:fisher_defn}
\end{align}
is the \emph{Fisher information matrix}, or simply the (population) Fisher. Note the expectations in \eqref{eq:fisher_defn} are under the joint distribution of the model $p(\bx, y| \boldtheta) = p(\bx) p(y| f_\boldtheta(\bx) )$.
In the literature, the Fisher is sometimes defined under the empirical data distribution, i.e.~based on a finite set of training examples $\{\bx_i\}_{i=1}^n$~\cite{amari2000adaptive}. We instead refer to this quantity as the \emph{sample Fisher}, the properties of which influence optimization and have been studied in various works \cite{karakida2018universal,karakida2019normalization,kunstner2019limitations,thomas2020interplay}. 
Note that in linear and kernel regression (unregularized) under the squared loss, sample Fisher-based preconditioned updates give the same stationary solution as GD (see \cite{zhang2019fast} and Section~\ref{sec:risk}), whereas population Fisher-based update may not.  

While the population Fisher is typically difficult to obtain, extra unlabeled data can be used in its estimation, which empirically improves generalization under appropriately chosen damping~\cite{pascanu2013revisiting}. 
Moreover, under structural assumptions, estimating the Fisher with parametric approaches can be more sample-efficient~\cite{martens2015optimizing,grosse2016kronecker,ollivier2015riemannian,marceau2016practical}, and thus closing the gap between the sample and population Fisher.

% The term \emph{empirical Fisher} has also appeared in the literature and refers to the case where the expectation in~\eqref{eq:fisher_defn} is modified to be under the joint empirical distribution of data and labels. 
% Its effect on optimization has been empirically studied~\cite{kunstner2019limitations}, and our work will mostly focus on preconditioning using Fisher and the sample Fisher.

When the per-instance loss $l$ is the negative log-probability of an exponential family, the sample Fisher coincides with the \emph{generalized Gauss-Newton matrix}~\cite{martens2014new}.
In least squares regression, which is the focus of this work, the quantity also coincides with the Hessian due to the linear prediction function. Therefore, we take NGD as a representative example of preconditioned update, and we expect our findings to also translate to other second-order methods (not including adaptive gradient methods) applied to regression problems.

\vspace{-2.5mm}  
\paragraph{Analysis of Preconditioned Gradient Descent.}
While \cite{wilson2017marginal} outlined an example on fixed training data where GD generalizes better than adaptive methods, in the online learning setting, for which optimization speed relates to generalization, several works have shown the advantage of preconditioning \cite{duchi2011adaptive,levy2019necessary,zhang2019algorithmic}. In addition, global convergence and generalization guarantees were derived for the sample Fisher-based update in neural networks in the kernel regime~\cite{zhang2019fast,cai2019gram,karakida2020understanding}. 
Lastly, the generalization of different optimizers also connects to the notion of ``sharpness''  \cite{keskar2016large,dinh2017sharp}, and it has been argued that second-order updates tend to find sharper minima \cite{wu2018sgd}.

Two concurrent works also discussed the generalization performance of preconditioned updates. \cite{wadia2020whitening} connected second-order methods to data whitening in linear models, and qualitatively showed that whitening (thus second-order update) harms generalization in certain cases.
\cite{vaswani2020each} analyzed the complexity of the maximum $\bP$-margin solution in linear classification problems.
We emphasize that instead of \textit{upper bounding} the risk (e.g.~Rademacher complexity), which may not decide the optimal $\bP$ (for generalization), we compute the \textit{exact risk} for least squares regression, which allows us to precisely compare different preconditioners.

}

\section{Asymptotic Risk of Ridgeless Interpolants}
\label{sec:risk}
{
\allowdisplaybreaks  
We consider a student-teacher setup: given $n$ training samples $\{\bx_i\}_{i=1}^n$, where labels are generated by a teacher model (target function) $f^*\!: \R^d\!\to\!\R$ with additive noise $y_i = f^*(\bx_i) + \varepsilon_i$, we learn a linear student model $f_{\btheta}$ by minimizing the (empirical) squared loss: $L(\bX,f) = \sum_{i=1}^n \LL(y_i-\bx_i^\top\bt\RR)^2$. We assume a random design: $\bx_i = \bSigma_\bX^{1/2}\bz_i$, where $\bz_i\in\R^d$ is an i.i.d.~random vector with zero-mean, unit-variance, and finite 12th moment, and $\varepsilon$ is i.i.d.~noise independent to $\bz$ with mean 0 and variance $\sigma^2$. 
Our goal is to compute the population risk $R(f) = \E_{P_X}[(f^*(\bx)-f(\bx))^2]$ in the proportional asymptotic limit: 
\begin{itemize}[leftmargin=*,topsep=0.5mm]
    \item \textbf{(A1) Overparameterized Proportional Limit:} $n,d\rightarrow\infty$, $d/n\to \gamma \in (1, \infty)$.
\end{itemize}
 
(A1) entails that the number of features (or trainable parameters) is larger than the number of samples, and there exist multiple empirical risk minimizers with potentially different generalization properties.

Denote $\bX = [\bx_1^\top,..., \bx_n^\top]^\top \in\R^{n\times d}$ the matrix of training data and $\by\in\R^{n}$ the corresponding label vector. We optimize the parameters $\bt$ via a preconditioned gradient flow with preconditioner $\bP(t)\in\R^{d\times d}$,
\vspace{-0.5mm} 
\begin{align}
    \dpar{\bt(t)}{t} = -\bP(t)\dpar{L(\bt(t))}{\bt(t)} = \frac{1}{n} \bP(t)\bX^\top(\by - \bX \bt(t)), \quad \bt(0) = 0.
\end{align}
\vspace{-4.5mm} 

As previously mentioned, in this linear setup, many common choices of preconditioner do not change through time: under Gaussian likelihood, the sample Fisher (and also Hessian) corresponds to the sample covariance $\bX^\top\bX/n$ up to variance scaling, whereas the population Fisher corresponds to the population covariance $\boldF=\bSigma_\bX$. We thus limit our analysis to fixed preconditioner of the form $\bP(t)=:\bP$.  

% Denote parameters at time $t$ under preconditioned gradient update with fixed $\bP$ as $\bt_\bP(t)$. For positive definite $\bP$, the gradient flow trajectory from zero initialization is given as
% \eq{ 
%     \bt_\bP(t) = \bP\bX^\top\left[\bI_n-\exp\left(-\frac{t}{n}\bX \bP\bX^\top\right) \right](\bX\bP\bX^\top)^{-1}\by,
% }
% The stationary solution is obtained by taking the large $t$ limit, which we denote as: $\hbt_\bP \!:=\! \lim_{t\to \infty} \bt_\bP(t) \!=\! \bP\bX^\top(\bX\bP\bX^\top)^{-1}\by$. 
% It is straightforward to check that the discrete time gradient descent update (with appropriately chosen step size) and other variants that do not alter the span of the gradient (e.g.~stochastic gradient or momentum update) converge to this stationary solution as well.  

\noindent
\parbox{0.7\linewidth}{
\vspace{1mm}
~~~~Denote parameters at time $t$ under preconditioned gradient update with fixed $\bP$ as $\bt_\bP(t)$. 
For positive definite $\bP$, the stationary solution is given as: $\hbt_\bP \!:=\! \lim_{t\to \infty} \bt_\bP(t) \!=\! \bP\bX^\top(\bX\bP\bX^\top)^{-1}\by$. 
It is straightforward to check that the discrete time gradient descent update (with appropriate step size) and other variants that do not alter the span of the gradient (e.g.~stochastic gradient or momentum) converge to the same stationary solution as well.  
\vspace{1.6mm} 

~~~~Intuitively speaking, if the data distribution (blue contour in Figure~\ref{fig:implicit_bias_2D}) is not isotropic, then some feature directions will be more ``important'' than others. In this case uniform $\ell_2$ shrinkage (which GD implicitly provides) may not be most desirable, and certain $\bP$ that takes the data geometry into account may lead to better generalization performance instead. The above intuition will be made rigorous in the following subsections.

\vspace{0.5mm} 
} 
\parbox{0.01\linewidth}{\hspace{0.1cm}}
\parbox{0.28\linewidth}{ 
% \vspace{-0.1cm} 
{ 
\vspace{0.5mm}
\begin{minipage}[t]{1.0\linewidth}
\centering
{\includegraphics[width=0.96\linewidth]{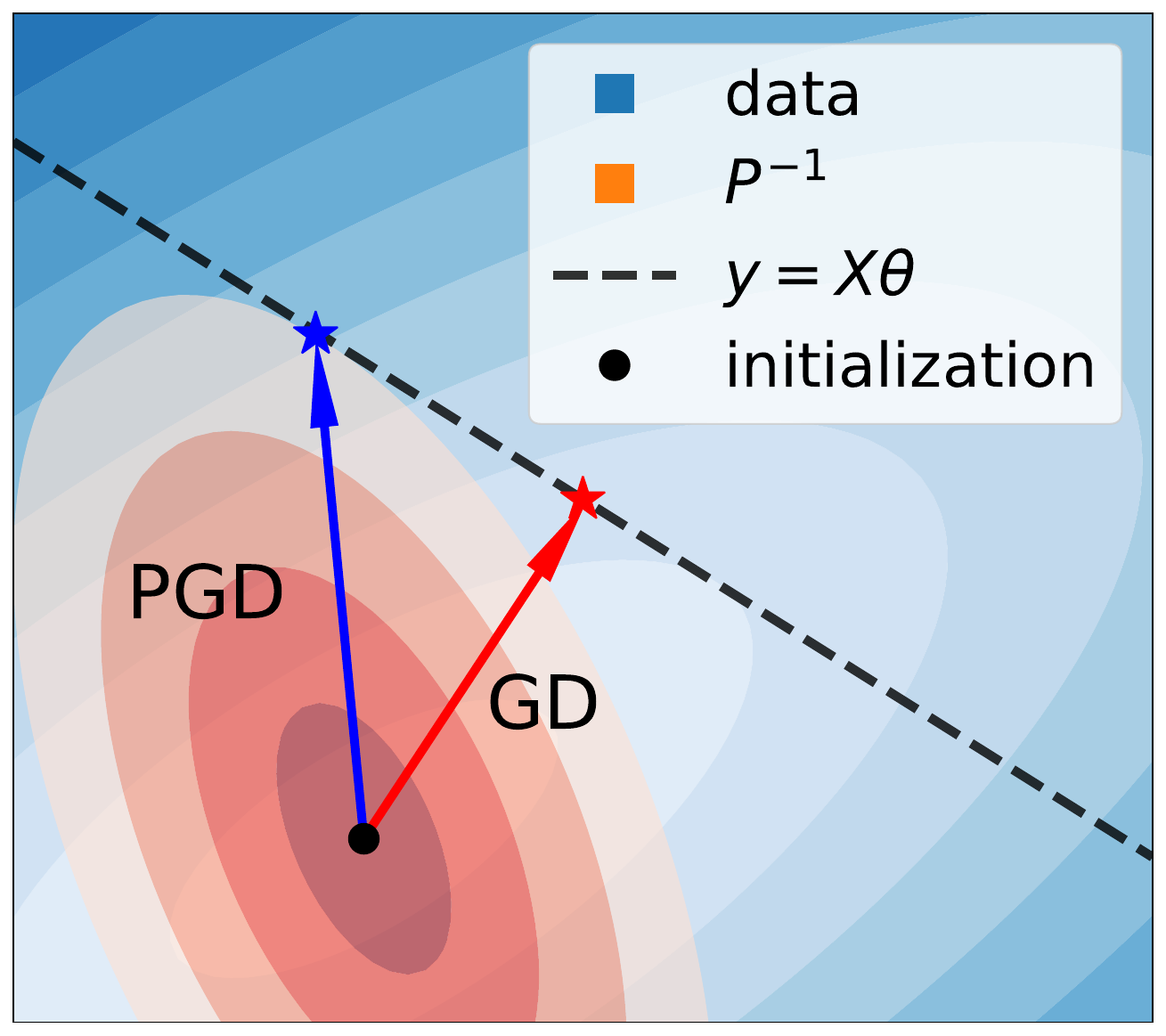}} \\ \vspace{-2.4mm}
\captionof{figure}{\small Geometric illustration (2D) of how the interpolating $\bt_\bP$ depends on the preconditioner.} 
\label{fig:implicit_bias_2D}
% \vspace{0.5mm}
\end{minipage}
}
} 

\begin{remark}
For positive definite $\bP$, the estimator $\hbt_\bP$ is the minimum $\norm{\bt}_{\bP^{-1}}$ norm interpolant: $\hbt_\bP\!=\!\mathrm{arg\,min}_{\bt} \norm{\bt}_{\bP^{-1}}, \, \mathrm{s.t.}\, \bX\bt \!=\! \by$. For GD this translates to the $\ell_2$ norm of the parameters, whereas for NGD ($\bP \!=\! \boldF^{-1} \!=\! \bSigma_{\bX}^{-1}$), the implicit bias is the $\norm{\bt}_{\boldF}$ norm. Since $\E_{P_X}[f(\bx)^2] \!=\! \norm{\bt}_{\bSigma_\bX}^2$, NGD finds an interpolating function with smallest norm under the data distribution. 
We empirically observe this divide between small parameter norm and function norm in neural networks as well (see Figure~\ref{fig:implicit_bias_illustration} and Appendix~\ref{subsec:implicit_bias_appendix}). 
\end{remark}   
% \vspace{-0.25cm}

\noindent
We highlight the following choices of $\bP$ and the corresponding stationary solution $\hbt_\bP$ as $t\to\infty$.

\vspace{-1.5mm} 

\noindent
\parbox{0.66\linewidth}{
\begin{itemize}[leftmargin=*,itemsep=0.3mm]
    \item \textbf{Identity:} $\bP \!=\! \bI_d$ recovers GD that converges to the minimum $\ell_2$ norm interpolant (also true for momentum GD and SGD), which we write as $\hbt_\bI := \bX^\top(\bX\bX^\top)^{-1}\by$ and refer to as the \textit{GD solution}.
    \item \textbf{Population Fisher:} $\bP \!=\!\boldF^{-1}\!=\!\bSigma_\bX^{-1}$ leads to the estimator $\hbt_{\boldF^{-1}}$, which we refer to as the \textit{NGD solution}.
    \item \textbf{Sample Fisher:} since the sample Fisher is rank-deficient, we may add a damping $\bP = (\bX^\top\bX + \lambda\bI_d)^{-1}$ or take the pseudo-inverse $\bP  = (\bX^\top\bX)^\dagger$. In both cases, the gradient is still spanned by $\bX$, and thus the update finds the same min $\ell_2$-norm solution $\hbt_\bI$ (also true for full-matrix Adagrad~\cite{agarwal2018case}), although the trajectory differs, as shown in Figure~\ref{fig:demo} (see Figure~\ref{fig:implicit-bias} for neural networks). 
\end{itemize}
\vspace{-4.mm}
\begin{remark}
The above choices reveal a gap between population- and sample-based preconditioners: while the sample Fisher accelerates optimization \cite{zhang2019fast}, the following sections demonstrate certain generalization properties only possessed by the population Fisher.  
\end{remark} 
} 
\parbox{0.015\linewidth}{\hspace{0.1cm}}
\parbox{0.31\linewidth}{
{ 
% \vspace{-0.5mm}
\begin{minipage}[t]{1.01\linewidth}
\centering
{\includegraphics[width=1.01\linewidth]{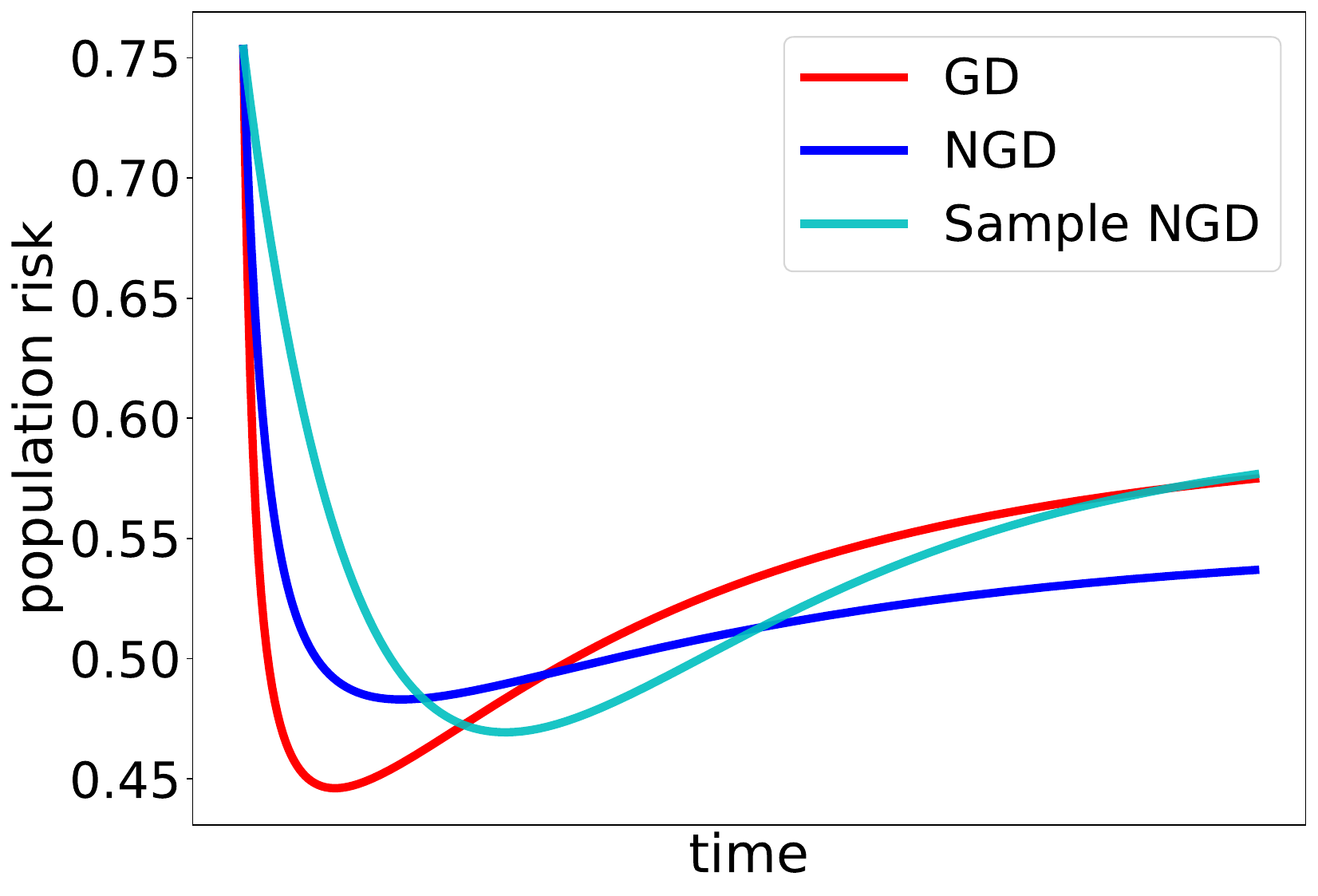}} \\ \vspace{-2mm}
\captionof{figure}{\small Population risk of preconditioned linear regression vs.~time with the following $\bP$: $\bI$ (red), $\bSigma_\bX^{-1}$ (blue) and $(\bX^\top\bX)^\dagger$ (cyan). Time is rescaled differently for each curve (convergence speed is not comparable).
Note that GD and sample NGD give the same stationary risk.}
\label{fig:demo}
\vspace{-0.5mm}
\end{minipage}
}
} 

We compare the population risk of the GD solution $\hbt_\bI$ and the NGD solution $\hbt_{\boldF^{-1}}$ in its bias-variance decomposition w.r.t. the label noise~\cite{hastie2019surprises}, and discuss the two components separately: 
\begin{align}
    R(\bt) = \underbrace{\E_{P_X}[(f^*(\bx) - \bx^\top\E_{P_\varepsilon}[\bt])^2]}_{B(\bt) \text{, bias}} + \underbrace{\Tr{\Cov(\bt)\bSigma_\bX}}_{V(\bt)  \text{, variance}}.  
\end{align}  
\vspace{-3.7mm} 

Note that the \textit{bias} does not depend on the label noise $\varepsilon$, and the \textit{variance} does not depend on the teacher model $f^*$. Additionally, given that $f^*$ can be independently decomposed into a linear component on features $\bx$ and a residual: $f^*(\bx) = \langle\bx,\bt^*\rangle + f^*_c(\bx)$, 
we can further decompose the bias term into a \textit{well-specified} component $\norm{\bt^* - \E\bt}_{\bSigma_\bX}^2$, which captures the difficulty in learning $\btheta^*$, and a \textit{misspecified} component, 
which corresponds to the error due to fitting $f^*_c$ (beyond the class of functions the student can represent).  

\subsection{The Variance Term: NGD is Optimal}
\vspace{-0.2mm} 
 
We first characterize the stationary variance which is independent to the teacher model $f^*$. We restrict ourselves to preconditioners satisfying the following assumption on the spectral distribution: 
\begin{itemize}[leftmargin=*,topsep=0.1mm,itemsep=0.5mm]
    \item \textbf{(A2) Converging Eigenvalues:} $\bP$ is positive definite and as $n,d\to\infty$, the spectral distribution of $\bSigma_{\bX\bP} := \bP^{1/2}\bSigma_{\bX}\bP^{1/2}$ converges weakly to $\boldH_{\bX\bP}$ supported on $[c,C]$ for $c,C>0$.
\end{itemize}
The following theorem characterizes the asymptotic variance and the corresponding optimal $\bP$. 
\begin{theo}
\label{theo:variance}
Given (A1-2), the asymptotic variance is given as
\begin{align}
V(\hbt_\bP) \to \sigma^2\LL(\lim_{\lambda\to 0_+} \frac{m'(-\lambda)}{m^{2}(-\lambda)} - 1\RR),
\label{eq:variance}
\end{align} 
where $m(z)>0$ is the Stieltjes transform of the limiting distribution of eigenvalues of $\frac{1}{n}\bX\bP\bX^\top$ (for $z$ beyond its support) defined as the solution to $m^{-1}(z) = -z + \gamma\int \tau(1+\tau m(z))^{-1}\mathrm{d}\boldH_{\bX\bP}(\tau)$.

Furthermore, under (A1-2), $V(\hbt_\bP) \ge \sigma^2 (\gamma-1)^{-1}$, and the equality is obtained when $\bP = \boldF^{-1} = \bSigma_\bX^{-1}$. 
\end{theo}

Formula \eqref{eq:variance} is a direct extension of \cite[Thorem 4]{hastie2019surprises}, which can be obtained from \cite[Thorem 2.1]{dobriban2018high} or \cite[Thorem 1.2]{ledoit2011eigenvectors}. 
% We note that the eigenvalue condition in (A2) may also be relaxed as in \cite{xu2019many}.
Theorem~\ref{theo:variance} implies that preconditioning with the inverse population Fisher $\boldF$ results in the optimal stationary variance, which is supported by Figure~\ref{fig:ridgeless-1}(a). In other words, when the labels are noisy so that the risk is dominated by the variance, we expect NGD to generalize better upon convergence. 
We emphasize that this advantage is only present when the population Fisher is used, but not its sample-based counterpart (which converges to $\hbt_\bI$). 
In Appendix \ref{subsec:approximate_fisher_appendix} we discuss the substitution error in replacing the population Fisher $\boldF$ with a sample-based estimate using unlabeled data. 
\vspace{-2.8mm}

\paragraph{Misspecification $\approx$ Label Noise.}
Under model misspecification, there does not exist a linear student that perfectly recovers the teacher model $f^*$, which we may decompose as: $f^*(\bx) = \bx^\top\btheta^* + f^*_c(\bx)$. 
In the simple case where $f_c^*$ is an independent linear function on unobserved features (considered in \cite[Section 5]{hastie2019surprises}): $y_i = \bx_i^\top\btheta^* \!+\! \bx_{c,i}^\top\btheta^c \!+\! \varepsilon_i$,  where $\bx_{c,i}\in \R^{d_c}$ is features independent to $\bx_i$, we can show that the additional error in the \textit{bias} term due to misspecification is analogous to the \textit{variance} term above: 

\noindent
\parbox{0.68\linewidth}{
\vspace{-1.5mm}
\begin{coro} 
\label{coro:misspecify}
Under (A1)(A2), for the above unobserved features model with $\E[\bx^c\bx^{c\top}]=\bSigma_\bX^c$ and $\E[\btheta^c\btheta^{c\top}] = d_c^{-1}\bSigma_{\btheta}^c$, the additional error (in the bias term) due to misspecification can be written as $B_c(\hbt_\bP) = d_c^{-1} \Tr{\bSigma_{\bX}^c\bSigma_{\btheta}^c}(V(\hbt_\bP) + 1)$, where $V(\hbt_\bP)$ is the variance term in \eqref{eq:variance}.  
\end{coro}
\vspace{-0.5mm}  
In this case, misspecification can be interpreted as additional label noise, for which NGD is optimal by Theorem~\ref{theo:variance}. 
While Corollary~\ref{coro:misspecify} describes one specific example of misspecification, we expect such characterization to hold under broader settings. 
In particular, \cite[Remark 5]{mei2019generalization} indicates that for many nonlinear $f^*_c$, the misspecified bias is same as variance due to label noise. This result is only shown for isotropic data, but we empirically observe similar phenomenon under general covariance in Figure~\ref{fig:misspecification_nonlinear}, in which $f^*_c$ is a quadratic function.
Observe that NGD leads to lower bias compared to GD as we further misspecify the teacher model.

\vspace{-0.5mm}
} 
\parbox{0.01\linewidth}{\hspace{0.1cm}}
\parbox{0.3\linewidth}{ 
{ 
% \vspace{1mm} 
\begin{minipage}[t]{0.99\linewidth}
\centering
{\includegraphics[width=0.97\linewidth]{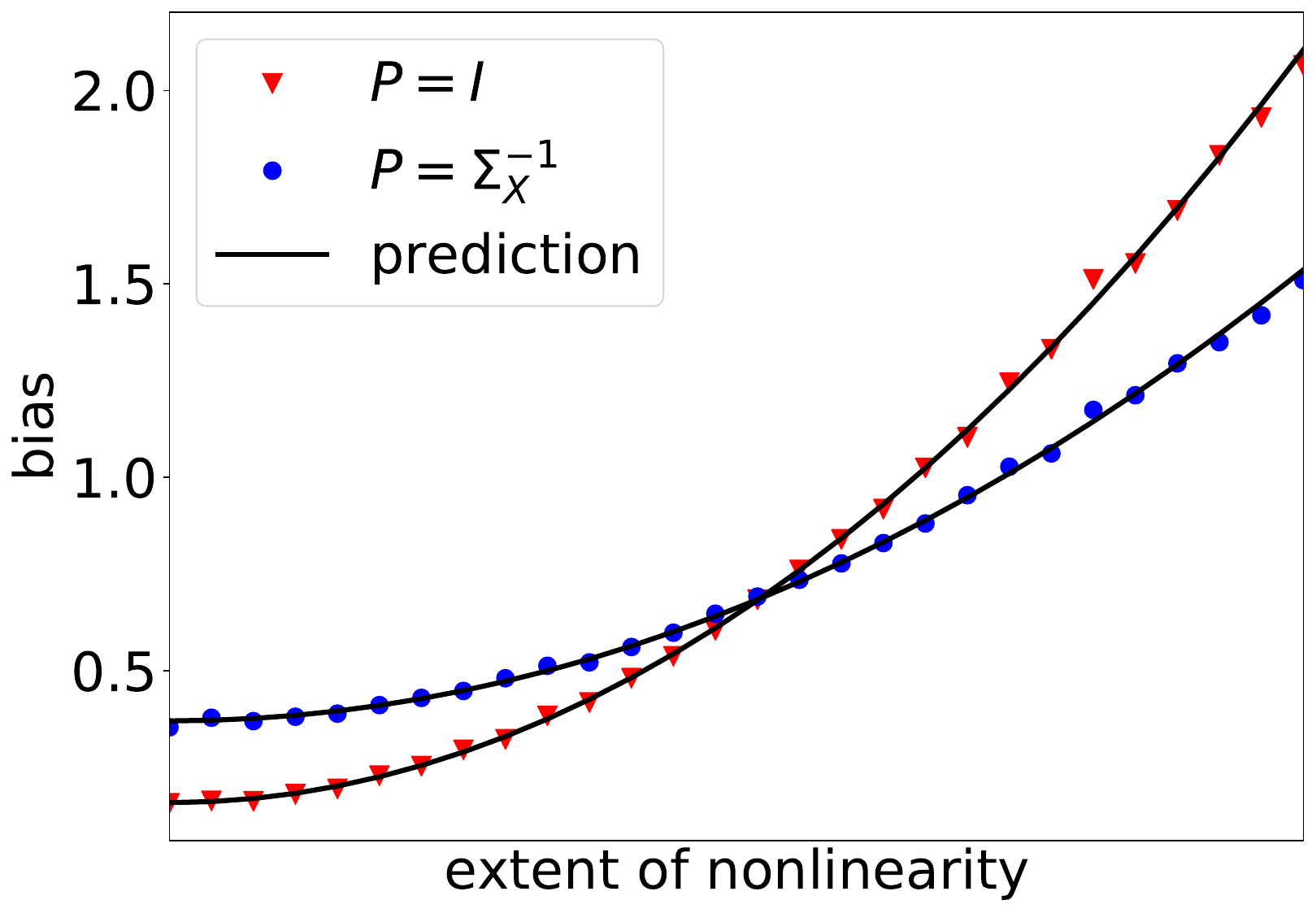}} \\ \vspace{-2.5mm}
\captionof{figure}{\small Misspecified bias with $\bSigma_{\btheta} \!=\! \bI_d$ (favors GD) and $f^*_c(\bx) \!=\! \alpha(\bx^\top\bx\!-\! \Tr{\bSigma_\bX})$, where $\alpha$ controls the extent of nonlinearity. Predictions are generated by matching $\sigma^2$ with second moment of $f^*_c$.} 
\label{fig:misspecification_nonlinear}
\vspace{-2mm}
\end{minipage}
}
} 
  
\subsection{The Bias Term: Alignment and ``Difficulty'' of Learning}
\label{subsec:bias-well_spec}
{
We now analyze the bias term when the teacher model is linear on the input features $\bx$ (hence well-specified): $f^*(\bx) = \bx^\top\btheta^*$. Extending the random effects hypothesis in \cite{dobriban2018high}, we consider a more general prior on $\btheta^*$: $\E[\btheta^*\btheta^{*\top}] =  d^{-1}\bSigma_{\btheta}$, and assume the following joint relations on the covariances and the preconditioner\footnote{Note that (A2)(A3) covers many common choices of preconditioner, such as the population Fisher and variants of the sample Fisher (which is degenerate but leads to the same minimum $\ell_2$ norm solution as GD).}: 

\begin{itemize}[leftmargin=*,topsep=0.5mm,itemsep=0.5mm]
    \item \textbf{(A3) Joint Convergence:} $\bSigma_\bX$ and $\bP$ share the same eigenvectors $\boldU$, and $\|\bP^{-1/2}\bSigma_{\btheta}\bP^{-1/2}\|_2$ is finite. 
    The empirical distributions of elements of $(\bolde_x,\bolde_\theta,\bolde_{xp})$ jointly converge to random variables $(\upsilon_x,\upsilon_{\theta},\upsilon_{xp})$ supported on $[c',\infty)$ for $c'>0$, where $\bolde_x$, $\bolde_{xp}$ are eigenvalues of $\bSigma_\bX$ and $\bSigma_{\bX\bP}$, and $\bolde_\theta = \diag{(\boldU^\top\bSigma_{\theta}\boldU)}$. 
\end{itemize}
We remark that when $\bP \!=\! \bI_d$, previous works \cite{hastie2019surprises,xu2019many} considered the special case of isotropic prior $\bSigma_{\btheta} \!=\! \bI_d$. Our assumption thus allows for analysis of the bias term under much more general $\bSigma_{\bt}$\footnote{Two concurrent works \cite{wu2020optimal,richards2020asymptotics} also considered similar relaxation of $\bSigma_\theta$ in the context of ridge regression.}, which gives rise to interesting phenomena that are not captured by simplified settings, such as non-monotonic bias and variance for $\gamma>1$ (see Figure~\ref{fig:non-monotone}), and the epoch-wise double descent phenomenon (see Appendix~\ref{subsec:epoch_wise_appendix}). Under this general setup, we have the following asymptotic characterization of the bias term:  

\begin{wrapfigure}{R}{0.266\textwidth}  
\vspace{-4.5mm}
\centering 
\includegraphics[width=0.26\textwidth]{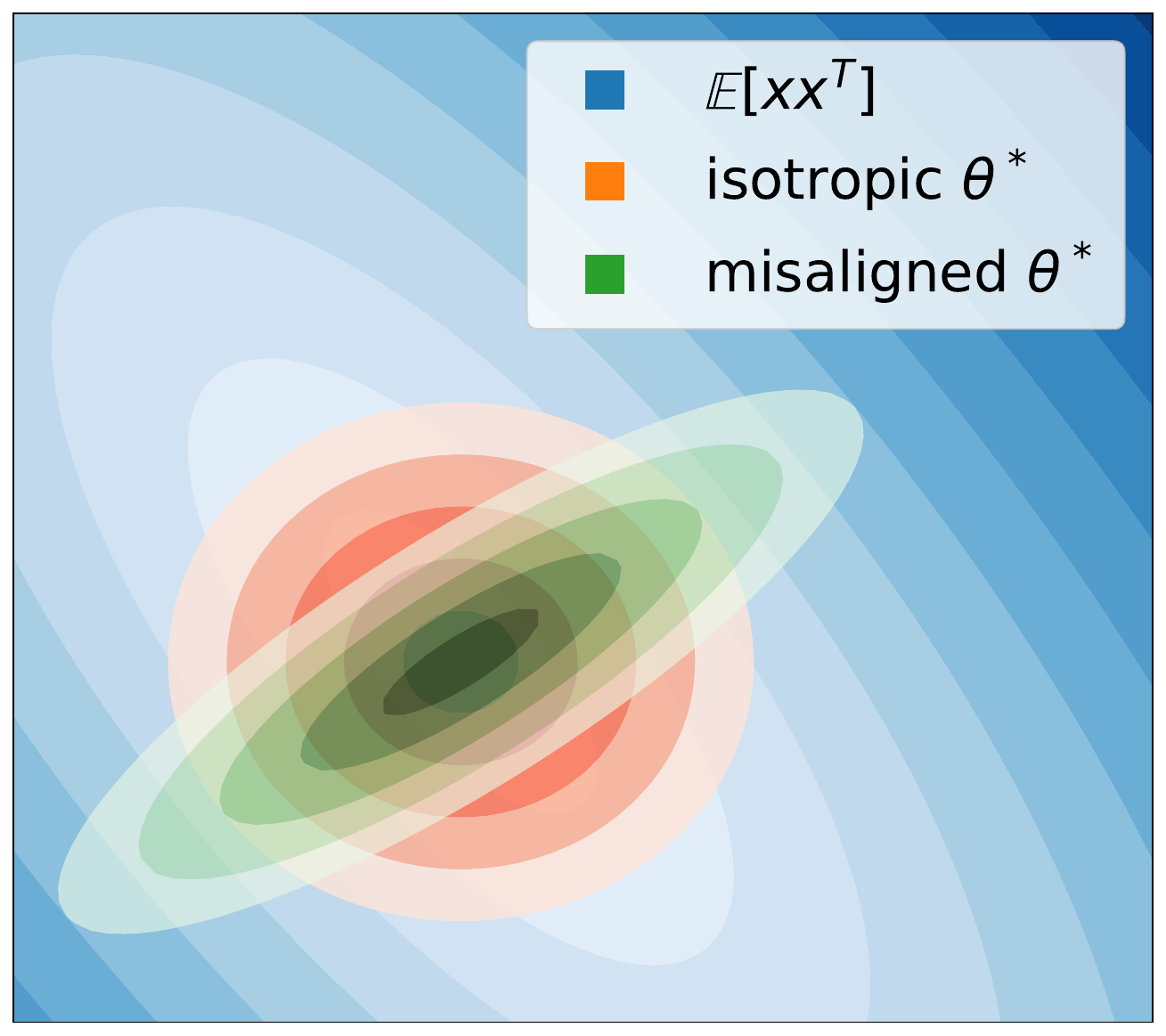}
\vspace{-6.4mm} 
\caption{\small Intuition of isotropic and misaligned teacher $\bt^*$.}   
\label{fig:alignment-illustration}
\vspace{-6.6mm}  
\end{wrapfigure}  

\begin{theo}
\label{theo:bias}
Under (A1)(A3), the expected bias $B(\hbt_\bP):=\E_{\btheta^*}[B(\hbt_\bP)]$ is given as
\begin{align*}
    B(\hbt_\bP) \to \lim_{\lambda\to 0_+} 
    \frac{m'(-\lambda)}{m^{2}(-\lambda)}
    \E\LL[\frac{\upsilon_x\upsilon_{\theta}}{(1 + \upsilon_{xp}m(-\lambda))^{2}}\RR],
    \numberthis
    \label{eq:bias}
\end{align*}
where expectation is taken over $\upsilon$ and $m(z)$ is the Stieltjes transform defined in Theorem~\ref{theo:variance}.

Furthermore, for $\bP$ satisfying (A3),
the optimal bias is achieved by $\bP = \boldU\diag{(\bolde_{\theta})}\boldU^\top$.

\end{theo}

\begin{figure}[t] 
% \vspace{-3.5mm}
\centering
\begin{minipage}[t]{0.32\linewidth}
\centering
{\includegraphics[width=0.94\textwidth]{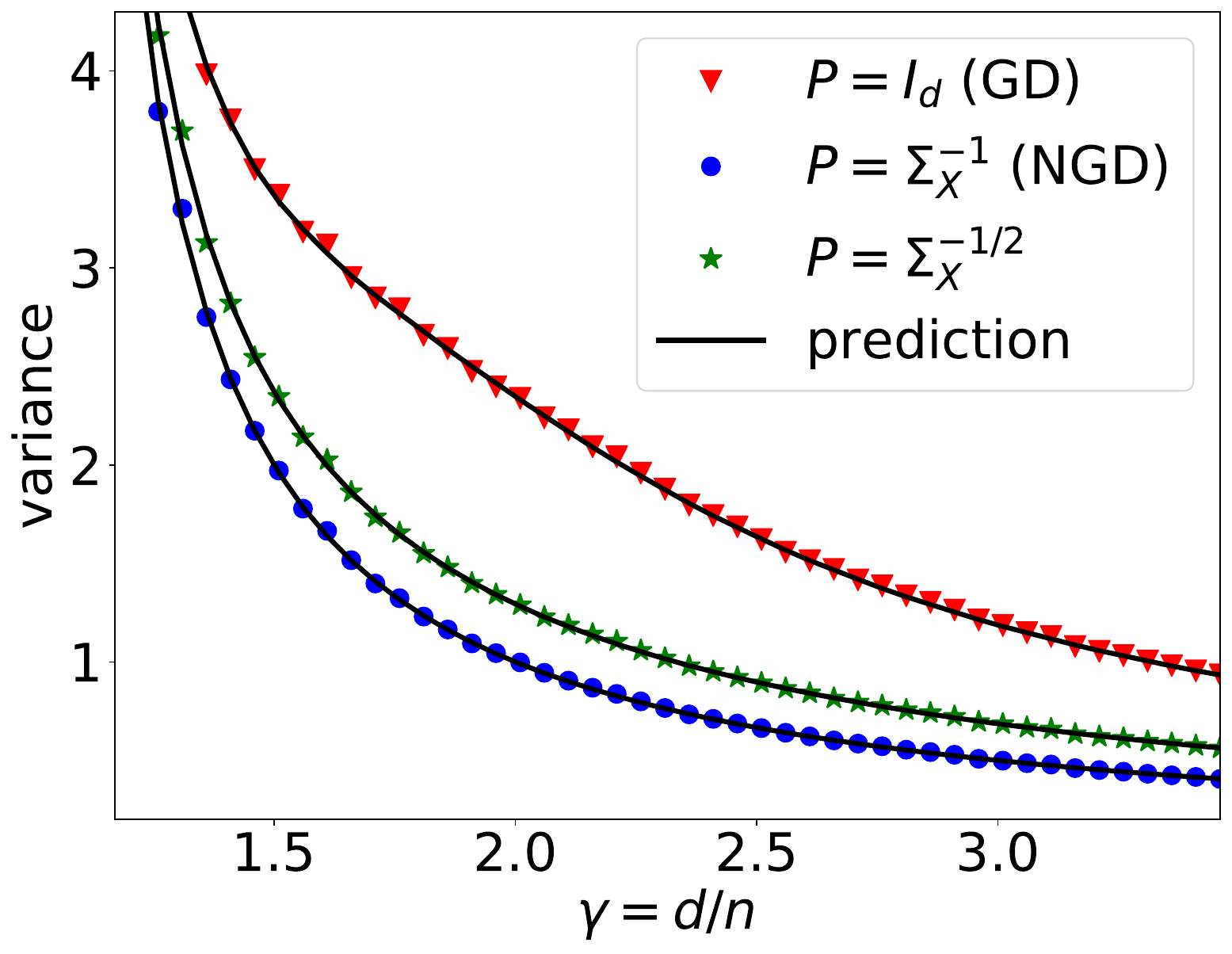}} \\ \vspace{-0.1cm}
\small (a) variance.
\end{minipage}
\begin{minipage}[t]{0.325\linewidth}
\centering
{\includegraphics[width=0.96\textwidth]{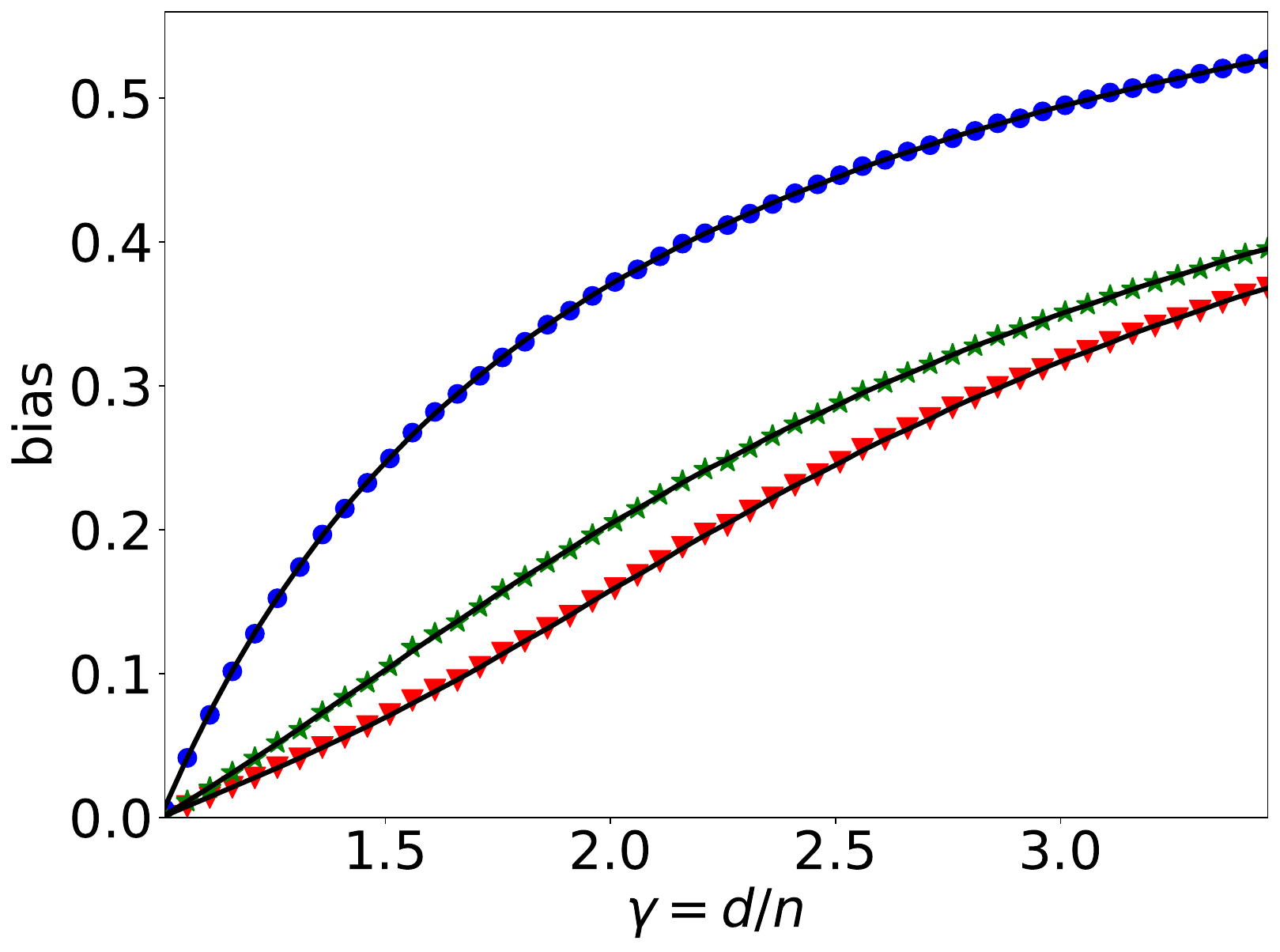}} \\ \vspace{-0.1cm}
\small (b) well-specified bias (isotropic).
\end{minipage}
\begin{minipage}[t]{0.32\linewidth}
\centering 
{\includegraphics[width=0.975\textwidth]{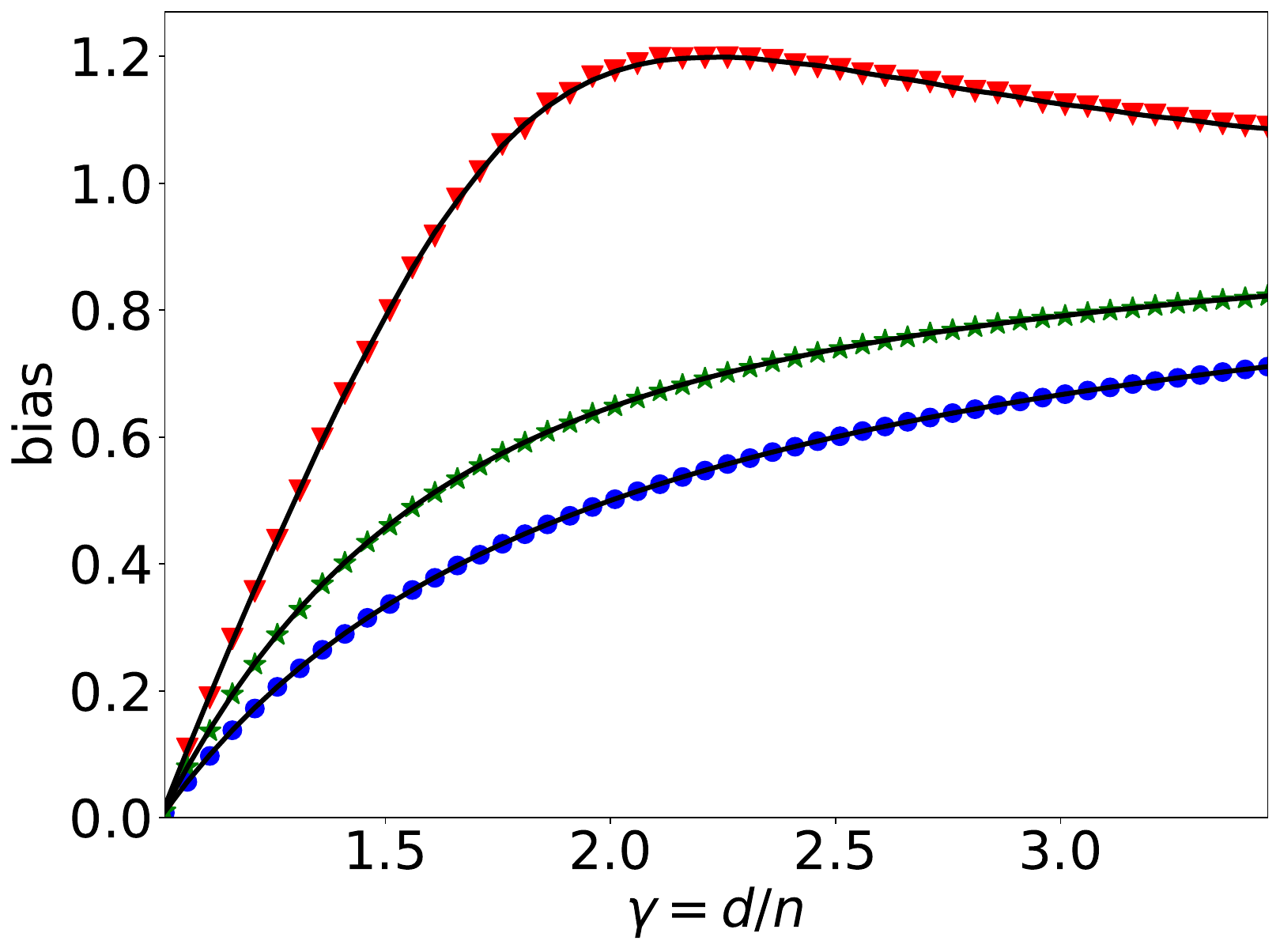}} \\ \vspace{-0.1cm}  
\small (c) well-specified bias (misaligned). 
\end{minipage}
\vspace{-0.15cm} 
\caption{\small We set eigenvalues of $\bSigma_\bX$ as two equally-weighted point masses with $\kappa_X = 20$ and $\norm{\bSigma_\bX}_F^2 = d$; empirical values (dots) are computed with $n=300$. 
(a) NGD (blue) achieves minimum variance. 
(b) GD (red) achieves lower bias under isotropic signal: $\bSigma_{\btheta} = \bI_d$.
(c) NGD achieves lower bias under ``misalignment'': $\bSigma_\bX = \bSigma_{\btheta}^{-1}$.
}
\label{fig:ridgeless-1}
% \vspace{0.5mm}
\end{figure}  

Note that the optimal $\bP$ depends on the ``orientation'' of the teacher model $\bSigma_{\btheta}$, which is usually not known in practice. This result can thus be interpreted as a \textit{no-free-lunch} characterization in choosing an optimal preconditioner for the bias term \textit {a priori}. As a consequence of the theorem, when the true parameters $\btheta^*$ have roughly equal magnitude (isotropic), GD achieves lower bias (see Figure~\ref{fig:ridgeless-1}(b) where $\bSigma_{\btheta}\!=\!\bI_d$). On the other hand, NGD leads to lower bias when $\bSigma_\bX$ is ``misaligned'' with $\bSigma_{\btheta}$, i.e.~when $\btheta^*$ focus on the least varying directions of input features (see Figure~\ref{fig:ridgeless-1}(c) where $\bSigma_{\btheta}\!=\!\bSigma_\bX^{-1}$), in which case learning is intuitively difficult since the features are not useful.
\vspace{-2.8mm}
   
\paragraph{Connection to Source Condition.} The ``difficulty'' of learning above relates to the \textit{source condition} in RKHS literature \cite{cucker2002mathematical} (i.e., $\E[\bSigma_\bX^{r/2}\bt^*]<\infty$, see (A4) in Section~\ref{subsec:RKHS}), in which the coefficient $r$ can be interpreted as a measure of ``misalignment''.
To elaborate this connection, we consider the setting of $\bSigma_{\bt}=\bSigma_\bX^{r}$: note that as $r$ decreases, the teacher $\btheta^*$ focuses more on input features with small magnitude, thus the learning problem becomes harder, and vice versa. In this case we can show a clear transition in $r$ for the comparison between GD and NGD.
\begin{prop}[Informal]
\label{prop:source_condition}
When $\bSigma_\btheta = \bSigma_\bX^{r}$, there exists a transition point $r^* \in (-1,0)$ such that GD achieves lower (higher) stationary bias than NGD if and only if $r>\!(<)\,r^*$.  
\end{prop} 
\vspace{-0.5mm}
The above proposition confirms that for the stationary bias (well-specified), NGD outperforms GD in the misaligned setting (i.e., when $r$ is small), whereas GD has an advantage when the signal is aligned (larger $r$). For formal statement and more discussion on the transition point $r^*$ see Appendix~\ref{subsec:source_condition_appendix}.

}

}

\section{Bias-variance Tradeoff}
\label{sec:bias-variance}
{
Our characterization of the stationary risk suggests that preconditioners that achieve the optimal bias and variance are generally different (except when $\bSigma_\bX \!=\! \bSigma_{\btheta}^{-1}$). 
This section discusses how the bias-variance tradeoff can be realized by interpolating between preconditioners or by early stopping. 
Additionally, we analyze the nonparametric least squares setting and show that by balancing the bias and variance, a preconditioned update that interpolates between GD and NGD also decreases the population risk faster than GD.  

\subsection{Interpolating between Preconditioners}
{
\label{subsec:interpolate}

\parbox{0.66\linewidth}{ 
Depending on the orientation of the teacher model, we may expect a bias-variance tradeoff in choosing $\bP$. Intuitively, given $\bP_1$ that minimizes the bias and $\bP_2$ that minimizes the variance, it is possible that a preconditioner ``in between'' $\bP_1$ and $\bP_2$ could balance the bias and variance and thus generalize better under certain SNR. 
The following proposition confirms this intuition in the setup of general $\bSigma_\bX$ and isotropic $\bSigma_\btheta$, for which GD ($\bP\!=\!\bI_d$) achieves optimal stationary bias and NGD ($\bP\!=\!\boldF^{-1}$) achieves optimal stationary variance\footnotemark.  
\begin{prop}[Informal] 
Let $\bSigma_\bX\!\neq\!\bI_d$ and  $\bSigma_{\btheta}\!=\!\bI_d$.
Consider the following choices of interpolation scheme (under appropriate scaling of $\bSigma_\bX$): (i) $\bP_\alpha \!=\! \alpha\bSigma_\bX^{-1} \!+\! (1\!-\!\alpha)\bI_d$, (ii) $\bP_\alpha = \LL(\alpha\bSigma_\bX \!+\! (1\!-\!\alpha)\bI_d\RR)^{-1}$, (iii) $\bP_\alpha \!=\! \bSigma_\bX^{-\alpha}$. The stationary variance monotonically decreases with $\alpha\in [0,1]$ for all three choices. For (i), the stationary bias monotonically increases with $\alpha\in [0,1]$, whereas for (ii) and (iii), the bias monotonically increases with $\alpha$ in a range that depends on $\bSigma_\bX$.
\label{prop:interpolate}  
\end{prop}
\vspace{-0.15cm} 
} 
\parbox{0.015\linewidth}{\hspace{0.1cm}}
\parbox{0.315\linewidth}{
\vspace{-0.1cm} 
{ 
\begin{minipage}[t]{1.0\linewidth}
\centering
{\includegraphics[width=1\linewidth]{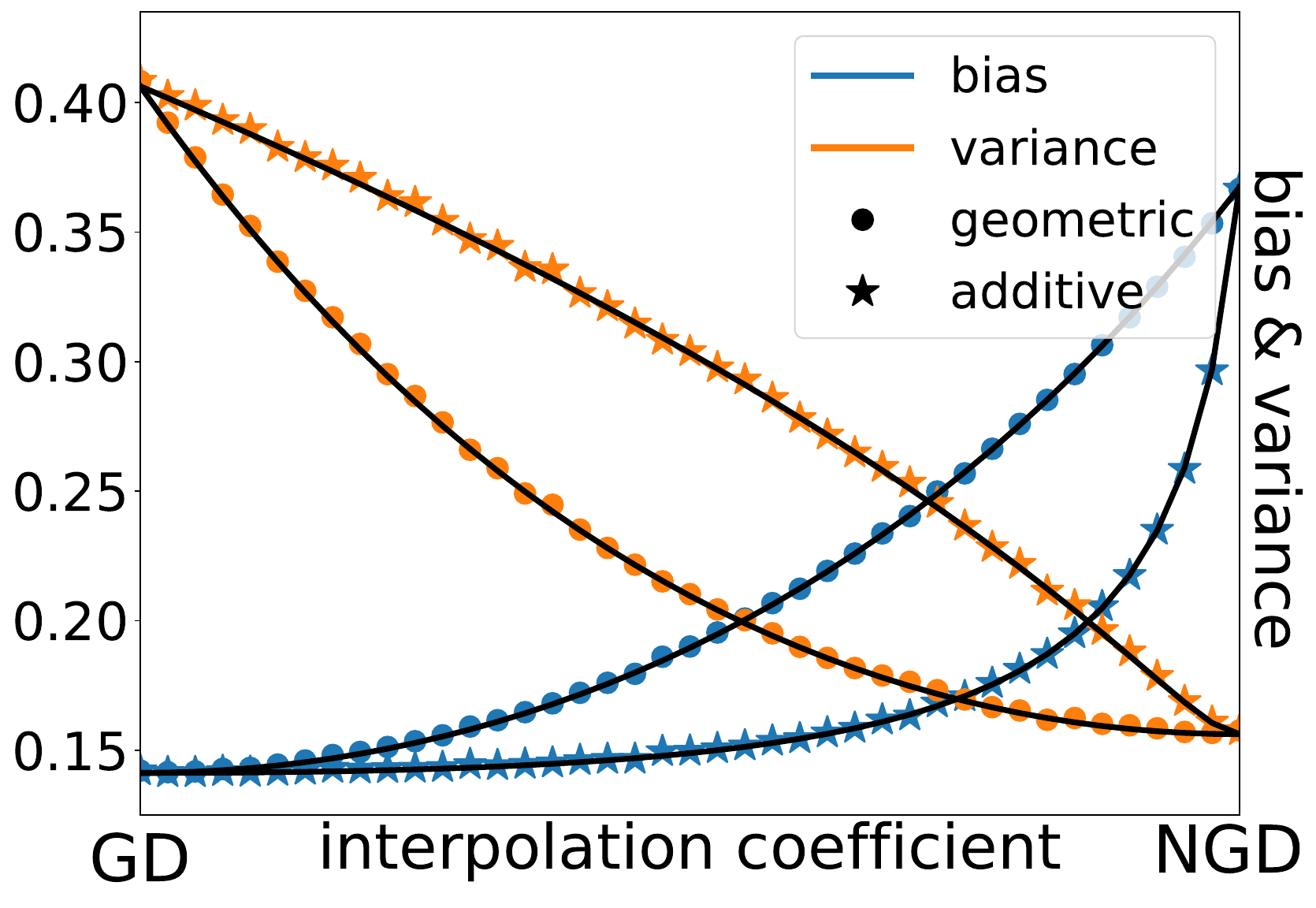}} \\ \vspace{-0.1cm}
\captionof{figure}{\small Bias-variance tradeoff with $\kappa_X \!=\! 25$, $\bSigma_{\bt} \!=\! \bI_d$ and $\text{SNR} \!=\! 32/5$. As we additively ($ii$) or geometrically ($iii$) interpolate from GD to NGD (left to right), the stationary bias (blue) increases and the stationary variance (orange) decreases.}   
\label{fig:tradeoff}
\vspace{-0.1cm}
\end{minipage}
}
} 
\footnotetext{Note that this setup reduces to the random effects model studied in \cite{dobriban2018high,xu2019many}.} 
 
In other words, as the signal-to-noise ratio (SNR) decreases (i.e., more label noise added), one can increase $\alpha$, which makes the update closer to NGD, to improve generalization, and vice versa\footnote{In Appendix~\ref{subsec:interpolate-proof} we empirically verify the monotonicity of the bias term over all $\alpha\!\in\![0,1]$ beyond the proposition.} (small $\alpha$ entails GD-like update).
This intuition is supported by Figure~\ref{fig:tradeoff} and~\ref{fig:ridgeless-2}(c): at certain SNR, a preconditioner that interpolates between $\bSigma_\bX^{-1}$ and $\bSigma_{\btheta}$ can achieve lower stationary risk than both GD and NGD.  
 
\begin{remark} 
Two of the aforementioned interpolation schemes are analogous to common choices in practice: The additive interpolation (ii) corresponds to damping to stably invert the Fisher, whereas the geometric interpolation (iii) includes the ``conservative'' square-root scaling in adaptive gradient methods \cite{duchi2011adaptive,kingma2014adam}. 
\end{remark}
}

\subsection{The Role of Early Stopping}
{\label{subsec:early-stop}
Thus far we considered the stationary solution of the unregularized objective. 
It is known that the bias-variance tradeoff can also be controlled by either explicit or algorithmic regularization. Here we briefly comment on the effect of early stopping, starting from the monotonicity of the variance term w.r.t.~time.  
\begin{prop}
\label{prop:monotone}
For all $\bP$ satisfying (A2), the variance $V(\bt_\bP(t))$ monotonically increases with time $t$.
\end{prop}
\vspace{-0.5mm}

The proposition confirms the intuition that early stopping reduces overfitting to label noise.
Variance reduction can benefit GD in its comparison with NGD, which achieves the lowest stationary variance.
% (this also holds for the misspecified bias, which is analogous to the variance). 
Indeed, Figure~\ref{fig:demo} and \ref{fig:early_stop-2} show that GD may be favored under early stopping even if NGD has lower stationary risk.

On the other hand, early stopping may not always improve the bias in the well-specified case.
While a complete analysis is difficult partially due to the potential non-monotonicity of the bias term (see Appendix~\ref{subsec:epoch_wise_appendix}), we speculate that previous findings for the stationary bias also translate to early stopping.  
As a concrete example, we consider well-specified settings in which either GD or NGD achieves the optimal stationary bias, and demonstrate that such optimality is also preserved under early stopping:

\begin{prop}
\label{prop:early-stop}
Assume (A1) and denote the optimal early stopping bias as $B^{\mathrm{opt}}(\bt) \triangleq \mathrm{inf}_{t\ge 0} B(\bt(t))$. When $\bSigma_{\btheta} = \bSigma_\bX^{-1}$, we have $B^{\mathrm{opt}}(\bt_\bP)\!\ge\! B^{\mathrm{opt}}(\bt_{\boldF^{-1}})$ for all preconditioners $\bP$ satisfying (A3). Whereas if $\bSigma_{\btheta} = \bI_d$, then $B^{\mathrm{opt}}(\bt_{\boldF^{-1}})\ge B^{\mathrm{opt}}(\bt_\bI)$.
\end{prop}
\vspace{-0.5mm}

Figure~\ref{fig:early_stop-2} illustrates that the observed trend in the stationary bias (well-specified) is indeed preserved under optimal early stopping: GD or NGD achieves lower early stopping bias under isotropic or misaligned teacher model, respectively.
We leave a more precise characterization of this observation as future work.
}

\subsection{Fast Decay of Population Risk}
{\label{subsec:RKHS}
Our previous analysis suggests that certain preconditioners can achieve lower population risk (generalization error), but does not address which method decreases the risk more efficiently.
Knowing that preconditioned updates often accelerates optimization, one natural question to ask is, 
is this speedup also present for the population risk under fixed dataset?
We answer this question in the affirmative in a slightly different model: we study least squares regression in the RKHS, and show that a preconditioned gradient update that interpolates between GD and NGD achieves the minimax optimal rate in much fewer iterations than GD. 
 
We provide a brief outline and defer the detailed setup to Appendix~\ref{subsec:RKHS_setup}. Let $\calH$ be an RKHS included in $L_2(P_X)$ equipped with a bounded kernel function $k$, and $K_\bx \in \calH$ be the Riesz representation of the kernel function. Define $S$ as the canonical operator from $\calH$ to $L_2(P_X)$, and write $\Sigma = S^*S$ and $L = SS^*$. We aim to learn the teacher model $f^*$ under the following standard regularity conditions: 
\begin{itemize}[leftmargin=*,itemsep=0.2mm,topsep=0.75mm] 
    \item \textbf{(A4) Source Condition:} $\exists r \!\in \!(0,\infty)$ and $M \!>\! 0$ such that $f^* \!=\! L^r h^*$ for $h^* \!\in\! L_2(P_X)$ and $\norm{f^*}_{\infty} \!\leq\! M$.
    \item \textbf{(A5) Capacity Condition:} There exists $s > 1$ such that $\Tr{\Sigma^{1/s}} < \infty$ and $2r + s^{-1} > 1$.
    \item \textbf{(A6) Regularity of RKHS:} $\exists \mu\in[s^{-1},1]$ and $C_\mu>0$ such that $\sup_{\bx\in\supp{P_X}}\norm{\Sigma^{1/2-1/\mu}K_\bx}_\calH \le C_\mu$.  
\end{itemize}
Note that in the source condition (A4), the coefficient $r$ controls the complexity of the teacher $f^*$ and relates to the notion of ``alignment'' discussed in Section~\ref{subsec:bias-well_spec}: large $r$ indicates a smoother teacher model which is ``easier'' to learn, and vice versa\footnote{We remark that most previous works, including \cite{rudi2017falkon}, considered the case where $r\ge 1/2$ which implies $f^*\!\in\!\calH$.}~\cite{steinwart2009optimal}. 
On the other hand, (A5)(A6) are common assumptions that provide capacity control of the RKHS (e.g., \cite{caponnetto2007optimal,pillaud2018statistical}). 
Given $n$ training points $\{(\bx_i,y_i)\}_{i=1}^n$, we consider the following preconditioned update on the student model $f_t\in\calH$:  
% \vspace{-0.1cm} 
\begin{align}
    f_t = f_{t-1}  - \eta (\Sigma + \alpha I)^{-1}(\hat{\Sigma} f_{t-1} - \hat{S}^*Y), \quad f_0 = 0,
\label{eq:RKHS-update}
\end{align}
% \vspace{-0.5cm}
where $\hat{\Sigma} = \frac{1}{n} \sum_{i=1}^n K_{\bx_i} \!\otimes\! K_{\bx_i}$ and $\hat{S}^*Y = \frac{1}{n}\sum_{i=1}^n y_i K_{\bx_i}$. In this setup, the population Fisher corresponds to the covariance operator $\Sigma$, and thus \eqref{eq:RKHS-update} can be interpreted as \textit{additive} interpolation between GD and NGD: update with large $\alpha$ behaves like GD, and small $\alpha$ like NGD. Related to our update is the FALKON algorithm~\cite{rudi2017falkon}, which is a preconditioned gradient method for kernel ridge regression.
The key distinction is that we consider optimizing the original objective (instead of a regularized version as in FALKON) under early stopping. This is important since we aim to understand \textit{how preconditioning affects generalization}, and thus explicit regularization should not be taken into account (for more discussion see Appendix~\ref{subsec:RKHS_setup}).

The following theorem shows that with appropriately chosen $\alpha$, the preconditioned update~\eqref{eq:RKHS-update} leads to more efficient decrease in the population risk compared to GD, due to faster decay of the bias term.  

\begin{theo}[Informal]
Under (A4-6) and sufficiently large $n$, the population risk of $f_t$ can be written as 
$R(f_t) = \norm{S f_t - f^*}_{\LPiPx}^2 \leq B(t) + V(t)$, where $B(t)$ and $V(t)$ are defined in Appendix~\ref{sec:proof}.
Given $r\!\ge\!1/2$ or $\mu\!\le\!2r$, 
the preconditioned update \eqref{eq:RKHS-update} with $\alpha = n^{-\frac{2s}{2rs+1}}$ achieves  the minimax optimal convergence rate
$R(f_t) = \tilde{O}\LL(n^{-\frac{2rs}{2rs+1}}\RR)$ in $t \!=\! \Theta(\log n)$ steps, whereas ordinary gradient descent requires $t \!=\!  \Theta\LL(n^{\frac{2rs}{2rs+1}}\RR)$ steps.
\label{theo:RKHS}
\end{theo} 
\vspace{-0.5mm}

We comment that the optimal interpolation coefficient $\alpha$ and stopping time $t$ are chosen to balance the bias $B(t)$ and variance $V(t)$. Note that $\alpha$ depends on the teacher model in the following way: for $n>1$, $\alpha$ decreases as $r$ becomes smaller, which corresponds to non-smooth and ``difficult'' $f^*$, and vice versa. This agrees with our previous observation that NGD is advantageous when the teacher model is difficult to learn.
We defer empirical verification of this result to Appendix~\ref{sec:additional_figure}.  

}

}

\section{Neural Network Experiments} 
\label{sec:experiment}
{

\subsection{Protocol}
We compare the generalization performance of GD and NGD in neural network settings and illustrate the influence of the following factors: $(i)$ label noise; $(ii)$ model misspecification; $(iii)$ signal misalignment. We also show that interpolating between GD and NGD can be advantageous due to bias-variance tradeoff.

We consider the MNIST and CIFAR10~\cite{krizhevsky2009learning} datasets.
To create a student-teacher setup, we split the original training set into two equal halves, one of which along with the original labels is used to pretrain the teacher, and the other along with the teacher's labels is used to distill~\cite{hinton2015distilling,bucilua2006model} the student. 
We refer to the splits as the \textit{pretrain} and \textit{distill} split, respectively.
In all scenarios, the teacher is either a two-layer fully-connected ReLU network~\cite{nair2010rectified} or a ResNet~\cite{he2016deep}; whereas the student model is a two-layer ReLU net. 
We normalize the teacher's labels (logits) following~\cite{ba2014deep} before potentially adding label noise and fit the student model by minimizing the L2 loss.
Student models are trained on a subset of the distill split with full-batch updates. 
We implement NGD using Hessian-free optimization~\cite{martens2010deep}.
To estimate the population Fisher, we use 100k unlabeled data obtained by possibly applying data augmentation. 
We report the test error when the training error is below $0.2\%$ of the training error at initialization as a proxy for the stationary risk. 
We defer detailed setup to Appendix~\ref{sec:experiment_setup} and additional results to Appendix~\ref{sec:additional_figure}.

\subsection{Empirical Findings}

\begin{figure}[!htb] 
\centering
\vspace{-2mm} 
\begin{minipage}[t]{0.325\linewidth}
\centering
{\includegraphics[width=0.95\textwidth]{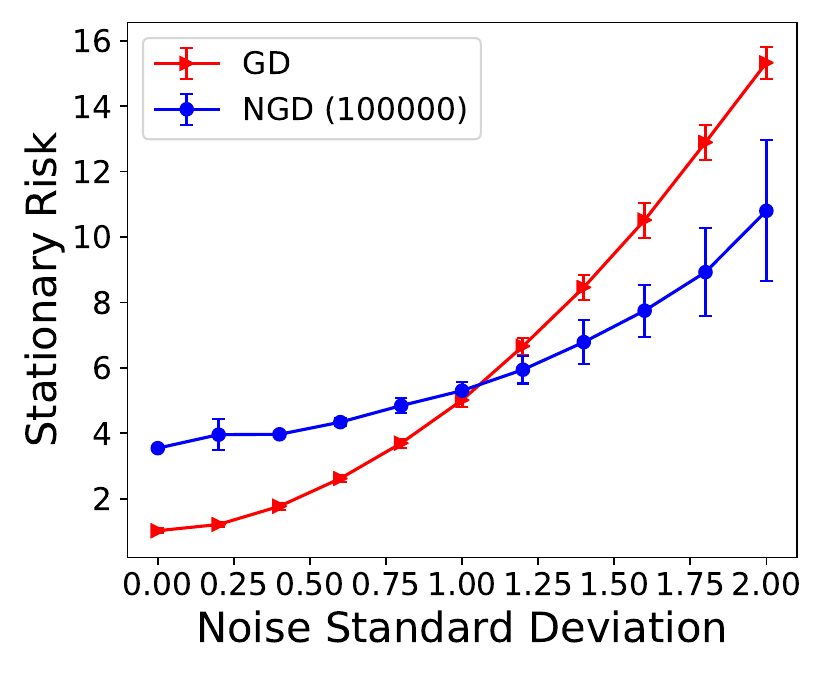}} \\ \vspace{-0.10cm}
\small (a) label noise (MNIST).
\end{minipage}
\begin{minipage}[t]{0.325\linewidth}
\centering
{\includegraphics[width=0.95\textwidth]{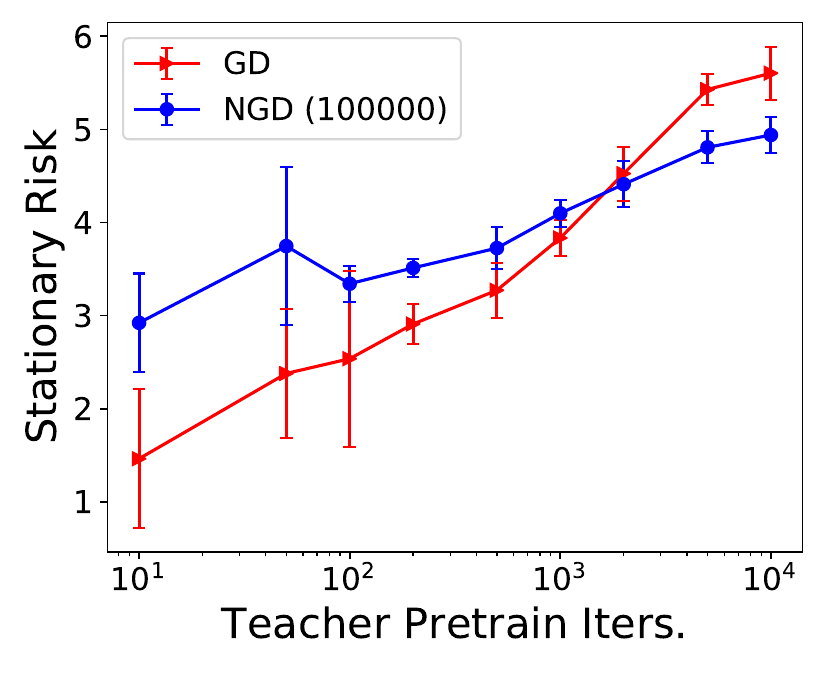}} \\ \vspace{-0.10cm}
\small (b) misspecification (CIFAR-10).
\end{minipage}
\begin{minipage}[t]{0.325\linewidth}
\centering
{\includegraphics[width=0.95\textwidth]{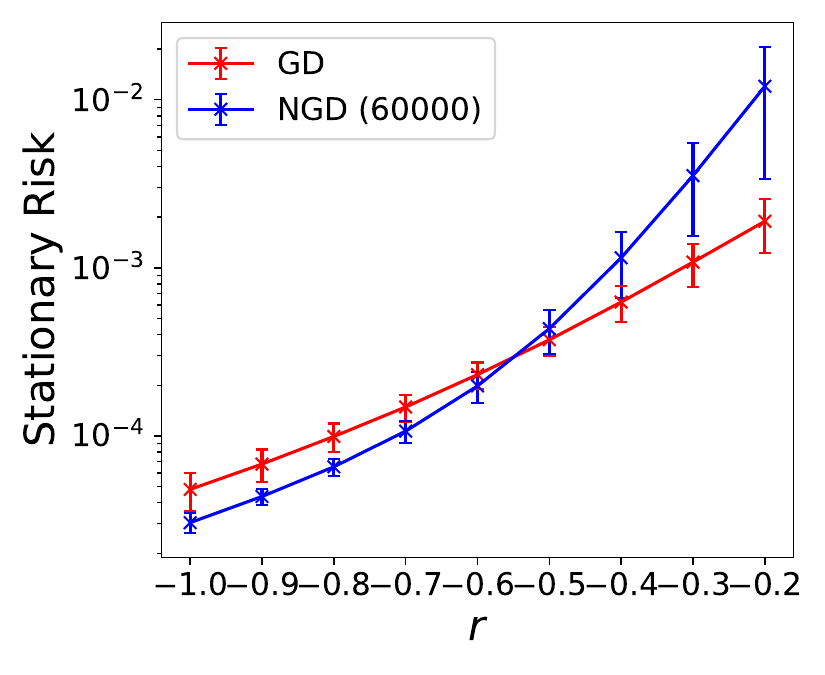}} \\ \vspace{-0.10cm}
\small (c) misalignment (MNIST).
\end{minipage}
\vspace{-1.25mm}
\caption{\small
Comparison between NGD and GD. 
Error bar is one standard deviation away from mean over five independent runs. 
Numbers in parentheses denote amount of unlabeled examples for estimating the Fisher. 
}
\label{fig:dl_generalization}
\vspace{-3.75mm} 
\end{figure}

\paragraph{Label Noise.}
We pretrain the teacher with the full pretrain split and use $1024$ examples from the distill split to fit the student. 
For both the student and teacher, we use a two-layer ReLU net with $80$ hidden units.
We corrupt the labels with isotropic Gaussian noise whose standard deviation we vary. 
Figure~\ref{fig:dl_generalization}(a) shows that as the noise level increases (variance begins to dominate), the stationary risk of both NGD and GD worsen, with GD worsening faster, which aligns with our observation in Figure~\ref{fig:ridgeless-1}. 

\vspace{-2.7mm}
\paragraph{Misspecification.}
We use the same two-layer student from the label noise experiment and a ResNet-20 teacher.
To vary the misspecification level, we take teacher models with the same initialization but varying amount of pretraining.
Intuitively, large teacher models that are trained more should be more complex and thus likely to be outside of functions that a small two-layer student can represent (therefore the problem becomes more misspecified).
Indeed, Figure~\ref{fig:dl_generalization}(b) shows that NGD eventually achieves better generalization as the number of training steps for the teacher increases. 
In Figure~\ref{fig:yky} of Appendix~\ref{subsec:yky_appendix} we report a heuristic measure of model misspecification that relates to the student's NTK matrix \cite{jacot2018neural}, and confirm that the quantity increases as more label noise is added or as the teacher model is trained longer.

\vspace{-2.7mm}
\paragraph{Misalignment.}
We set the student and teacher to be the same two-layer ReLU network. We construct the teacher model by perturbing the student's initialization, the direction of which is given by ${\boldF}^{r}$, where $\boldF$ is the student's Fisher (estimated using extra unlabeled data) and $r\!\in\![-1,0]$. Intuitively, as $r$ decreases, the important parameters of the teacher (i.e.~larger update directions) becomes misaligned with the student's gradient, and thus learning becomes more ``difficult''.
While this analogy is rather superficial due to the non-convex nature of neural network optimization, Figure~\ref{fig:dl_generalization}(c) shows that as $r$ becomes smaller (setup is more misaligned), NGD begins to generalize better than GD (in terms of stationary risk).

\vspace{-2.7mm}
\paragraph{Interpolating between Preconditioners.} We validate our observations in Section~\ref{sec:risk} and \ref{sec:bias-variance} on the difference between the sample Fisher and population Fisher, and the potential benefit of interpolating between GD and NGD, in neural network experiments.
Figure~\ref{fig:interpolate_NN}(a) shows that as we decrease the number of unlabeled data in estimating the Fisher, which renders the preconditioner closer to the sample Fisher, the stationary risk becomes more akin to that of GD, especially in the large noise setting. This agrees with our remark on sample vs.~population Fisher in Section~\ref{sec:risk} and Appendix~\ref{subsec:implicit_bias_appendix}.
 
Figure~\ref{fig:interpolate_NN}(b)(c) supports the bias-variance tradeoff discussed in Section~\ref{subsec:interpolate} in neural network settings. 
In particular, we interpret the left end of the figure to correspond to the bias-dominant regime (due to the same architecture of two-layer MLP for the student and teacher), and the right end to be the variance-dominant regime (due to the large label noise). 
Observe that at certain SNR, a preconditioner that interpolates (additively or geometrically) between GD and NGD can achieve lower stationary risk.

\begin{figure}[!htb]
\centering
\vspace{-0.1cm} 
\begin{minipage}[t]{0.325\linewidth}
\centering
{\includegraphics[width=0.95\textwidth]{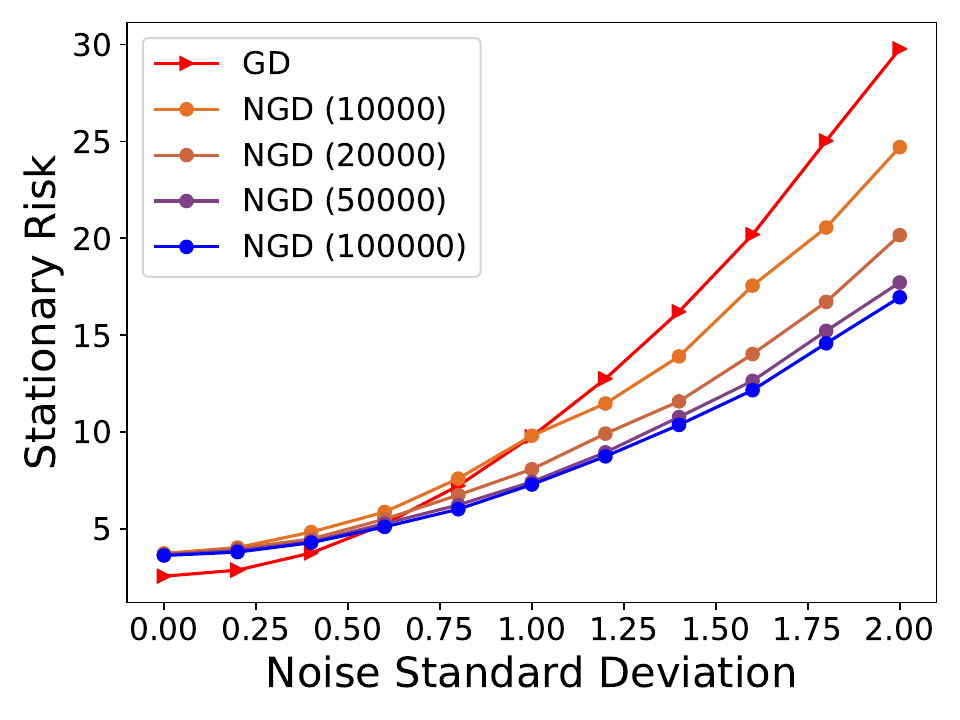}} \\ \vspace{-0.10cm}
\small (a) interpolation between sample and population Fisher (CIFAR-10).
\end{minipage}
\begin{minipage}[t]{0.325\linewidth}
\centering
{\includegraphics[width=0.95\textwidth]{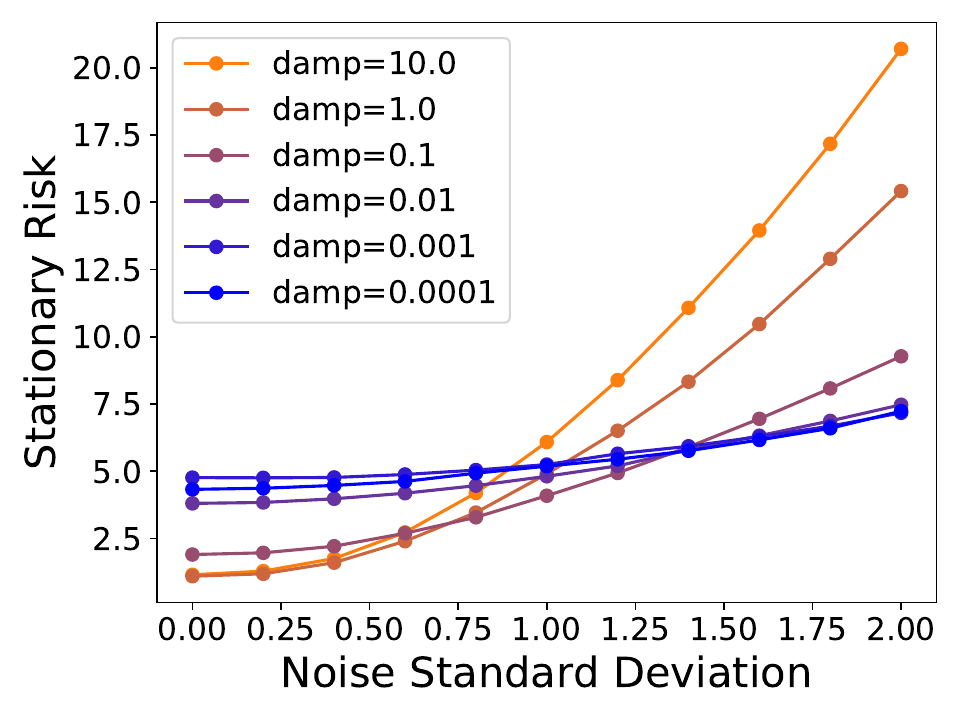}} \\ \vspace{-0.10cm}
\small (b) additive interpolation between GD and NGD (MNIST).
\end{minipage}
\begin{minipage}[t]{0.325\linewidth}
\centering
{\includegraphics[width=0.95\textwidth]{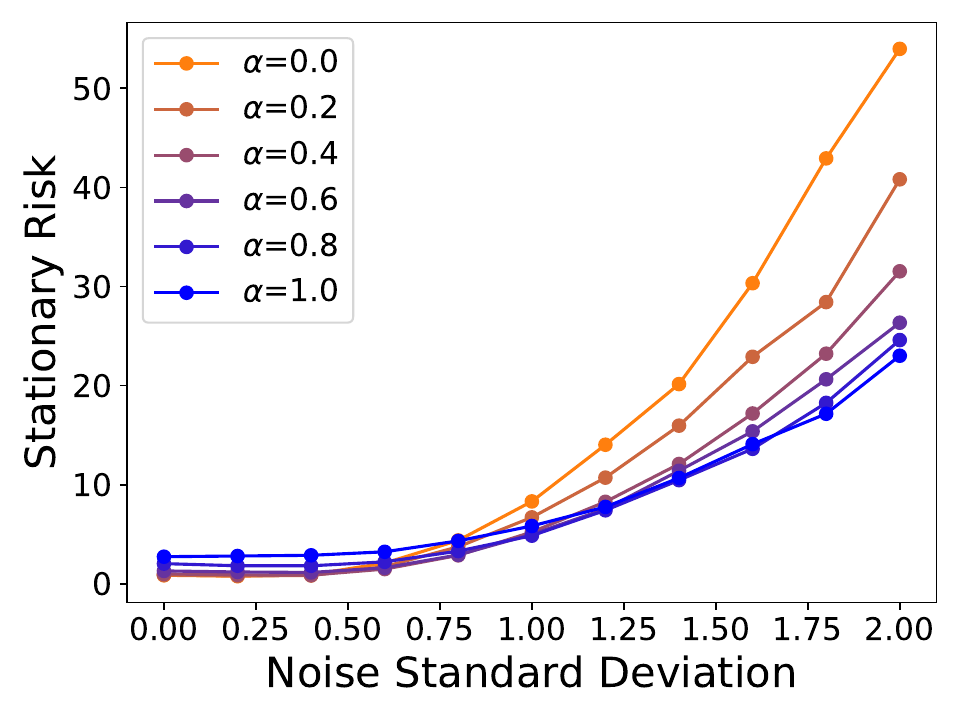}} \\ \vspace{-0.10cm}
\small (c) geometric interpolation between GD and NGD (MNIST). 
\end{minipage}
\vspace{-0.05cm}
\caption{\small
(a) numbers in parentheses indicate the amount of unlabeled data used in estimating the Fisher $\boldF$; we expect the estimated Fisher to be closer to the sample Fisher when the number of unlabeled data is small.
(a) additive interpolation $\boldP = (\boldF + \alpha \bI_d)^{-1}$; larger damping parameter yields update closer to GD (orange). 
(b) geometric interpolation $\boldP = \boldF^{-\alpha}$; larger $\alpha$ parameter yields update closer to that of NGD (blue).
}
\label{fig:interpolate_NN}
\end{figure}
\vspace{-1.5mm}

}

\section{Discussion and Conclusion}
\label{sec:conclusion}
{
 
We analyzed the generalization properties of a general class of preconditioned gradient descent in overparameterized least squares regression, with particular emphasis on natural gradient descent. 
We identified three factors that affect the comparison of generalization performance of different optimizers, the influence of which we also empirically observed in neural network\footnote{We however note that observations in linear or kernel models do not always translate to neural networks -- many recent works have demonstrated such a gap (e.g., see \cite{ghorbani2019limitations,allen2019can,yang2020feature,suzuki2020benefit}).}.  
We then determined the optimal preconditioner for each factor. While the optimal $\bP$ is usually not known in practice, we provided justification for common algorithmic choices by discussing the bias-variance tradeoff.
Note that our current theoretical setup is limited to fixed preconditioners or those that do not alter the span of gradient, and thus does not cover many adaptive gradient methods; characterizing these optimizers in similar setting is an important problem.
In addition, there are many other ingredients that influence generalization, such as loss functions \cite{taheri2020sharp},
and explicit (e.g.~weight decay\footnote{In a companion work \cite{wu2020optimal} we characterized the impact of $\ell_2$ regularization in overparameterized linear regression.} \cite{lewkowycz2020training}) or implicit regularization (e.g.~large learning rate~\cite{li2019towards}); understanding the interplay between preconditioning and these factors would be an interesting future direction. 
 
It is worth noting that our optimal preconditioner may require knowledge of population second-order statistics, which we empirically approximate using extra unlabeled data. Consequently, our results suggest that different ``types'' of second-order information (sample vs.~population) may affect generalization differently.
Broadly speaking, there are two types of practical approximate second-order optimizers for neural networks. 
Some algorithms, such as Hessian-free optimization \cite{martens2010deep,martens2012training,desjardins2013metric}, approximate second-order matrices (typically the Hessian or Fisher) using the exact matrix on finite training examples. In high-dimensional problems, this sample-based approximation may be very different from the population quantity (e.g.~it is necessarily degenerate in the overparameterized regime). 
Other algorithms fit a parametric approximation to the Fisher, such as diagonal \cite{duchi2011adaptive,kingma2014adam}, quasi-diagonal \cite{ollivier2015riemannian}, or Kronecker-factored \cite{martens2015optimizing}. 
If the parametric assumption is accurate, these approximations are more statistically efficient and thus may lead to better approximation to the population Fisher.
Our analysis reveals a difference between sample- and population-based preconditioned updates (in terms of generalization properties), which may also suggest a separation between the two kinds of approximate second-order optimizers.  
As future work, we intend to investigate this discrepancy in real-world problems. 

}

\bigskip
\bigskip

\subsection*{Acknowledgement}
{
The authors would like to thank (in alphabetical order) Murat A.~Erdogdu, Fartash Faghri, Ryo Karakida, Yiping Lu, Jiaxin Shi, Shengyang Sun, Yusuke Tsuzuku, Guodong Zhang, Michael Zhang and Tianzong Zhang for helpful comments and discussions. 
The authors are also grateful to Tomoya Murata for his contribution to preliminary studies on preconditioned update in the RKHS. 
\vspace{-0.1mm}  
 
\noindent
JB and RG were supported by the CIFAR AI Chairs program.
JB and DW were partially supported by LG Electronics and NSERC. 
AN was partially supported by JSPS Kakenhi (19K20337) and JST-PRESTO.
TS was partially supported by JSPS Kakenhi
(26280009, 15H05707 and 18H03201), Japan Digital Design and JST-CREST.
JX was supported by a Cheung-Kong Graduate School of Business Fellowship.  
Resources used in preparing this research were provided, in part, by the Province of Ontario, the Government of Canada through CIFAR, and companies sponsoring the Vector Institute.  
}

\newpage 
{
 
\small
% \fontsize{10}{11.5}\selectfont    

\bibliography{others/citation}
\bibliographystyle{amsalpha}

}

\newpage
\appendix

{

\etocdepthtag.toc{mtappendix}
\etocsettagdepth{mtchapter}{none}
\etocsettagdepth{mtappendix}{subsection}
\tableofcontents
}
\newpage

\section{Discussion on Additional Results}
\label{sec:additional_result}
{

\subsection{Implicit Bias of GD vs.~NGD}
\label{subsec:implicit_bias_appendix}

It is known that gradient descent is the steepest descent with respect to the $\ell_2$ norm, i.e.~the update direction is constructed to decrease the loss under small changes in the parameters measured by the $\ell_2$ norm \cite{gunasekar2018characterizing}.
Following this analogy, NGD is the steepest descent with respect to the KL divergence on the predictive distributions \cite{martens2014new}; this can be interpreted as a proximal update which penalizes how much the predictions change on the data distribution.  

Intuitively, the above discussion suggests GD tend to find solution that is close to the initialization in the Euclidean distance between parameters, whereas NGD prefers solution close to the initialization in terms of the function outputs on $P_X$. This observation turns out to be exact in the case of ridgeless interpolant under the squared loss, as remarked in Section~\ref{sec:risk}.  Moreover, Figure~\ref{fig:implicit_bias_illustration} and \ref{fig:implicit-bias} confirms the same trend in the optimization of overparameterized neural network. In particular, we observe that  

\begin{itemize}[leftmargin=*,itemsep=0.1mm,topsep=0.1mm] 
    \item GD results in small changes in the parameter, whereas NGD results in small changes in the function.
    \item preconditioning with the pseudo-inverse of the sample Fisher, i.e., $\bP = (\boldJ^\top\boldJ)^\dagger$, leads to implicit bias similar to that of GD (also noted in \cite{zhang2019fast}), but not NGD with the population Fisher.
    \item ``interpolating'' between GD and NGD ($\bP\!=\!\boldF^{-1/2}$, green) results in properties in between GD and NGD.
\end{itemize}

% \vspace{1.2mm}
\begin{remark}
Qualitatively speaking, the small change in the function output is the essential reason that NGD performs well under noisy labels in the interpolation setting: NGD seeks to interpolate the training data by changing the function only ``locally'', so that memorizing the noisy labels has small impact on the ``global'' shape of the learned function (see Figure~\ref{fig:implicit_bias_illustration}).
\end{remark}

\begin{figure}[!htb] 
\vspace{-1.5mm} 
\centering
\begin{minipage}[t]{0.4\linewidth}
\centering
{\includegraphics[width=0.94\textwidth]{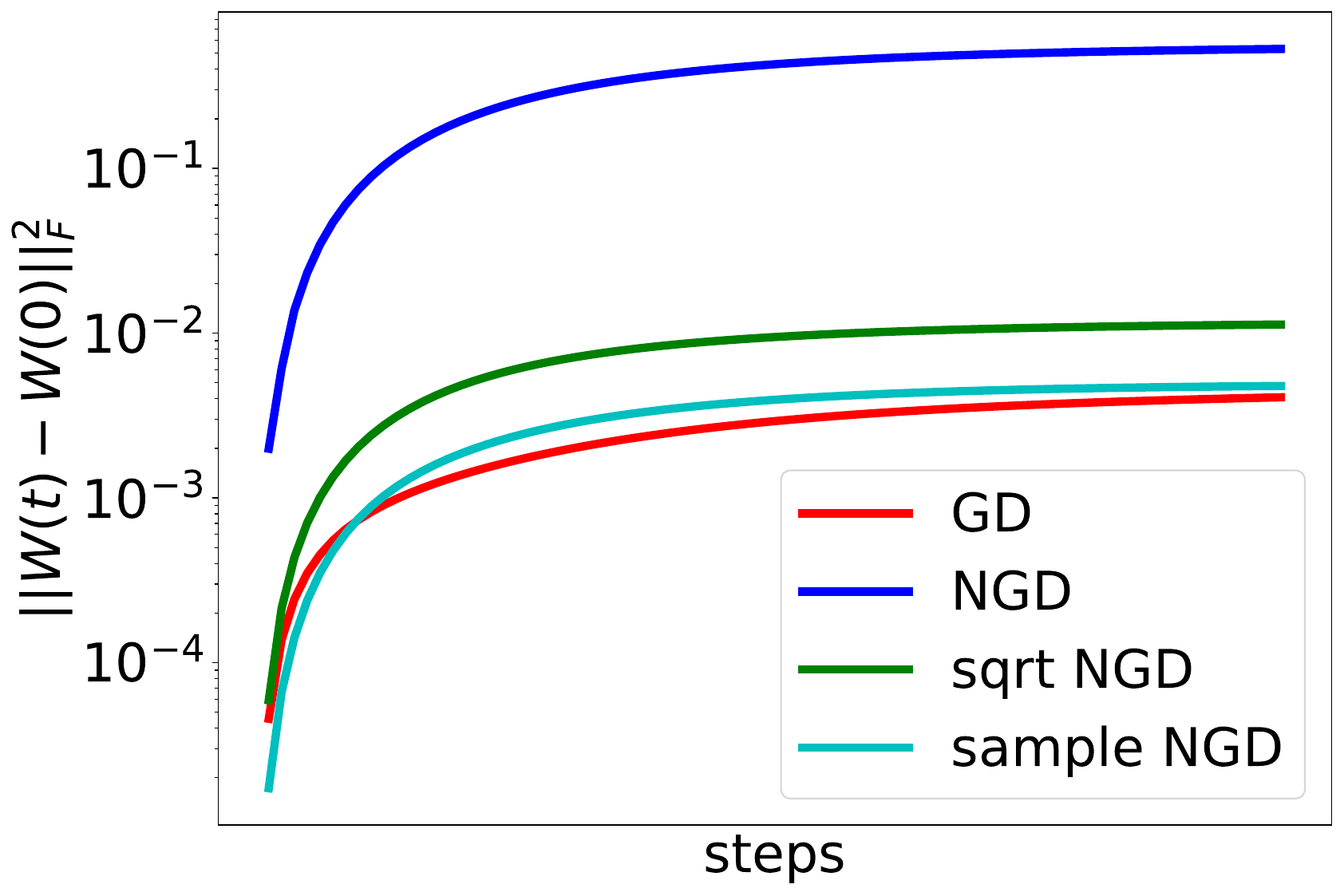}} \\ \vspace{-0.10cm}
\small (a) Difference in Parameters.
\end{minipage}
\begin{minipage}[t]{0.4\linewidth}
\centering
{\includegraphics[width=0.94\textwidth]{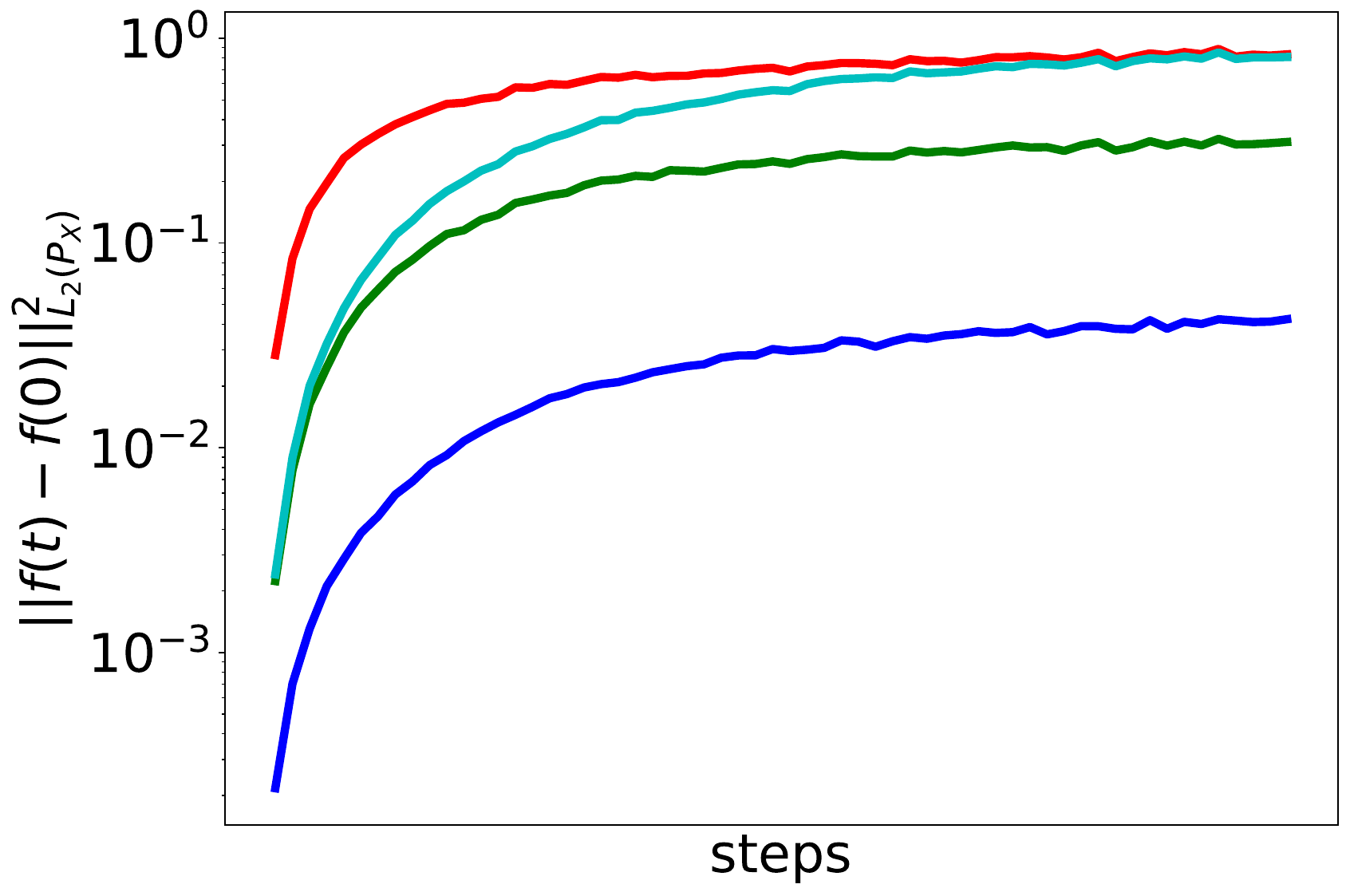}} \\ \vspace{-0.10cm}
\small (b) Difference in Function Values.
\end{minipage} 
\vspace{-1mm} 
 \caption{\small Illustration of implicit bias of GD and NGD. We set $n=100$, $d=50$, and regress a two-layer ReLU network with $50$ hidden units towards a teacher model of the same architecture on Gaussian input. The x-axis is rescaled for each optimizer such that the final training error is below $10^{-3}$. 
 GD finds solution with small changes in the parameters, whereas NGD finds solution with small changes in the function. 
 Note that the sample Fisher (cyan) has implicit bias similar to GD and does not resemble NGD (population Fisher).}
\label{fig:implicit-bias}
% \vspace{-0.1cm}
\end{figure}  

We note that the above observation also implies that wide neural networks trained with NGD (population Fisher) is less likely to stay in the kernel regime: the distance traveled from initialization can be large (see Figure~\ref{fig:implicit-bias}(a)) and thus the Taylor expansion around the initialization is no longer accurate.
In other words, the analogy between wide neural net and its linearized kernel model (which we partially employed in Section~\ref{sec:experiment}) may not be valid in models trained by NGD\footnote{Note that this gap is only present when the population Fisher is used; previous works have shown NTK-type global convergence for sample Fisher-related update \cite{zhang2019fast,cai2019gram}.}.

% \vspace{-1.5mm}
% \bigskip
\paragraph{Implicit Bias of Interpolating Preconditioners.} 
We also expect that as we interpolate from GD to NGD, the distance traveled in the parameter space would gradually increase, and distance traveled in the function space would decrease. Figure~\ref{fig:additional_3_neural_networks} demonstrate that this is indeed the case for neural networks:
we use the same two-layer MLP setup on MNIST as in Section~\ref{sec:experiment}. Observe that updates closer to GD result in smaller change in the parameters, and ones close to NGD lead to smaller change in the function outputs.

\begin{figure}[!htb] 
% \vspace{-10mm}
\centering
\begin{minipage}[t]{0.24\linewidth}
\centering
{\includegraphics[width=1.02\textwidth]{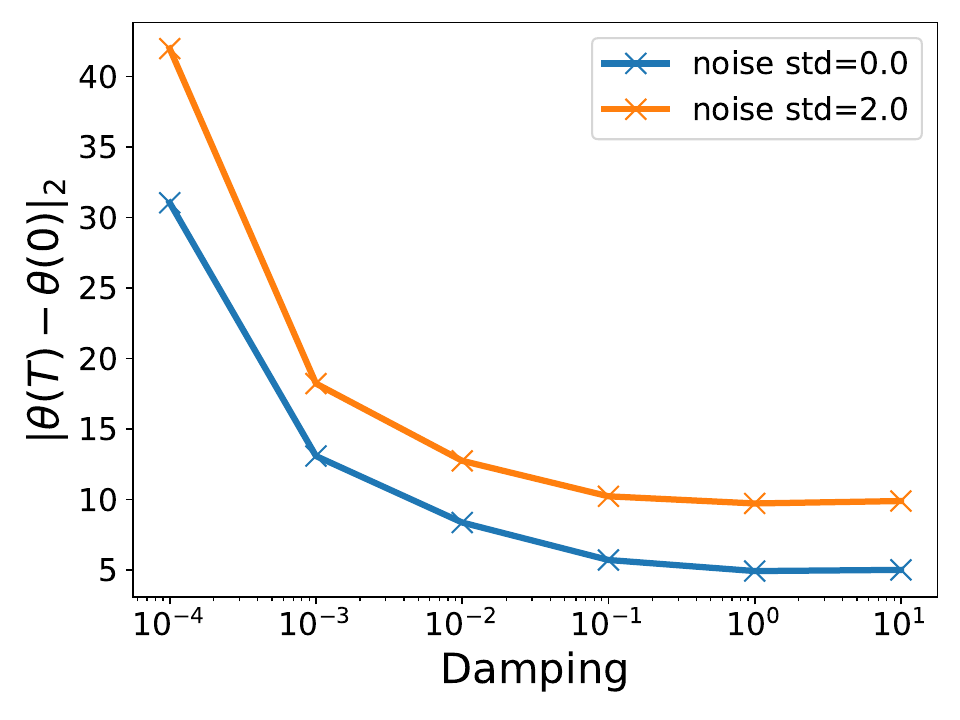}} 
\small (a) additive interp.; \\ difference in parameters. 
\end{minipage}
\begin{minipage}[t]{0.24\linewidth}
\centering
{\includegraphics[width=1.02\textwidth]{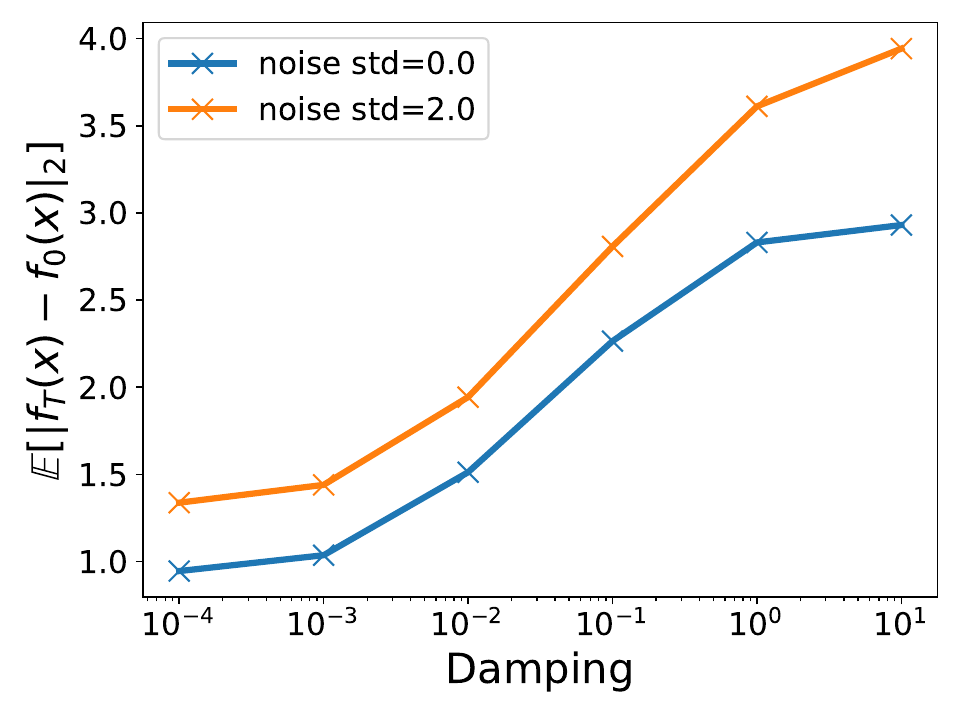}} 
\small (b) additive interp.; \\ difference in functions.
\end{minipage}
\begin{minipage}[t]{0.24\linewidth}
\centering
{\includegraphics[width=1.02\textwidth]{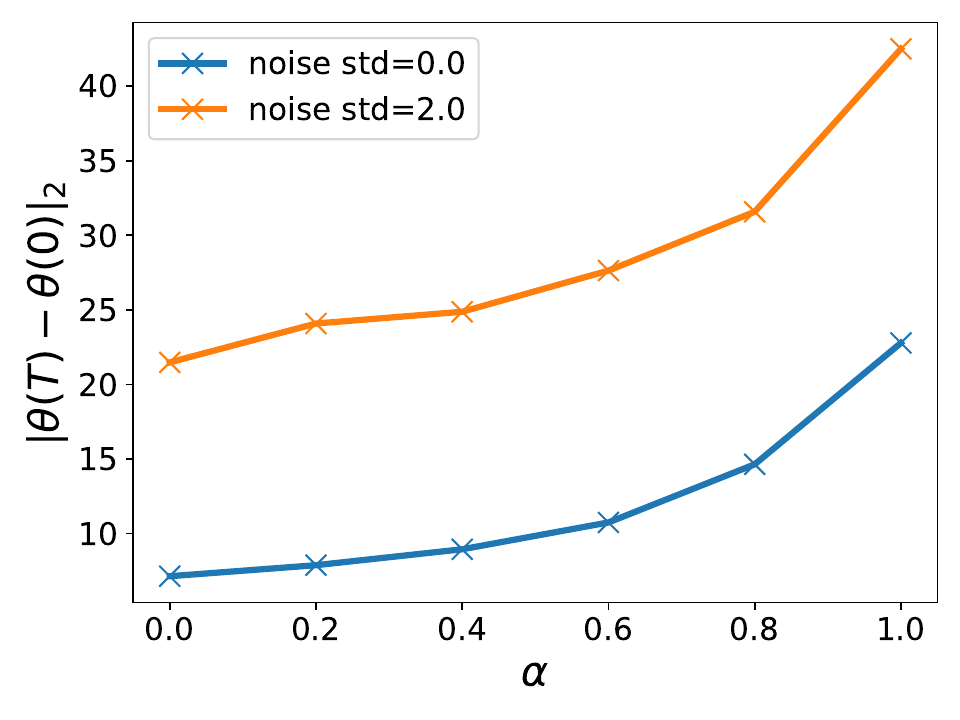}}  
\small (c) geometric interp.; \\ difference in parameters.
\end{minipage}
\begin{minipage}[t]{0.24\linewidth}
\centering
{\includegraphics[width=1.02\textwidth]{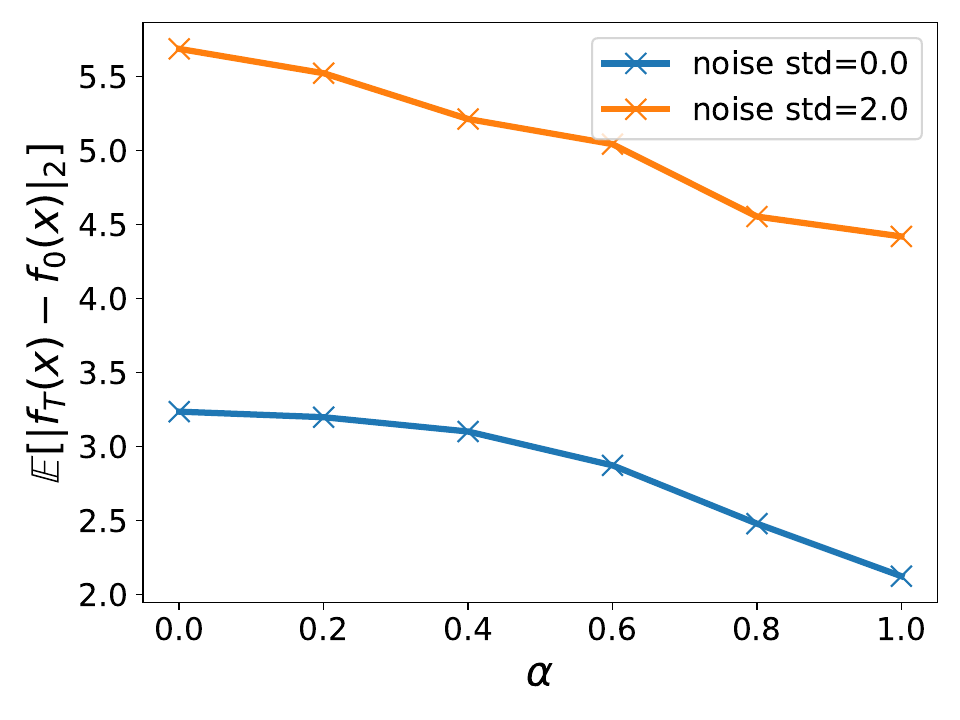}}  
\small (d) geometric interp.; \\ difference in functions.
\end{minipage}
\vspace{-0.5mm}
\caption{\small Illustration of the implicit bias of preconditioned gradient descent that interpolates between GD and NGD on MNIST. Note that as the update becomes more similar to NGD (smaller damping or larger $\alpha$), the distance traveled in the parameter space increases, where as the distance traveled on the output space decreases.}
\label{fig:additional_3_neural_networks}
\vspace{-1mm} 
\end{figure}

\subsection{Non-monotonicity of the Bias Term w.r.t.~Time}
\label{subsec:epoch_wise_appendix}
\begin{wrapfigure}{R}{0.375\textwidth}  
\vspace{-3.6mm} 
\centering 
\includegraphics[width=0.37\textwidth]{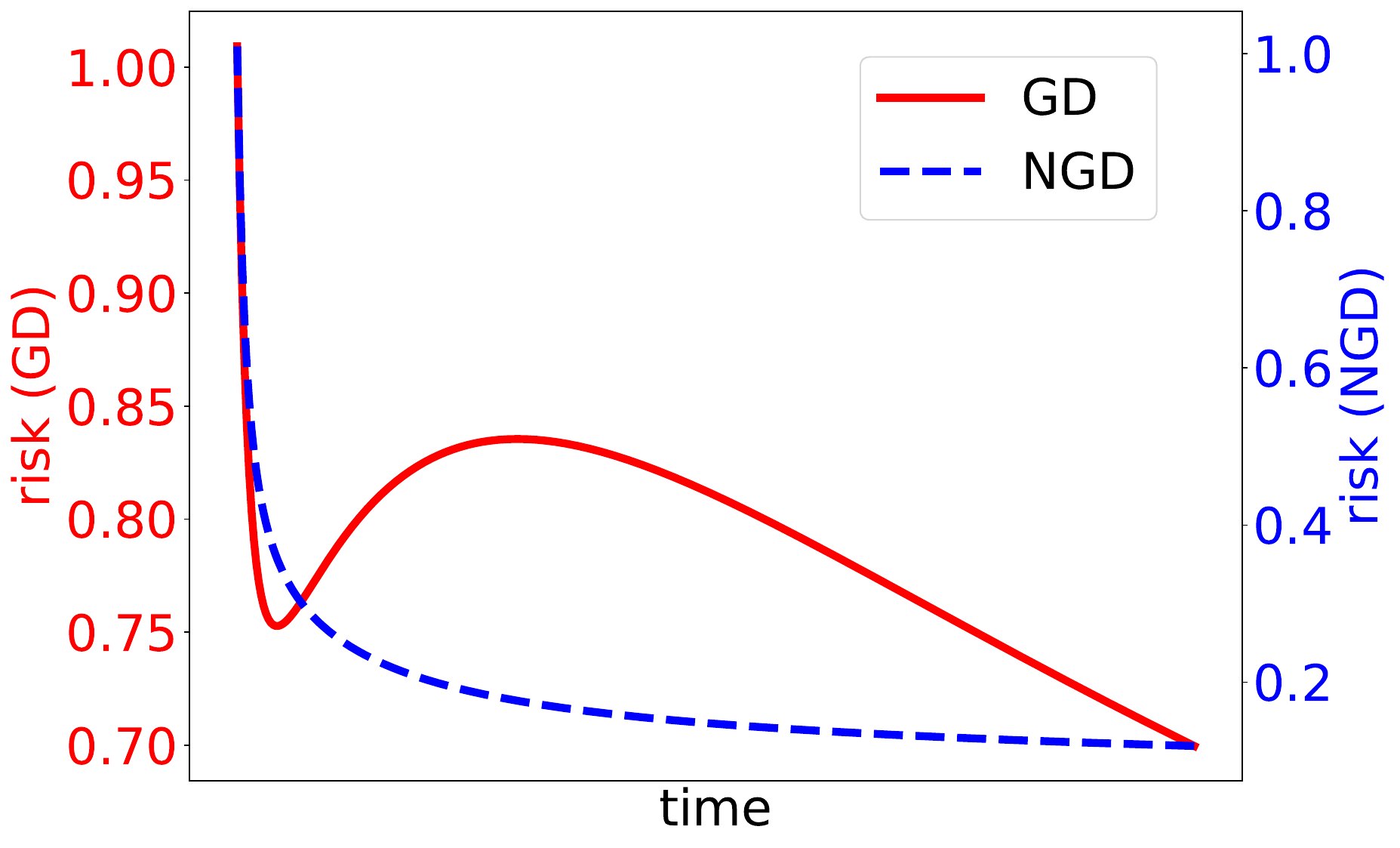}
\vspace{-6mm} 
\caption{\small Epoch-wise double descent. Note that non-monotonicity of the bias term is present in GD but not NGD.}  
\label{fig:epoch-wise}
\vspace{-0.2cm} 
\end{wrapfigure} 

Many previous works on the high-dimensional characterization of linear regression assumed a random effects model with an isotropic prior on the true parameters \cite{dobriban2018high,hastie2019surprises,xu2019many}, which may not hold in practice.
As an example of the limitation of this assumption, note that when $\bSigma_{\bt} = \bI_d$, it can be shown that the expected bias $B(\hbt(t))$ monotonically decreases through time (see proof of Proposition~\ref{prop:early-stop} for details). In contrast, when the target parameters do not follow an isotropic prior, the bias of GD can exhibit non-monotonicity, which gives rise to the surprising ``epoch-wise double descent'' phenomenon also observed in deep learning~\cite{nakkiran2019deep,ishida2020we}.

We empirically demonstrate this non-monotonicity when the model is close to the interpolation threshold in Figure~\ref{fig:epoch-wise}. We set eigenvalues of $\bSigma_\bX$ to be two point masses with $\kappa_X=32$, $\bSigma_{\bt} = \bSigma_\bX^{-1}$ and $\gamma=16/15$. 
Note that the GD trajectory (red) exhibits non-monotonicity in the bias term, whereas for NGD the bias is monotonically decreasing through time (which we confirm in the proof of Proposition~\ref{prop:early-stop}).
We remark that this mechanism of epoch-wise double descent may not relate the empirical findings in deep neural networks (the robustness of which is also unknown), in which it is typically speculated that the variance term exhibits non-monotonicity.

\subsection{Estimating the Population Fisher}
\label{subsec:approximate_fisher_appendix}

Our analysis on linear model considers the idealized setup with access to the exact population Fisher, which can be estimated using additional unlabeled data. 
In this section we discuss how our result in Section~\ref{sec:risk} and Section~\ref{sec:bias-variance} are affected when the population covariance is approximated from $N$ i.i.d.~(unlabeled) samples $\bX_u\in\R^{N\times d}$. For the ridgeless interpolant we have the following result on the substitution error in replacing the true population covariance with the sample estimate.
\begin{prop}
Given (A1)(A3) and $N/d\to\psi>1$ as $d\to\infty$, let $\hat{\bSigma}_\bX = \bX^\top_u\bX_u/N$, we have
\begin{itemize}[leftmargin=18pt,itemsep=0pt,topsep=0pt]
    \item[(a)] $\|\bSigma_\bX- \hat{\bSigma}_\bX\|_2 = O(\psi^{-1/2})$ almost surely.
    \item[(b)] Denote the stationary bias and variance of NGD (with the exact population Fisher) as $B^*$ and $V^*$, respectively, and the bias and variance of the preconditioned update using the approximate Fisher $\hat{\boldF}=\hat{\bSigma}_\bX$ as $\hat{B}$ and $\hat{V}$, respectively. Let $0<\epsilon<1$ be the desired accuracy. Then $\psi=\Theta(\epsilon^{-2})$ suffices to achieve $|B^*-\hat{B}|<\epsilon$ and $|V^*-\hat{V}|<\epsilon$.
\end{itemize}
\label{prop:approximate_fisher}
\end{prop}

Proposition~\ref{prop:approximate_fisher} entails that when the preconditioner is a sample estimate of the Fisher $\hat{\boldF}$ (based on unlabeled data), we can approach the stationary bias and variance of the population Fisher at a rate of $\psi^{-1/2}$ as we increase the number of unlabeled data $N$ linearly in the dimensionality $d$. In other words, any accuracy $\epsilon$ such that $1/\epsilon$ is bounded can be achieved with finite $\psi$
(to push $\epsilon\to 0$, additional logarithmic factor is required, e.g.~$N=O(d\,\text{log}d)$, which is beyond the proportional limit).

On the other hand, for our result in Section~\ref{subsec:RKHS}, \cite[Lemma A.5]{murata2020gradient} directly implies that setting $N = \Theta(n^s \log n)$ is sufficient to approximate the population covariance operator (i.e., so that $\|\Sigma^{1/2}\Sigma_{N,\lambda}^{-1/2}\|=O(1)$). 
Finally, we remark that our analysis above does not impose any structural assumptions on the estimated matrix. When the Fisher exhibits certain structures (e.g.~Kronecker factorization~\cite{martens2015optimizing}), then estimation can be more sample-efficient. For analysis on such approximations of the Fisher see \cite{karakida2020understanding}. 

\subsection{Bias Term Under Specific Source Condition}
\label{subsec:source_condition_appendix}

Motivated by the connection between the notion of ``alignment'' and the \textit{source condition} in Section~\ref{subsec:bias-well_spec}, we consider a specific case of $\bt^*$: $\bSigma_{\bt} = \bSigma_\bX^{r}$, where $r$ controls the extent of misalignment, and Theorem~\ref{theo:bias} implies that the optimal preconditioner for the bias term (well-specified) is $\bP=\bSigma_\bX^{r}$. Note that smaller $r$ corresponds to more misaligned and thus ``difficult'' problem, and vice versa. In this setup we have the following comparison between GD and NGD.
\begin{prop*}[Formal Statement of Proposition~\ref{prop:source_condition}]
Consider the setting of Theorem~\ref{theo:bias} and $\bSigma_{\bt} = \bSigma_\bX^{r}$, then for all $r \le -1$ we have $B(\hbt_{\boldF^{-1}})\le B(\hbt_\bI)$, whereas for all $r \ge 0$, we have $B(\hbt_{\boldF^{-1}}) \ge B(\hbt_\bI)$; the equality is achieved when features are isotropic (i.e., $\bSigma_\bX = c\bI_d$).
\end{prop*}
 
\begin{wrapfigure}{R}{0.345\textwidth} 
\centering 
\vspace{-5.5mm}  
\includegraphics[width=0.335\textwidth]{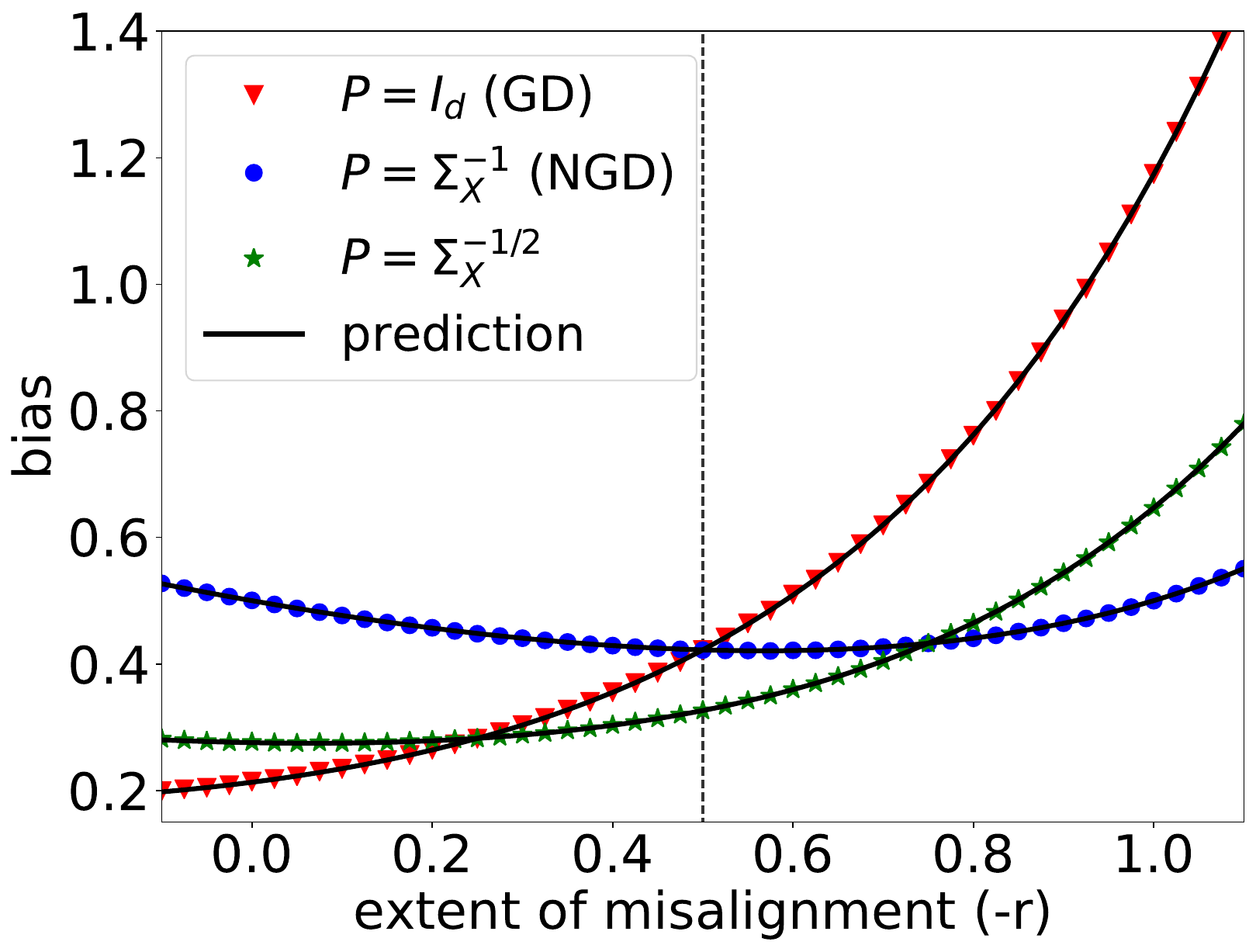}
\vspace{-3mm}  
\caption{\small We set $\bSigma_{\bt} \!=\! \bSigma_\bX^{r}$, $\gamma\!=\!2$, $\kappa\!=\!20$ and plot the bias under varying $r$.}  
\label{fig:extent_misalignment}
\vspace{-5mm}
\end{wrapfigure}  

The proposition confirms the intuition that when parameters of the teacher model $\bt^*$ are more ``aligned'' with features $\bx$ than the isotropic setting ($r\ge 0$), then GD achieves lower bias than NGD; on the other hand, when $\bSigma_{\bt}$ is more ``misaligned'' than the $\bSigma_\bX^{-1}$ case ($r\le-1$), then NGD is guaranteed to be advantageous for the bias term.
We therefore expect a transition from the NGD-dominated to GD-dominated regime for some $r\in(-1,0)$. The exact value of $r$ for such transition depends on the spectral distribution of $\bSigma_\bX$ and varies case-by-case (one would need to specifically evaluate the equality \eqref{eq:source_condition_comparison}). 
To give a concrete example, when $\bSigma_\bX$ has a simple block structure, we can explicitly determine the the transition point $r^*\in (-1,0)$, as shown by the following corollary.
\begin{coro}
Assume $\bSigma_{\bt} = \bSigma_\bX^{r}$, and eigenvalues of $\bSigma_\bX$ come from two equally-weighted point masses with $\kappa \triangleq\lambda_{\max}(\bSigma_\bX) / \lambda_{\min}(\bSigma_\bX)$. WLOG we take $\Tr{\bSigma_\bX}/d=1$. Then given $r^* = -\ln{c_{\kappa,\gamma}} / \ln{\kappa}$ (see Appendix \ref{subsec:proof_special_case} for definition), we have $B(\hbt_{\bI}) \gtreqless B(\hbt_{\boldF^{-1}})$ if and only if $r \lesseqgtr r^*$. 
 
\label{coro:source_condition_special_case}
\end{coro}
\begin{remark}
When $\gamma \!=\! 2$, the transition happens at $r^* \!=\! -1/2$ which is independent of the condition number $\kappa$, as indicated by the dashed line in Figure~\ref{fig:extent_misalignment}. However for other $\gamma\!>\!1$, $r^*$ generally relates to both $\gamma$ and $\kappa$. 
\end{remark}
 
Our characterization above is supported by 
Figure~\ref{fig:extent_misalignment}, in which we plot the bias term under varying extent of misalignment (controlled by $r$) in the setting of Corollary~\ref{coro:source_condition_special_case}. Observe that as we construct the teacher model to be more ``misaligned'' (and thus difficult to learn) by decreasing $r$, NGD (blue) achieves lower bias compared to GD (red), and vice versa.

\subsection{Interpretation of $\small \sqrt{\by^\top \bK^{-1}\by / n}$}

\label{subsec:yky_appendix}
\begin{wrapfigure}{R}{0.34\textwidth} 
\centering 
\vspace{-1.5mm}  
\includegraphics[width=0.33\textwidth]{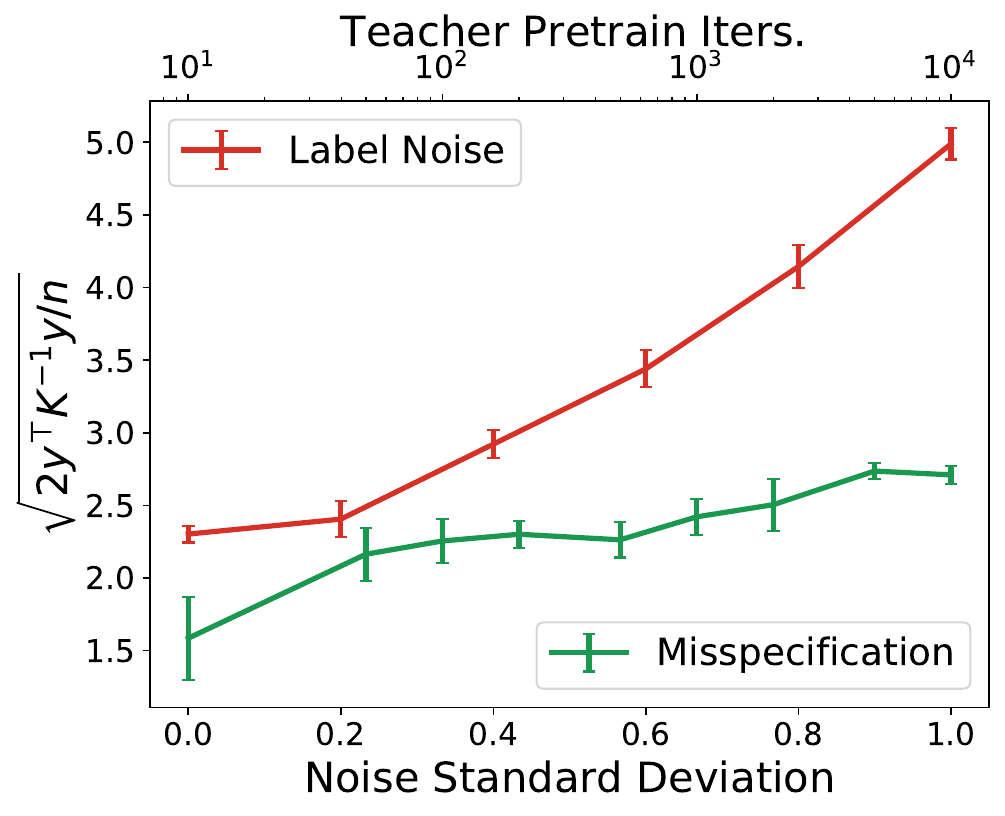}
\vspace{-2.5mm}  
\caption{\small $\small \sqrt{\by^\top \bK^{-1}\by / n}$ on two-layer neural network (CIFAR-10).}  
\label{fig:yky}
\vspace{-10.5mm} 
\end{wrapfigure}  
  
As a heuristic measure of model misspecification, in Figure~\ref{fig:yky} we report $\small \sqrt{\by^\top \bK^{-1}\by / n}$ studied in~\cite{arora2019fine}, where $\by$ is the label vector and $\bK$ is the NTK matrix \cite{jacot2018neural} of the student model.
This quantity relates to generalization of neural networks in the kernel regime, an can be interpreted as a proxy for measuring how much signal and noise are distributed along the eigendirections of the NTK  (e.g.,  \cite{li2019gradient,dong2019distillation,su2019learning}). 
Roughly speaking, large $\small \sqrt{\by^\top \bK^{-1}\by / n}$ implies that the problem is ``difficult'' to learn by GD, vice versa.

Here we give a heuristic argument on how this quantity relates to label noise and misspecification. For the ridgeless regression model considered in Section~\ref{sec:risk}, write $y_i = f^*(\bx_i) + f^c(\bx_i) + \varepsilon_i$, where $f^*(\bx) = \bx^\top\bt^*$, $f^c$ is the misspecified component, and $\varepsilon_i$ is the label noise, then we have the following heuristic calculation:
\vspace{-1mm} 
\begin{align*}  
    &\E\LL[\by^\top \bK^{-1}\by\RR]
=
    \E\LL[\norm{(\bX\bX^\top)^{-1/2}(f^*(\bX) + f^c(\bX) + \boldvarepsilon)}_2^2\RR]
    \\
\overset{(i)}{\approx}&  
    \Tr{\bt^*{\bt^*}^\top \bX^\top(\bX\bX^\top)^{-1}\bX} + (\sigma^2 + \sigma_c^2)\Tr{(\bX\bX^\top)^{-1}},
    \numberthis
    \label{eq:yK-1y}
\end{align*}
where we heuristically replaced the misspecified component with i.i.d.~noise of the same variance $\sigma_c^2$.
The first term of \eqref{eq:yK-1y} resembles an RKHS norm of the target $\bt^*$, whereas the second term is small when the feature matrix is well-conditioned or when the level of label noise $\sigma$ and misspecification $\sigma_c^2$ is small (note that these are conditions under which GD achieves good generalization, due to Theorem~\ref{theo:variance} and Corollary~\ref{coro:misspecify}). 
We may expect similar behavior for neural networks close to the kernel regime. This provides a non-rigorous explanation of the trend observed in Figure~\ref{fig:yky}: as we increase the level of label noise or model misspecification (by pretraining the teacher for more steps), the quantity of interest becomes larger.

\bigskip

}

\section{Additional Related Works}
\label{sec:additional_related}
{

\paragraph{Implicit Regularization in Optimization.} 
In overparameterized linear models, GD finds the minimum $\ell_2$ norm solution under many loss functions. For the more general mirror descent, the implicit bias is determined by the Bregman divergence of the update \cite{azizan2018stochastic,azizan2019stochastic,gunasekar2018implicit,suggala2018connecting}. Under the exponential or logistic loss, recent works demonstrated that GD finds the max-margin direction in various models \cite{ji2018gradient,ji2019implicit,soudry2018implicit,lyu2019gradient,chizat2020implicit}. The implicit bias of Adagrad has been analyzed under similar setting \cite{qian2019implicit}.
The implicit regularization of the optimizer often relates to the model architecture; examples include matrix factorization \cite{gunasekar2017implicit,saxe2013exact,gidel2019implicit,arora2019implicit} and various types of neural network \cite{li2017algorithmic,gunasekar2018implicit,williams2019gradient,woodworth2020kernel}.
For neural networks in the kernel regime \cite{jacot2018neural}, the implicit bias of GD relates to properties of the limiting neural tangent kernel (NTK) \cite{xie2016diverse,arora2019fine,bietti2019inductive}.
We also note that the implicit bias of GD is not always explained by the minimum norm property \cite{razin2020implicit}.

\paragraph{Asymptotics of Interpolating Estimators.} 
In Section~\ref{sec:risk} we analyze overparameterized estimators that interpolate the training data. Recent works have shown that interpolation may not lead to overfitting \cite{liang2018just,belkin2018does,belkin2018overfitting,bartlett2019benign}, and the optimal generalization error may be achieved under no regularization and extreme overparameterization \cite{belkin2018reconciling,xu2019many}. The asymptotic prediction risk of overparameterized models has been characterized in various settings, such as linear regression \cite{karoui2013asymptotic,dobriban2018high,hastie2019surprises,wu2020optimal}, random features regression \cite{mei2019generalization,gerace2020generalisation,hu2020universality}, max-margin classification \cite{montanari2019generalization,deng2019model}, and certain stylized neural networks \cite{advani2017high,ba2020generalization}.
Our analysis is based on results in random matrix theory developed in \cite{rubio2011spectral,ledoit2011eigenvectors}.  
Similar tools can also be used to study the gradient descent dynamics of linear regression \cite{liao2018dynamics,ali2019continuous}.

}

% \newpage

\bigskip

\section{Additional Figures}
\label{sec:additional_figure}
{
\subsection{Overparameterized Linear Regression}
\label{subsec:additional_figures_linear}
{

\paragraph{Non-monotonicity of the Risk.} Under our generalized (anisotropic) assumption on the covariance of the features and the target, both the bias and the variance term can exhibit non-monotonicity w.r.t.~the overparameterization level $\gamma>1$: in Figure~\ref{fig:non-monotone} we observe two peaks in the bias term and three peaks in the variance term. In contrast, it is known that when $\bSigma_\bX = \bI_d$, both the bias and variance are \textit{monotone} in the overparameterized regime (e.g., \cite{hastie2019surprises}).

\begin{figure}[!htb]
\vspace{-1mm} 
\centering
\begin{minipage}[t]{0.4\linewidth}
\centering
{\includegraphics[width=0.96\textwidth]{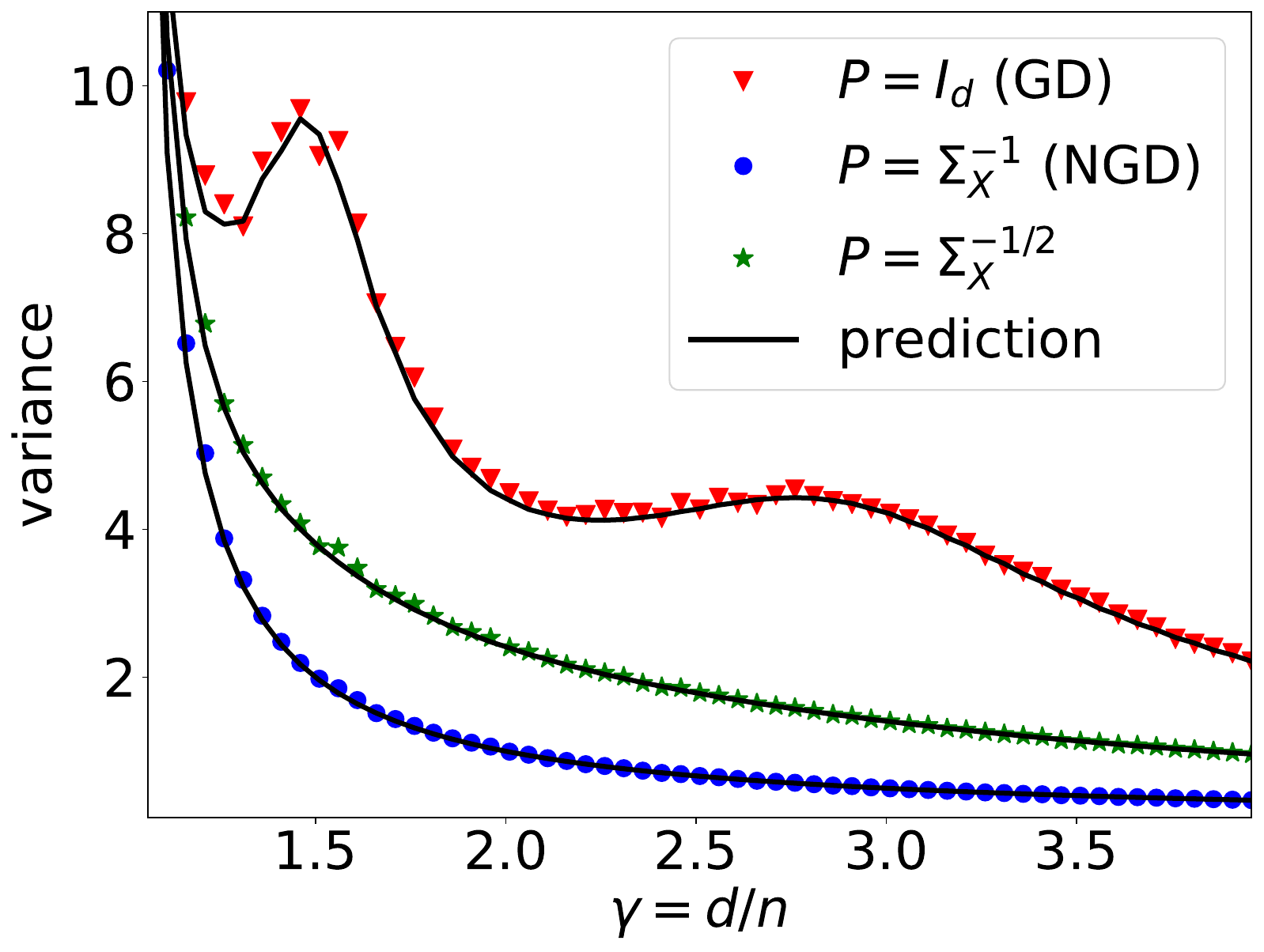}} 
\small (a) variance.
\end{minipage}
\begin{minipage}[t]{0.4\linewidth}
\centering
{\includegraphics[width=0.96\textwidth]{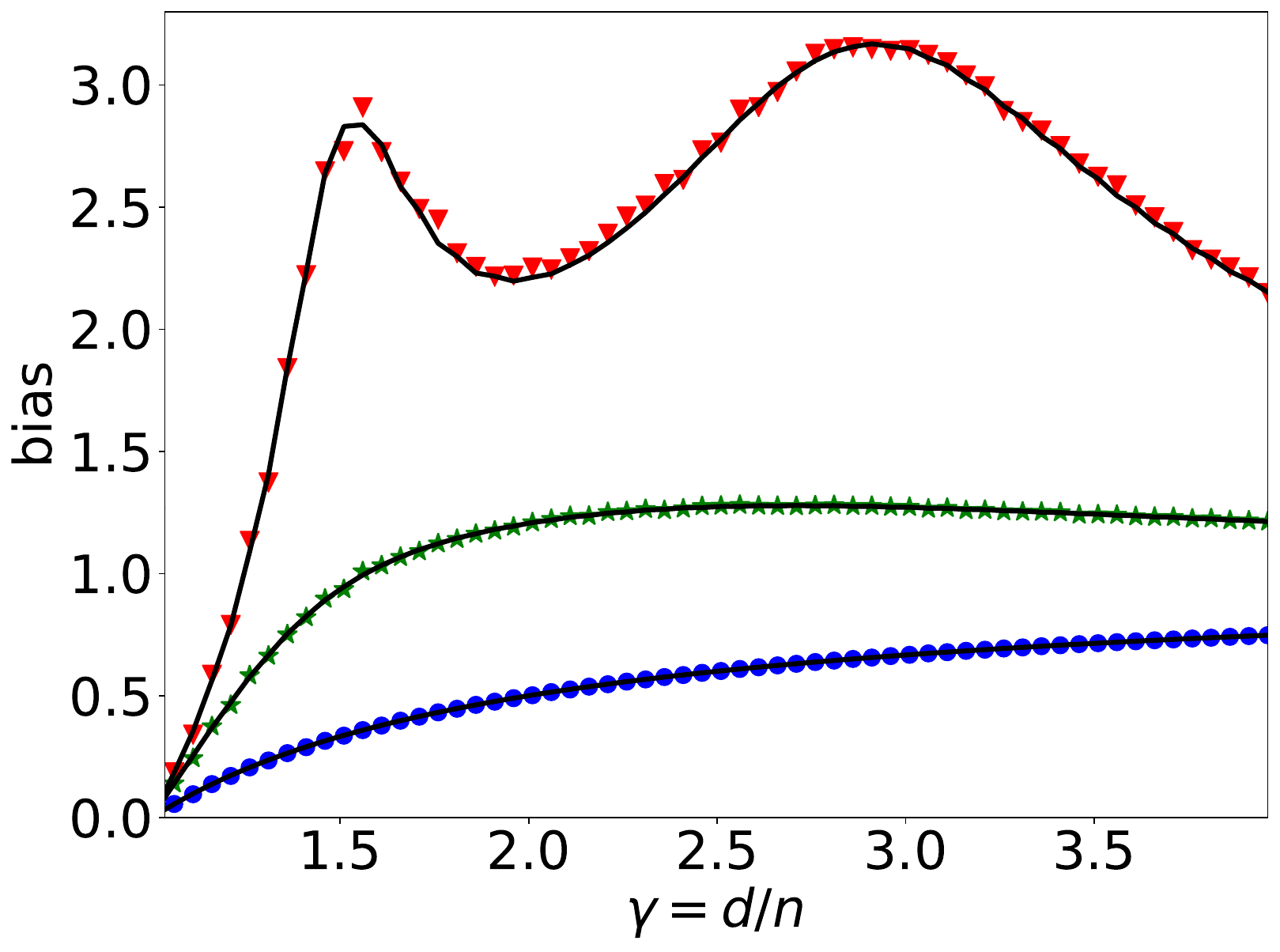}}
\small (b) bias (well-specified).
\end{minipage} 
\vspace{-0.15cm}
 \caption{\small Illustration of the ``multiple-descent'' curve of the risk for $\gamma>1$. We take $n=300$, eigenvalues of $\bSigma_\bX$ as three equally-spaced point masses with $\kappa_X = 5000$ and $\norm{\bSigma_\bX}_F^2 = d$, and $\bSigma_{\btheta}=\bSigma_\bX^{-1}$ (misaligned). Note that for GD, both the bias and the variance are highly non-monotonic for $\gamma>1$.}
\label{fig:non-monotone}
\vspace{-0.5mm}
\end{figure}   

\paragraph{Additional Figures for Section~\ref{sec:risk} and \ref{sec:bias-variance}.} We include additional figures on (a) well-specified bias when $\bSigma_{\btheta}=\bI_d$ (GD is optimal); (b) misspecified bias under unobserved features (predicted by Corollary~\ref{coro:misspecify}); (c) bias-variance tradeoff by interpolating between preconditioners (SNR=5). As shown in Figure~\ref{fig:ridgeless-2} and \ref{fig:ridgeless-SF}, in all cases the experimental values match the theoretical predictions.

\begin{figure}[!htb]
\vspace{-1mm}
\centering
\begin{minipage}[t]{0.31\linewidth}
\centering
{\includegraphics[width=0.99\textwidth]{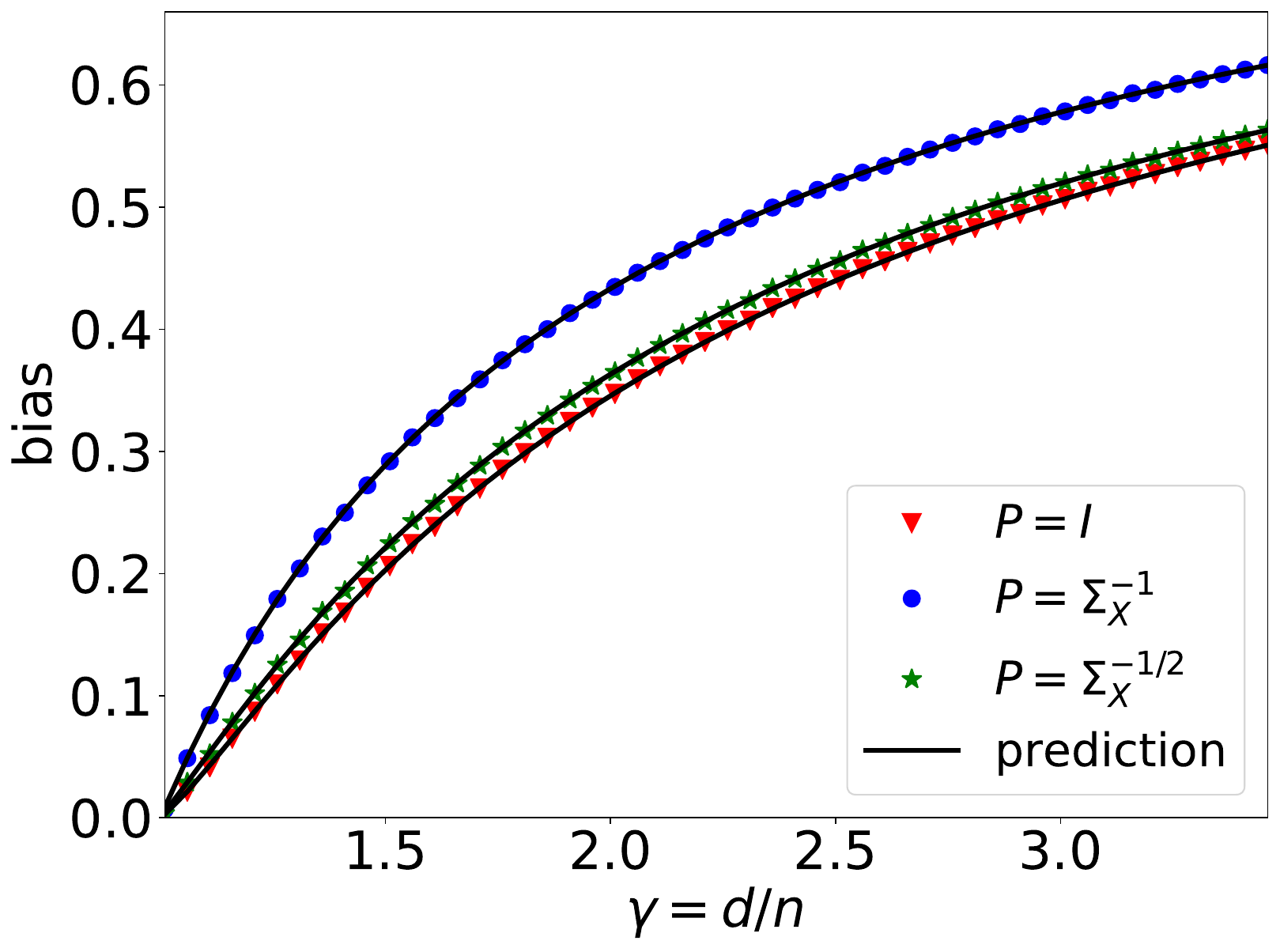}} 
\small (a) well-specified bias (aligned).
\end{minipage}
\begin{minipage}[t]{0.32\linewidth}
\centering
{\includegraphics[width=0.99\textwidth]{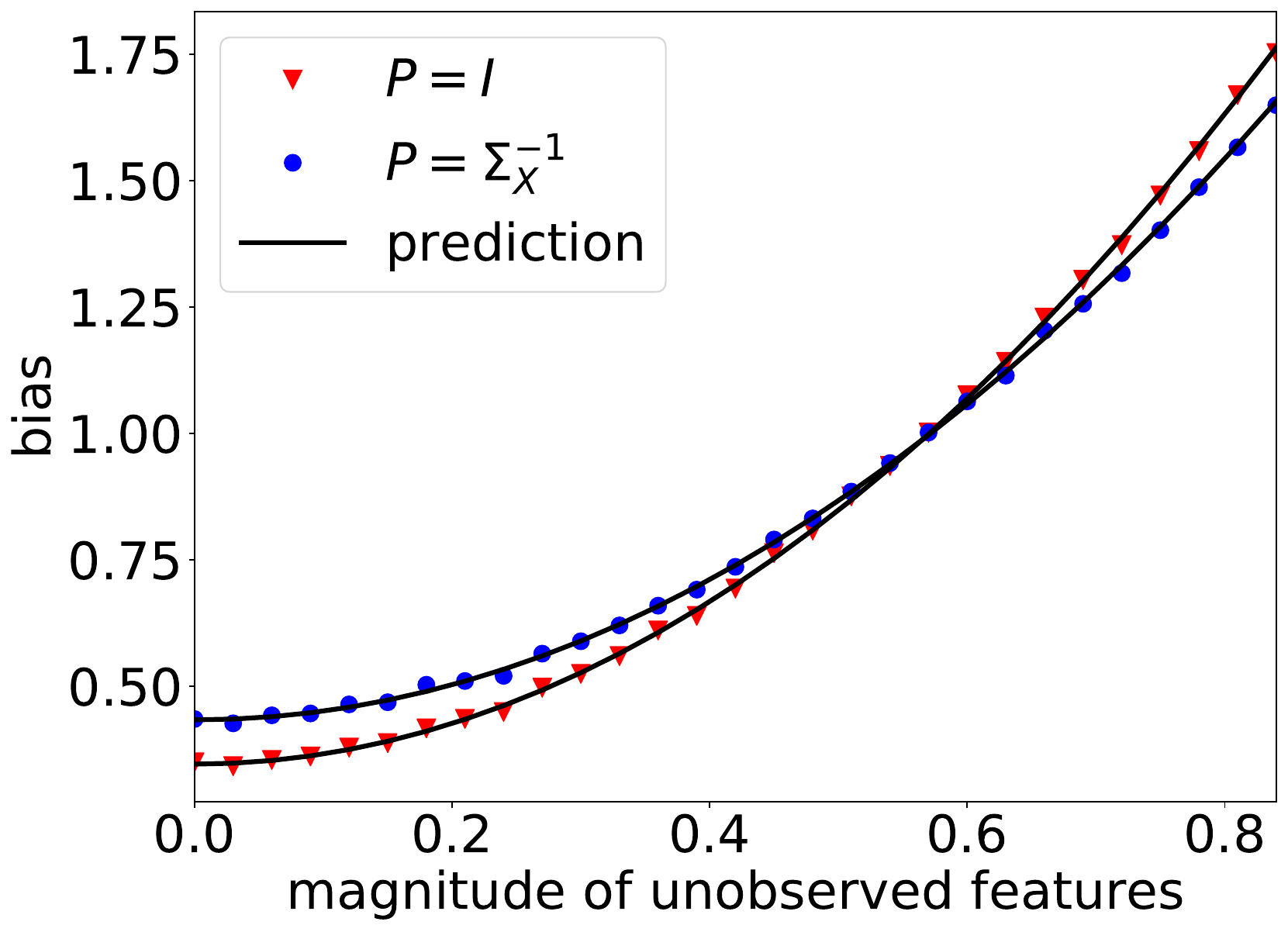}}
\small (b) misspecified bias \\(unobserved features).
\end{minipage} 
\begin{minipage}[t]{0.32\linewidth}
\centering
{\includegraphics[width=0.99\textwidth]{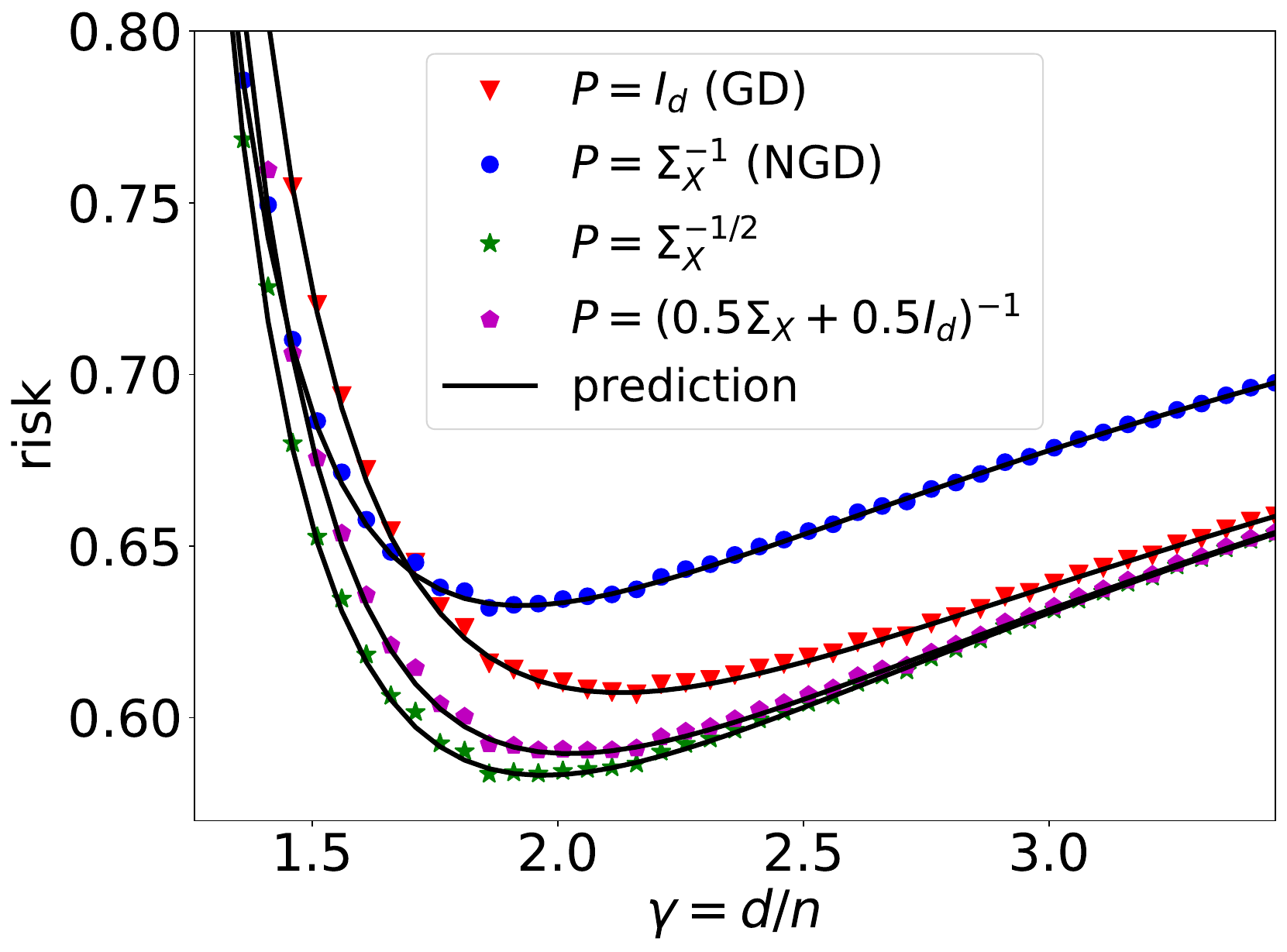}}
\small (c) bias-variance tradeoff.
\end{minipage}
\vspace{-0.15cm}
\caption{\small 
We set eigenvalues of $\bSigma_\bX$ as a uniform distribution with $\kappa_X = 20$ and $\norm{\bSigma_\bX}_F^2 = d$.
} 
\label{fig:ridgeless-2}
\vspace{-0.5mm}
\end{figure}

\begin{figure}[!htb]
% \vspace{-1mm}
\centering
\begin{minipage}[t]{0.325\linewidth}
\centering
{\includegraphics[width=0.99\textwidth]{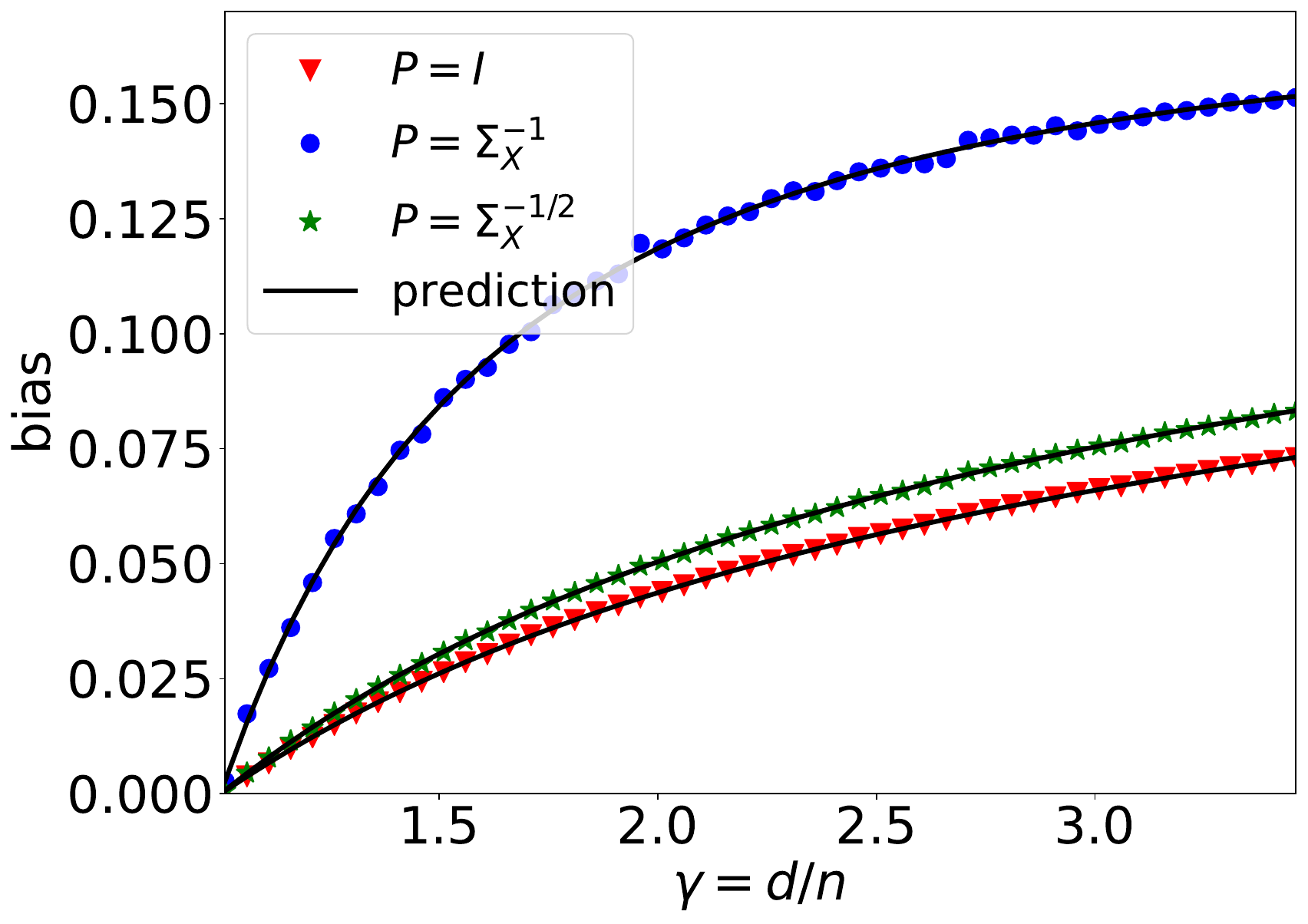}}  
\small (a) well-specified bias (aligned).
\end{minipage}
\begin{minipage}[t]{0.31\linewidth}
\centering
{\includegraphics[width=0.99\textwidth]{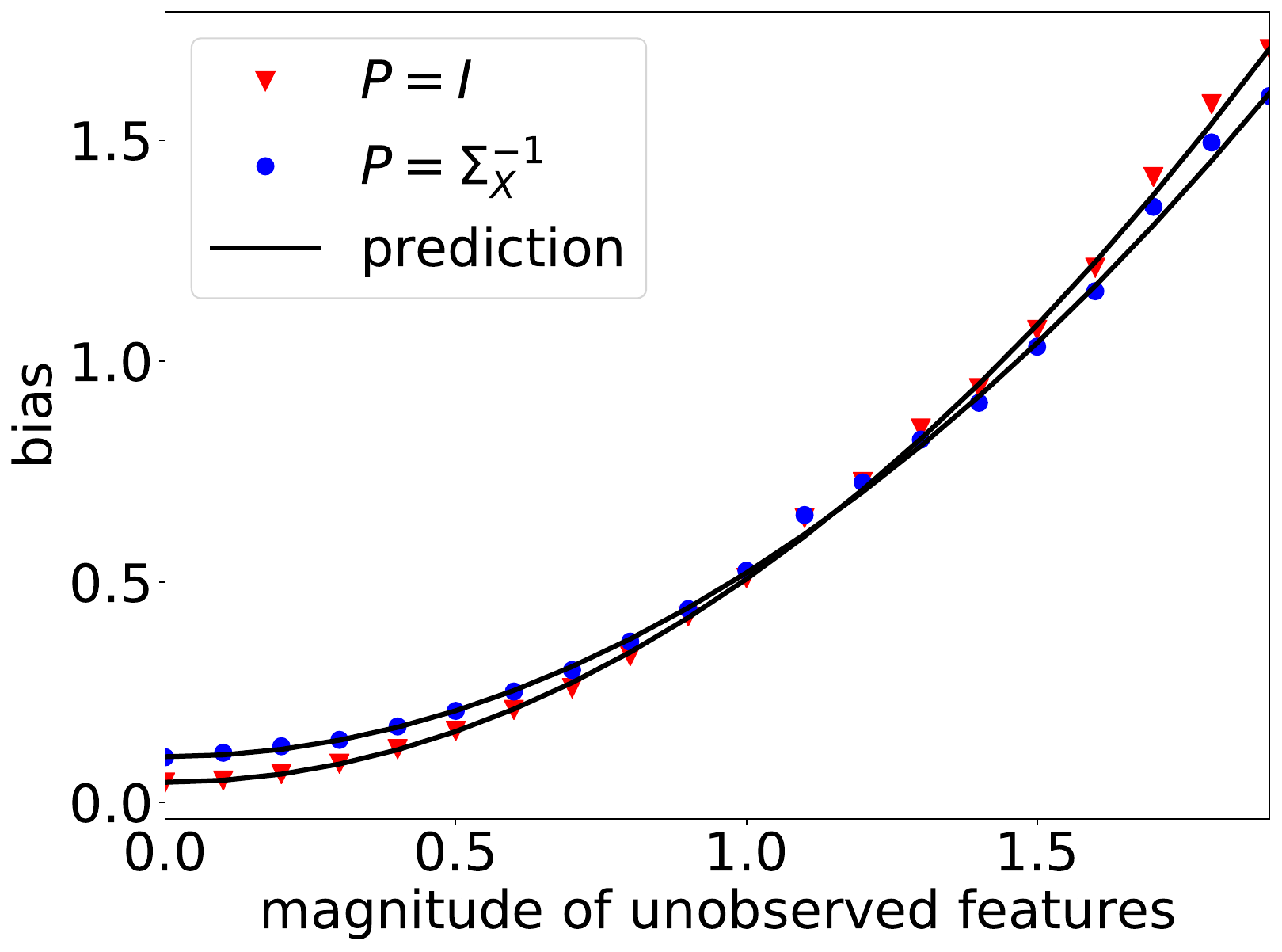}} 
\small (b) misspecified bias \\(unobserved features).
\end{minipage} 
\begin{minipage}[t]{0.315\linewidth}
\centering
{\includegraphics[width=0.99\textwidth]{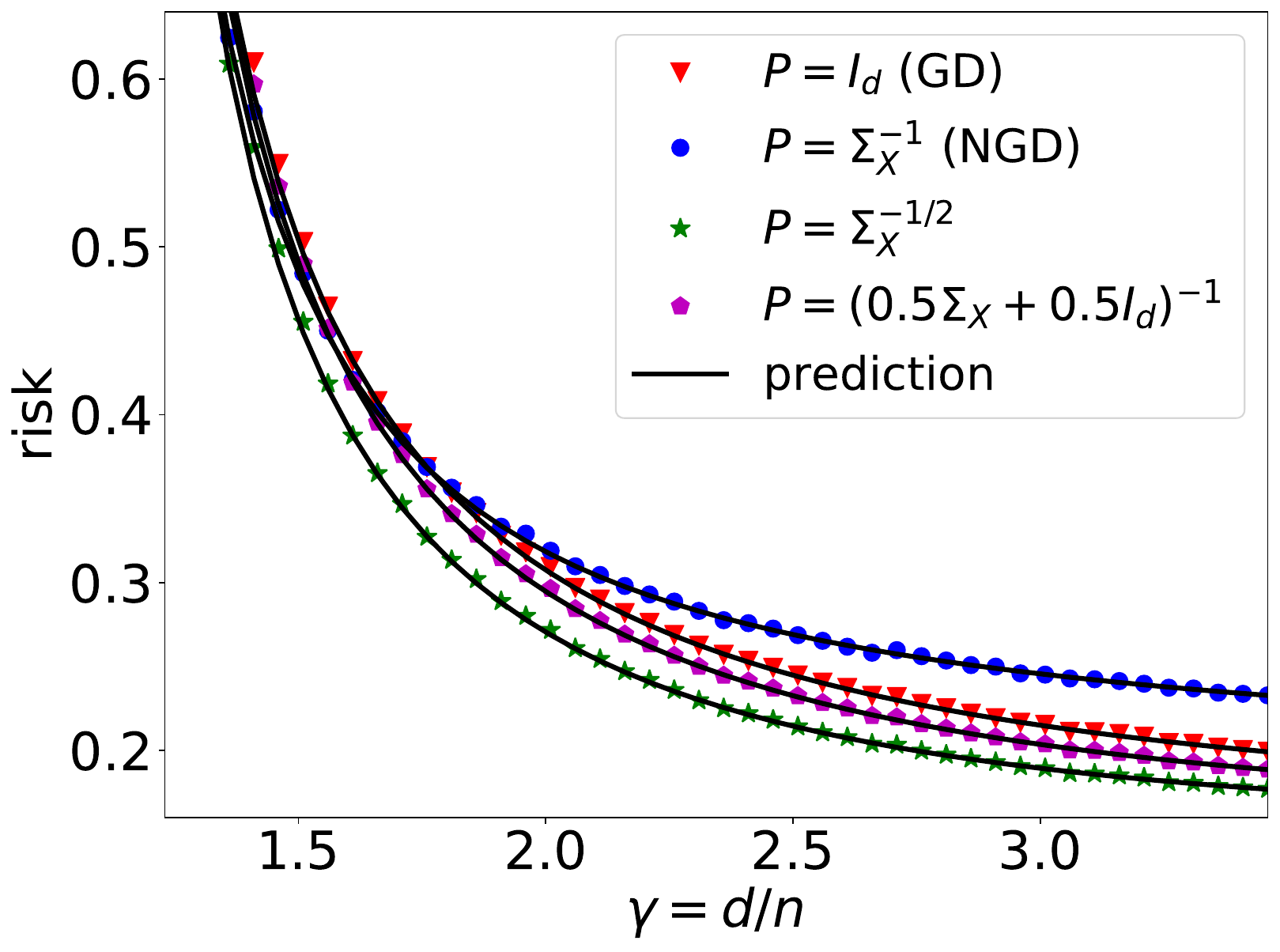}}
\small (c) bias-variance tradeoff.
\end{minipage} 
\vspace{-0.15cm} 
\caption{\small We construct eigenvalues of $\bSigma_\bX$ with a polynomial decay: $\lambda_i(\bSigma_\bX) = i^{-1}$ and then rescale the eigenvalues such that $\kappa_X=500$ and $\norm{\bSigma_\bX}_F^2 = d$. } 
\label{fig:ridgeless-SF}
\vspace{-1.5mm}  
\end{figure}    

\paragraph{Early Stopping Risk.} Figure~\ref{fig:early_stop-1} compares the stationary risk with the optimal early stopping risk under varying misalignment level. 
We set $\bSigma_{\bt} = \bSigma_\bX^{r}$ and vary $r$ from 0 to -1. As discussed in Section~\ref{subsec:bias-well_spec} smaller $\alpha$ entails more ``misaligned'' teacher, and vice versa. Note that as the problem becomes more misaligned, NGD achieves lower stationary and early stopping risk. 

Figure~\ref{fig:early_stop-2} reports the optimal early stopping risk under misspecification (same trend can be obtained when the x-axis is label noise). In contrast to the stationary risk (Figure~\ref{fig:misspecification_nonlinear}), GD can be advantageous under early stopping even with large extent of misspecification (for isotropic teacher). 
This aligns with our finding in Section~\ref{subsec:early-stop} that early stopping reduces the variance and the misspecified bias.
 
\begin{figure}[!htb]
% \vspace{-0.2cm} 
\centering
\begin{minipage}[t]{0.4\linewidth}
\centering
{\includegraphics[width=0.94\textwidth]{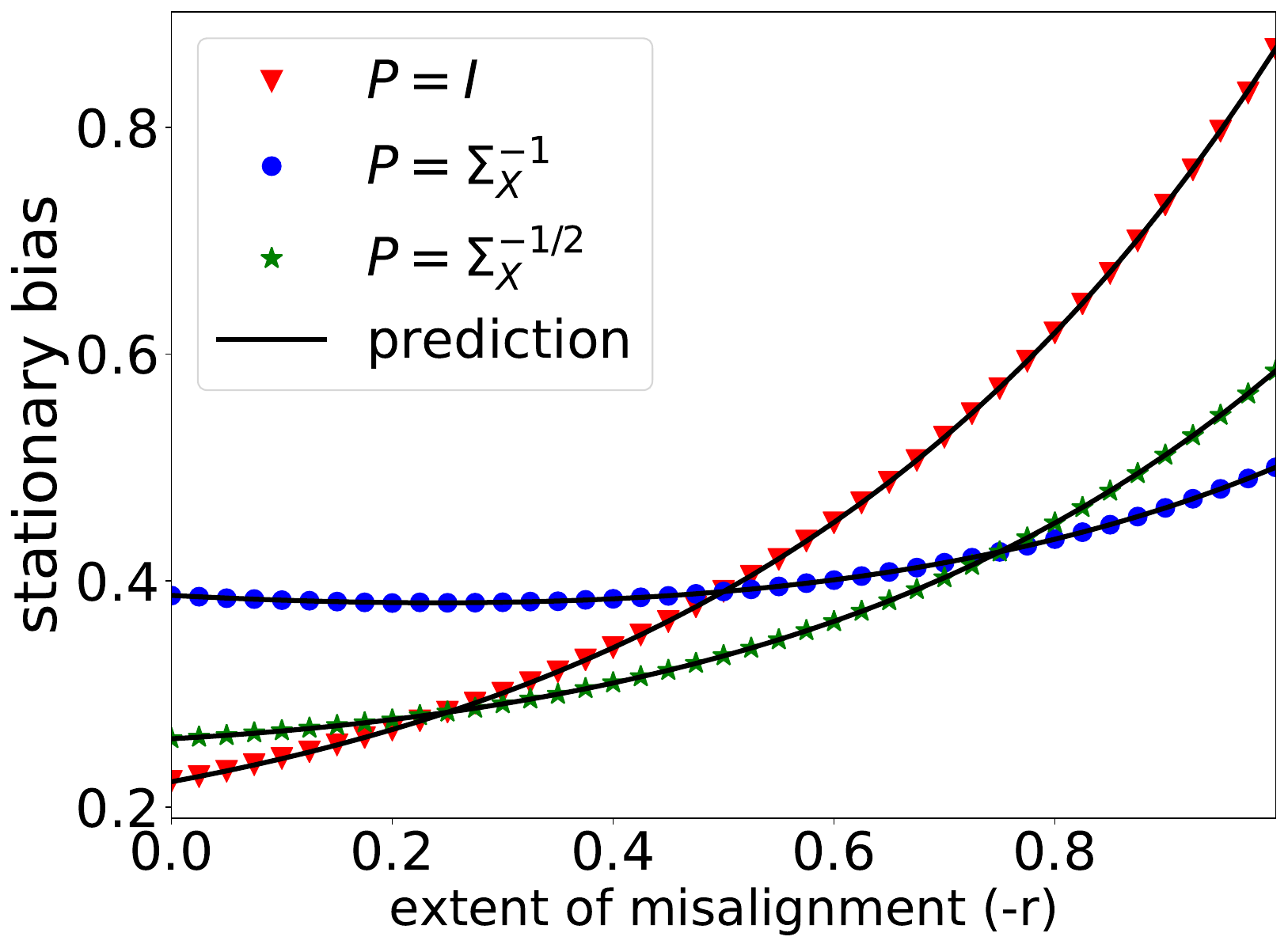}} 
\small (a) stationary risk.
\end{minipage}
\begin{minipage}[t]{0.4\linewidth}
\centering
{\includegraphics[width=0.94\textwidth]{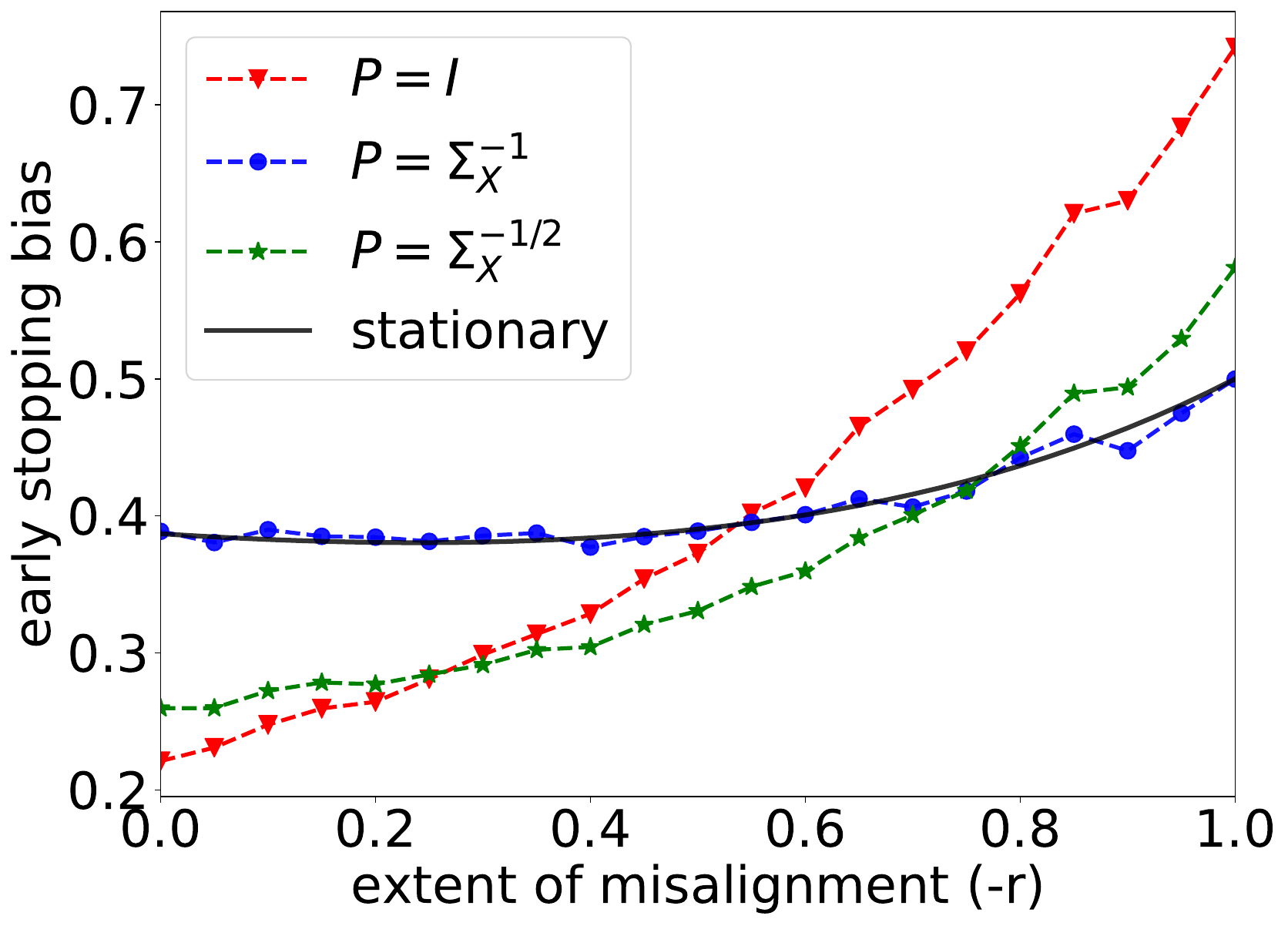}}
\small (b) optimal early stopping risk.
\end{minipage} 
\vspace{-0.15cm}
 \caption{\small Well-specified bias against different extent of ``alignment''. We set $n=300$, eigenvalues of $\bSigma_\bX$ as two point masses with $\kappa_X = 20$, and take $\bSigma_{\bt} = \bSigma_\bX^{r}$ and vary $r$ from -1 to 0. (a) GD achieves lower bias when $\bSigma_{\bt}$ is isotropic, whereas NGD dominates when $\bSigma_\bX=\bSigma_{\bt}^{-1}$; also observe $\bP=\bSigma_\bX^{-1/2}$ (interpolates between GD and NGD) is advantageous in between. (b) optimal early stopping bias follows similar trend as stationary bias.}  
\label{fig:early_stop-1}
% \vspace{-0.1cm}
\end{figure}

\begin{figure}[!htb]
\vspace{-0.1mm} 
\centering
\begin{minipage}[t]{0.4\linewidth}
\centering
{\includegraphics[width=0.9\textwidth]{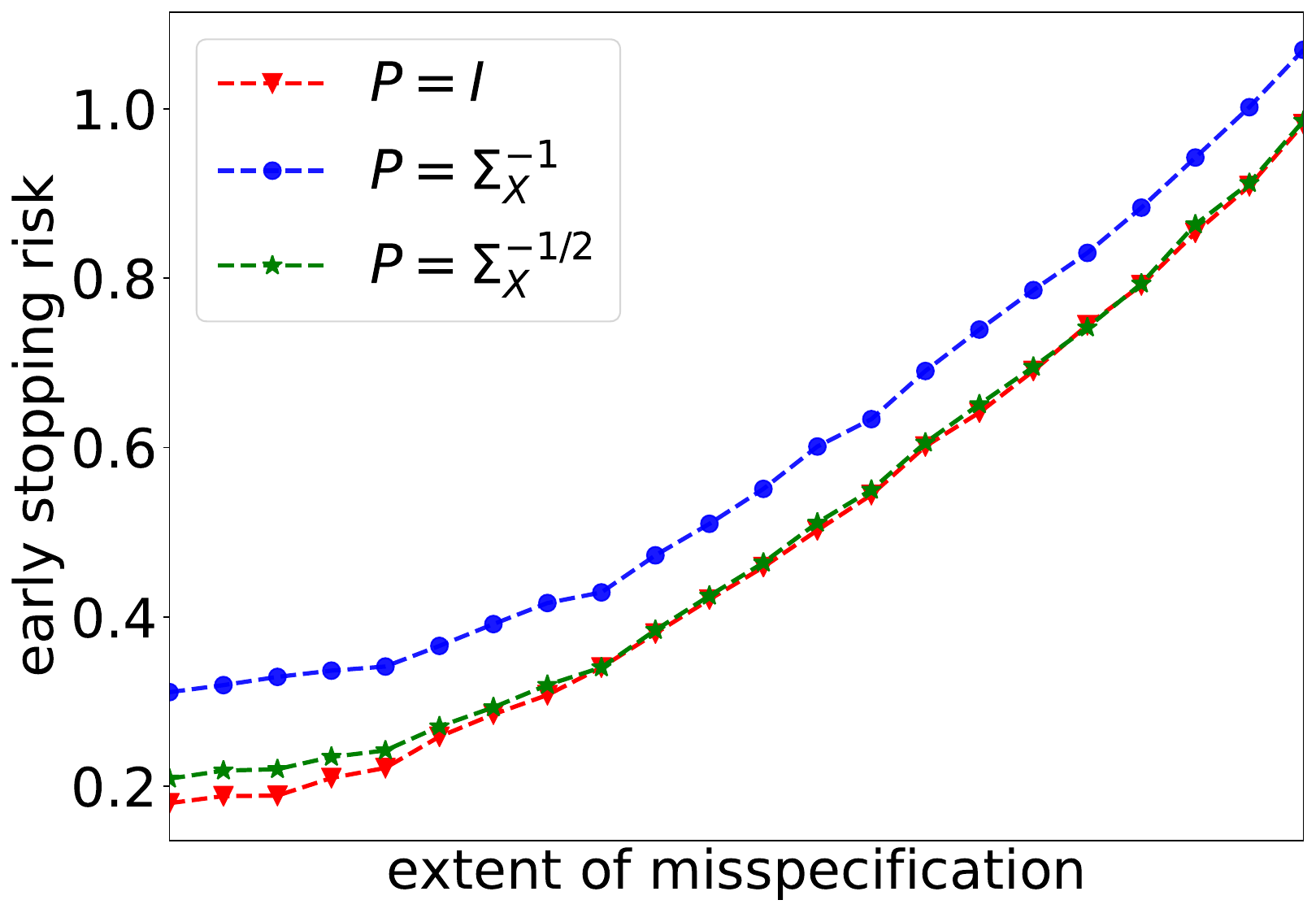}}\\
\small (a) optimal early stopping risk \\
(aligned \& misspecified).
\end{minipage}
\begin{minipage}[t]{0.4\linewidth}
\centering
{\includegraphics[width=0.9\textwidth]{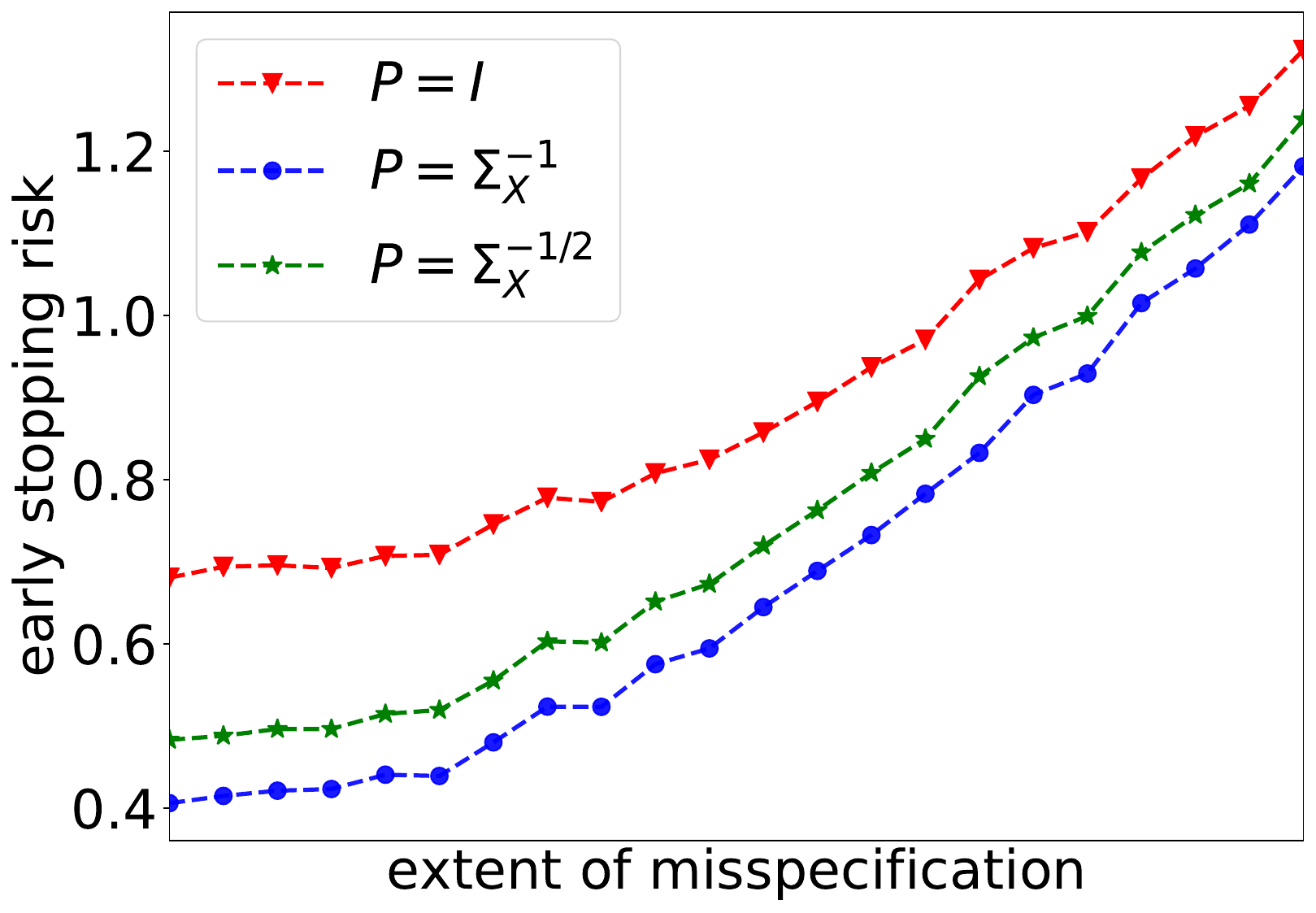}}\\
\small (b) optimal early stopping risk \\
(misaligned \& misspecified).
\end{minipage} 
\vspace{-0.15cm}
 \caption{\small Optimal early stopping risk vs. increasing model misspecification. We follow the same setup as Figure~\ref{fig:ridgeless-1}(c). (a) $\bSigma_{\bt} = \bI_d$ (favors GD); unlike Figure~\ref{fig:ridgeless-1}(c), GD has lower early stopping risk even under large extent of misspecification. (b) $\bSigma_{\bt} = \bSigma_\bX^{-1}$ (favors NGD); NGD is also advantageous under early stopping.} 
\label{fig:early_stop-2}
% \vspace{-0.3cm}
\end{figure} 

}
 
\subsection{RKHS Setting}
\label{subsec:additional_figures_RKHS}
We simulate the optimization in the coordinates of RKHS via a finite-dimensional approximation (using extra unlabeled data). In particular, we consider the teacher model in the form of $f^*(\bx) = \sum_{i=1}^N h_i\mu_i^r \phi_i(\bx)$ for square summable $\{h_i\}_{i=1}^N$, in which $r$ controls the ``difficulty'' of the learning problem. We find $\{\mu_i\}_{i=1}^N$ and $\{\phi_i\}_{i=1}^N$ by solving the eigenfunction problem for some kernel $k$. The student model takes the form of $f(\bx) = \sum_{i=1}^N \frac{a_i}{\sqrt{\mu_i}}\phi_i(\bx)$ and we optimize the coefficients $\{a_i\}_{i=1}^N$ via the preconditioned update \eqref{eq:RKHS-update}. 
We set $n=1000$, $d=5$, $N=2500$ and consider the inverse multiquadratic (IMQ) kernel:  $k(\bx,\by) = \frac{1}{\sqrt{1+\norm{\bx-\by}_2^2}}$. 

Recall that Theorem~\ref{theo:RKHS} suggests that for small $r$, i.e. ``difficult'' problem, the damping coefficient $\lambda$ would need to be small (which makes the update NGD-like), and vice versa. This result is (qualitatively) supported by Figure~\ref{fig:RKHS}, from which we can see that small $\lambda$ is beneficial when $r$ is small, and vice versa. 
We remark that this observed trend is rather fragile and sensitive to various hyperparameters, and we leave a comprehensive characterization of this observation as future work.

\begin{figure}[!htb]
% \vspace{-10mm} 
\centering
\begin{minipage}[t]{0.4\linewidth}
\centering
{\includegraphics[width=0.91\textwidth]{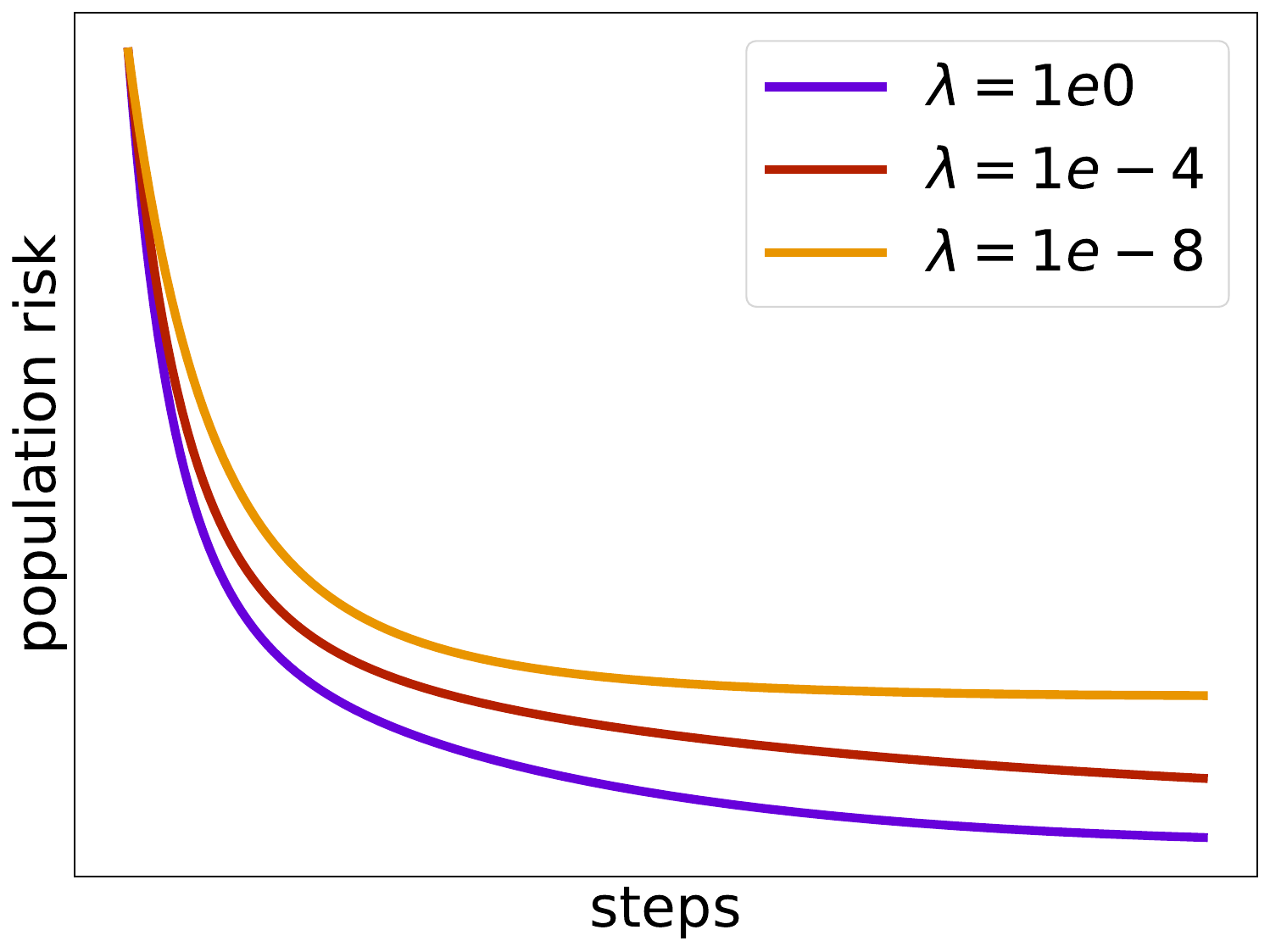}} \\ \vspace{-0.10cm}
\small (a) $r=3/4$.
\end{minipage}
\begin{minipage}[t]{0.4\linewidth}
\centering
{\includegraphics[width=0.91\textwidth]{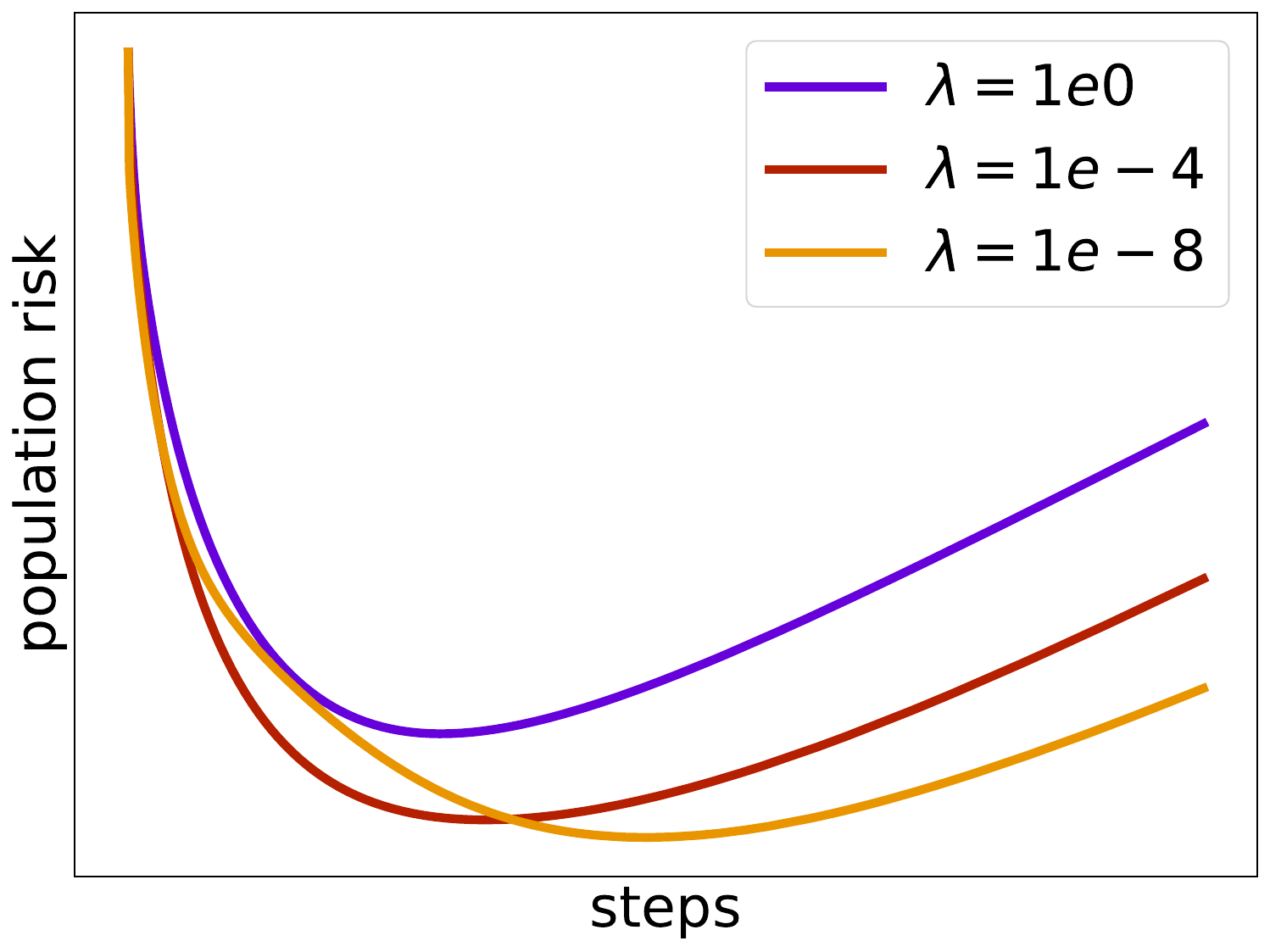}} \\ \vspace{-0.10cm} 
\small (b) $r=1/4$.
\end{minipage} 
\vspace{-0.1cm}
 \caption{\small Population risk of the preconditioned update in RKHS that interpolates between GD and NGD. We use the IMQ kernel and set $n=1000$, $d=5$, $N=2500$, $\sigma^2=5\times 10^{-4}$. 
 The x-axis has been rescaled for each curve and thus convergence speed is not directly comparable. Note that (a) large $\lambda$ (i.e., GD-like update) is beneficial when $r$ is large, and (b) small $\lambda$ (i.e., NGD-like update) is beneficial when $r$ is small.
 }
\label{fig:RKHS}
% \vspace{-0.1cm}
\end{figure}

\subsection{Neural Network}
\label{subsec:additional_figures_NN}
\paragraph{Label Noise.} In Figure~\ref{fig:additional_1_neural_networks}, (a) we observe the same phenomenon on CIFAR-10 that NGD generalizes better as more label noise is added to the dataset Figure~\ref{fig:additional_1_neural_networks} (b) shows that in all cases with varying amounts of label noise, the early stopping risk is worse than that of GD. This agrees with the observation in Section~\ref{sec:bias-variance} and Figure~\ref{fig:early_stop-2}(a) 
that early stopping can potentially favor GD due to reduced variance.

\vspace{-0.3cm} 
\paragraph{Misalignment.} In Figure~\ref{fig:additional_1_neural_networks}(c)(d) we confirm the finding in Proposition~\ref{prop:early-stop} and Figure~\ref{fig:early_stop-1}(b) in neural networks under synthetic data: we consider 50-dimensional Gaussian input, and both the teacher and the student model are two-layer ReLU networks with 50 hidden units. We construct the teacher by perturbing the initialization of the student as described in Section~\ref{sec:experiment}. As $r$ approaches -1 (problem more ``misaligned''), NGD achieves lower early stopping risk (Figure~\ref{fig:additional_1_neural_networks}(d)), whereas GD dominates the early stopping risk in less misaligned setting (~\ref{fig:additional_1_neural_networks}(c)). 
We note that this phenomenon is difficult to observe in practical neural network training on real-world data, which may be partially due to the fragility of the analogy between neural nets and linear models, especially under NGD (discussed in Appendix~\ref{subsec:implicit_bias_appendix}).

\begin{figure}[!htb] 
\centering
\vspace{-0.5mm}
\begin{minipage}[t]{0.24\linewidth}
\centering
{\includegraphics[width=1.03\textwidth]{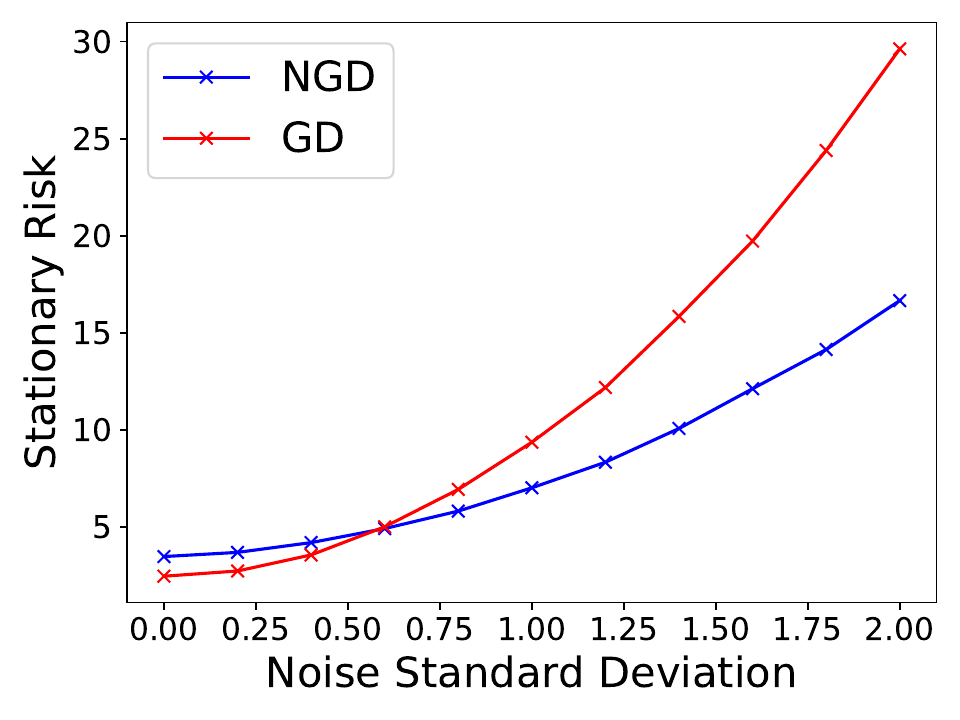}} \\
\small (a) stationary risk (CIFAR).
\end{minipage}
\begin{minipage}[t]{0.24\linewidth}
\centering
{\includegraphics[width=1.03\textwidth]{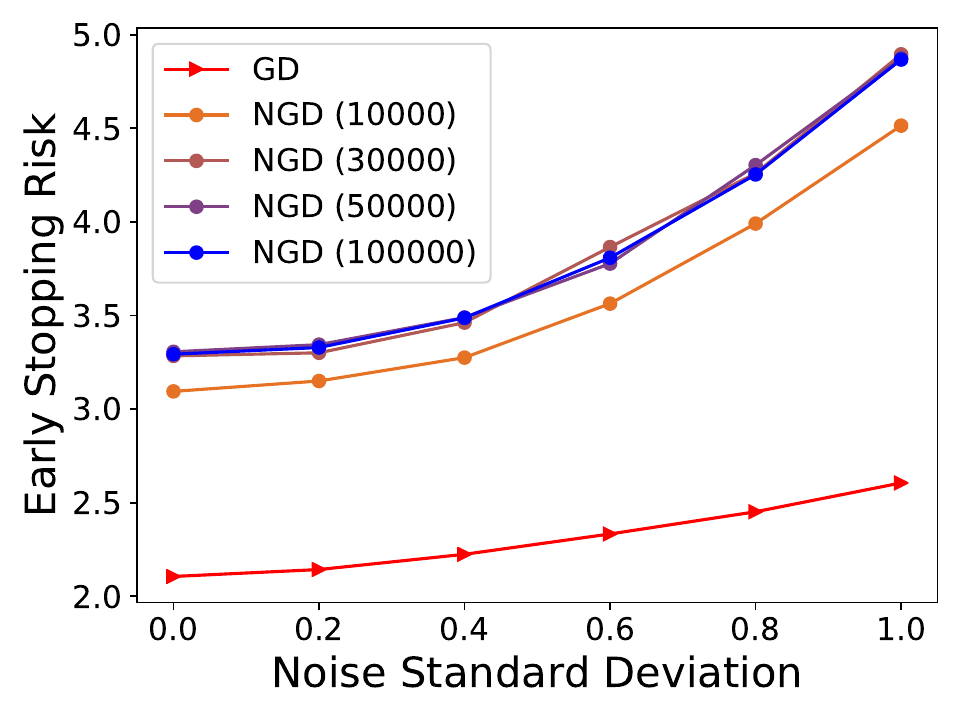}} \\
\small (b) optimal early stopping risk (CIFAR).
\end{minipage}
\begin{minipage}[t]{0.24\linewidth}
\centering
{\includegraphics[width=0.99\textwidth]{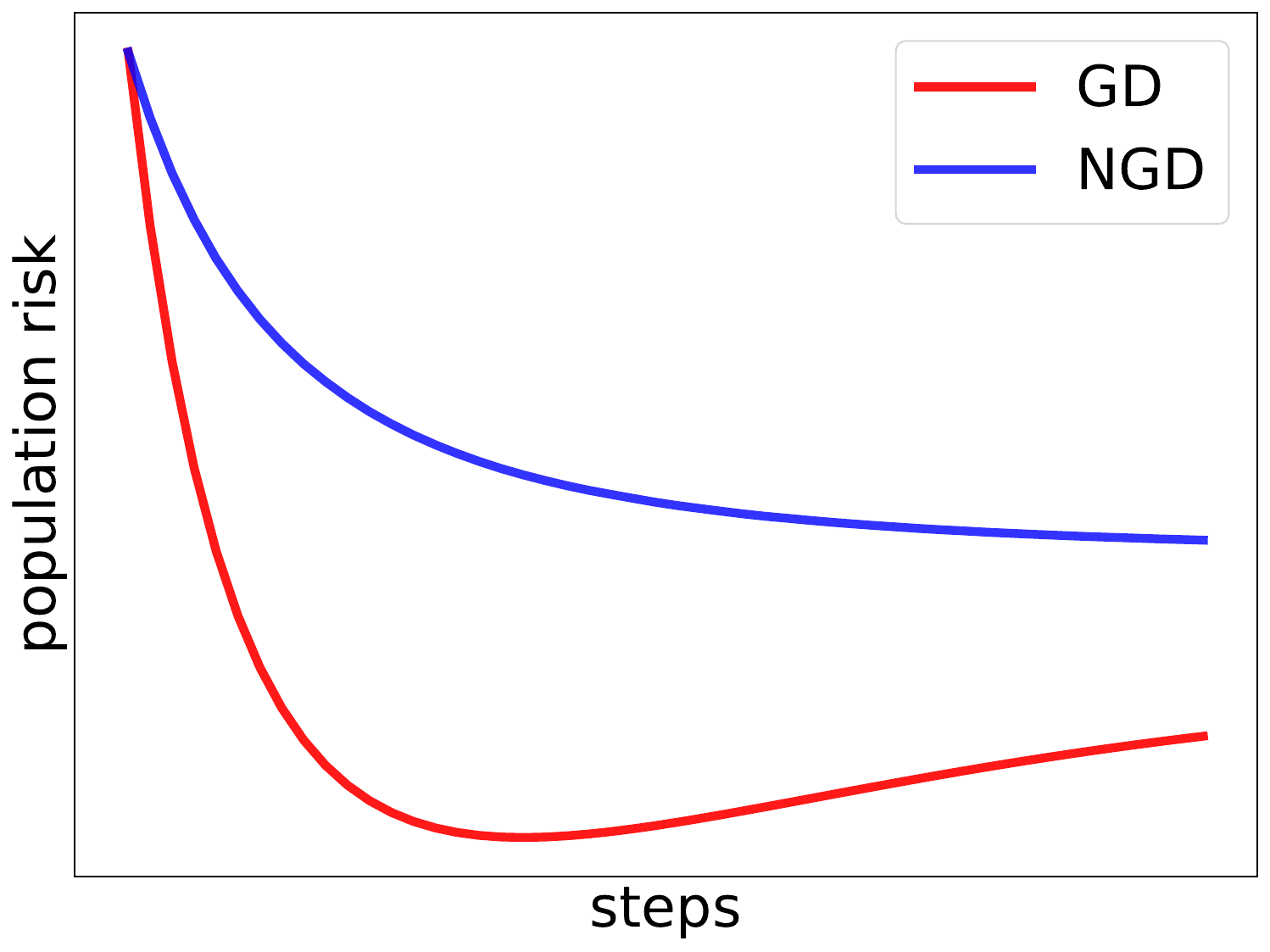}} \\
\small (c) $r=-1/2$ (synthetic).
\end{minipage}
\begin{minipage}[t]{0.24\linewidth}
\centering
{\includegraphics[width=0.99\textwidth]{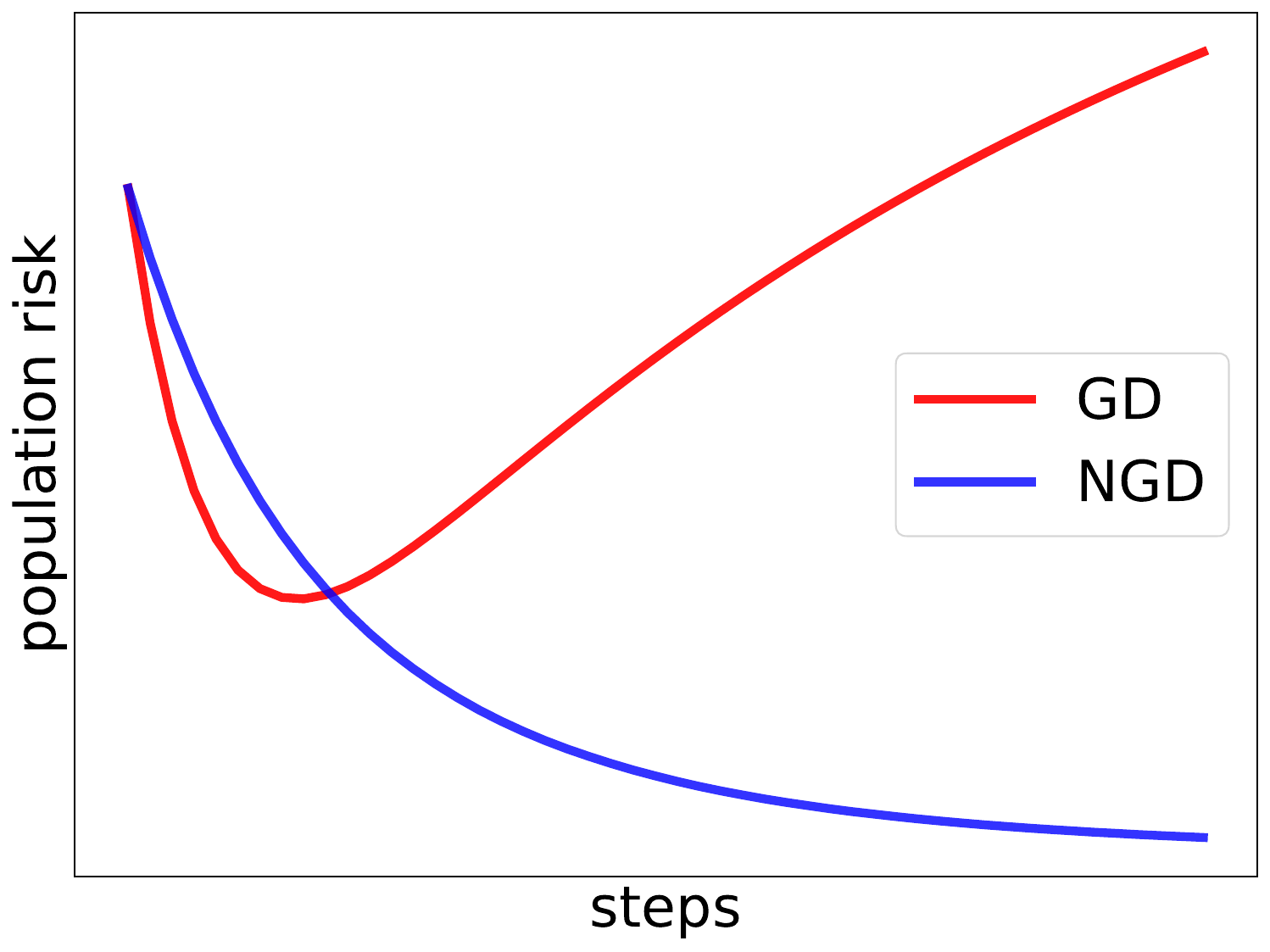}} \\
\small (d) $r=-3/4$ (synthetic).
\end{minipage}
\vspace{-0.15cm}
\caption{\small (a)(b) Additional label noise experiment on CIFAR-10.
(c)(d) Population risk of two-layer neural networks in the misalignment setup (noiseless) with synthetic Gaussian data. We set $n=200$, $d=50$, the damping coefficient $\lambda=10^{-6}$, and both the student and the teacher are two-layer ReLU networks with 50 hidden units. The x-axis and the learning rate have been rescaled for each curve (i.e., optimization speed not comparable). When $r$ is sufficiently small, NGD achieves lower early stopping risk, whereas GD is beneficial for larger $r$.
}
\label{fig:additional_1_neural_networks}
% \vspace{-0.3cm}
\end{figure}

\bigskip

}

% \newpage
\section{Proofs and Derivations}
\label{sec:proof}
{
\allowdisplaybreaks

\subsection{Missing Derivations in Section~\ref{sec:risk}}
\label{subsec:proof_derivations}
\paragraph{Gradient Flow of Preconditioned Updates.} 
Given positive definite $\bP$ and $\gamma>1$, it is clear that the gradient flow solution at time $t$ can be written as
\begin{align*}
    \bt_\bP(t) = \bP\bX^\top\left[\bI_n-\exp\left(-\frac{t}{n}\bX \bP\bX^\top\right) \right](\bX\bP\bX^\top)^{-1}\by.
\end{align*}
Taking $t\to\infty$ yields the stationary solution $\hbt_\bP = \bP\bX^\top(\bX\bP\bX^\top)^{-1}\by$. We remark that the damped inverse of the sample Fisher $\bP=(\bX\bX^\top + \lambda\bI_d)^{-1}$ leads to the same minimum-norm solution as GD $\hbt_\bI = \bX^\top(\bX\bX^\top)^{-1}\by$ since $\bP\bX^\top$ and $\bX$ share the same eigenvectors.
On the other hand, when $\bP$ is the pseudo-inverse of the sample Fisher $(\bX\bX^\top)^\dagger$ which is not full-rank, the trajectory can be obtained via the variation of constants formula:
\begin{align*}
    \bt(t) = \left[ \frac{t}{n} \sum_{k=0}^{\infty} \frac{1}{(k+1)!} \left(-\frac{t}{n} \bX^\top(\bX\bX^\top)^{-1}\bX \right)^k\right] \bX^\top(\bX\bX^\top)^{-1} \by, 
\end{align*}
for which taking the large $t$ limit also yields the minimum-norm solution $\bX^\top(\bX\bX^\top)^{-1}\by$.

\paragraph{Minimum $\norm{\bt}_{\bP^{-1}}$ Norm Interpolant.}
For positive definite $\bP$ and the corresponding stationary solution $\hbt_\bP = \bP\bX^\top(\bX\bP\bX^\top)^{-1}\by$, note that given any other interpolant $\hbt'$, we have $(\hbt_\bP-\hbt')\bP^{-1}\hbt_\bP = 0$ because both $\hbt_\bP$ and $\hbt'$ achieves zero empirical risk. 
Therefore, $\|\hbt'\|_{\bP^{-1}}^2 - \|\hbt_\bP\|_{\bP^{-1}}^2 =  \|\hbt'-\hbt_\bP\|_{\bP^{-1}}^2 \ge 0$. This confirms that $\hbt_{\bP}$ is the unique minimum $\norm{\bt}_{\bP^{-1}}$ norm solution. 

\subsection{Proof of Theorem~\ref{theo:variance}}
\label{subsec:proof_variance}
\begin{proof}
By the definition of the variance term and the stationary $\hbt$,
\begin{align*}
    V(\hbt) = \Tr{\Cov(\hbt)\bSigma_\bX} 
=
    \sigma^2\Tr{\bP\bX^\top(\bX \bP\bX^\top)^{-2}\bX \bP\bSigma_\bX}.
\end{align*}
Write $\bar{\bX} = \bX\bP^{1/2}$. Similarly, we define $\bSigma_{\bX\bP} = \bP^{1/2}\bSigma_\bX\bP^{1/2}$. The equation above thus simplifies to
\begin{align*}
    V(\hbt_\bP) 
=
    \sigma^2\Tr{\bar{\bX}^\top(\bar{\bX} \bar{\bX}^\top)^{-2}\bar{\bX} \bSigma_{\bX\bP}}.
\end{align*}
 
The analytic expression of the variance term follows from a direct application of cite[Thorem 4]{hastie2019surprises}, in which conditions on the population covariance are satisfied by (A2).

Taking the derivative of $m(-\lambda)$ yields
\begin{align*}
    m'(-\lambda) = \LL(\frac{1}{m^2(-\lambda)} - 
    \gamma\int \frac{\tau^2}{(1+\tau m(-\lambda))^2}\mathrm{d}\boldH_{\bX\bP}(\tau)\RR)^{-1}.
\end{align*}
Plugging the quantity into the expression of the variance (omitting the scaling $\sigma^2$ and constant shift),
\begin{align*}
    \frac{m'(-\lambda)}{m^2(-\lambda)} 
=
   \LL(1 - \gamma m^2(-\lambda) \int\frac{\tau^2}{(1+\tau m(-\lambda))^2}\mathrm{d}\boldH_{\bX\bP}(\tau)\RR)^{-1}.
\end{align*}
From the monotonicity of $\frac{x}{1+x}$ on $x>0$ or the Jensen's inequality we know that
\begin{align*}
    1 - \gamma\int\LL(\frac{\tau m(-\lambda)}{1+\tau m(-\lambda)}\RR)^2\mathrm{d}\boldH_{\bX\bP}(\tau)
\le
    1 - \gamma\LL(\int\frac{\tau m(-\lambda)}{1+\tau m(-\lambda)}\mathrm{d}\boldH_{\bX\bP}(\tau)\RR)^2.
\end{align*}
Taking $\lambda\to 0$ and omitting the scalar $\sigma^2$, the RHS evaluates to $1 - 1/\gamma$. We thus arrive at the lower bound $V\ge(\gamma-1)^{-1}$. Note that the equality is only achieved when $\boldH_{\bX\bP}$ is a point mass, i.e.~$\bP = \bSigma_\bX^{-1}$. In other words, the minimum variance is achieved by NGD. As a verification, the variance of the NGD solution $\hbt_{\boldF^{-1}}$ agrees with the calculation of the isotropic case in \cite[A.3]{hastie2019surprises}.
\end{proof}

\subsection{Proof of Corollary~\ref{coro:misspecify}}
\label{subsec:proof_misspecified}
\begin{proof} 
Via calculation similar to \cite[Section 5]{hastie2019surprises}, the bias can be decomposed as 
\begin{align*}
    \E\LL[B(\hbt_\bP)\RR] =& \E_{\bx,\hat{\bx},\btheta^*,\btheta^c}\LL[\LL(\bx^\top\bP\bX^\top\LL(\bX\bP\bX^\top\RR)^{-1}(\bX\btheta^* + \bX^c\btheta^c) - (\bx^\top\btheta^* + \hat{\bx}^\top\btheta^c)\RR)^2\RR] \\
\overset{(i)}{=}&
    \E_{\bx,\btheta^*}\LL[\LL(\bx^\top\bP\bX^\top\LL(\bX\bP\bX^\top\RR)^{-1}\bX\btheta^* - \bx^\top\btheta^*\RR)^2\RR] + \E_{\bx^c,\btheta^x}\LL[(\hat{\bx}^\top\btheta^c)^2\RR]\\
&+
    \E_{\bx,\btheta^c}\LL[\LL(\bx^\top\bP\bX^\top\LL(\bX\bP\bX^\top\RR)^{-1}\bX^c\btheta^c{\btheta^c}^\top{\bX^c}^\top\LL(\bX\bP\bX^\top\RR)^{-1}\bX\bP\bx\RR)^2\RR] \\
\overset{(ii)}{\to}&
    B_{\btheta}(\hbt_\bP) + \frac{1}{d^c}\Tr{\bSigma_\bX^c\bSigma_{\btheta}^c}(1 + V(\hbt_\bP)),
\end{align*}
where we used the independence of $\bx,\hat{\bx}$ and $\btheta^*,\btheta^c$ in (i), and (A1-3) as well as the definition of the well-specified bias $B_{\btheta}(\hbt_\bP)$ and variance $V(\hbt_\bP)$ in (ii).
\end{proof} 

\subsection{Proof of Theorem~\ref{theo:bias}}
\label{subsec:proof_bias}
\begin{proof}
By the definition of the bias term (note that $\bSigma_\bX$, $\bSigma_\btheta$, $\bP$ are all positive semi-definite),
\begin{align*}
    B(\hbt_\bP)
&=
    \E_{\btheta^*} \LL[\norm{\bP\bX^\top(\bX\bP\bX^\top)^{-1}\bX\bt_* - \btheta^*}_{\bSigma_\bX}^2\RR] \\
&=
    \frac{1}{d}\Tr{\bSigma_\btheta\LL(\bI_d - \bP\bX^\top(\bX\bP\bX^\top)^{-1}\bX \RR)^\top\bSigma_\bX\LL(\bI_d - \bP\bX^\top(\bX \bP\bX^\top)^{-1}\bX\RR)} \\
&\overset{(i)}{=}
    \frac{1}{d}\Tr{\bSigma_{\btheta/\bP}\LL(\bI_d - \bar{\bX}^\top(\bar{\bX}\bar{\bX}^\top)^{-1}\bar{\bX} \RR)^\top\bSigma_{\bX\bP}\LL(\bI_d - \bar{\bX}^\top(\bar{\bX}\bar{\bX}^\top)^{-1}\bar{\bX}\RR)} \\
&\overset{(ii)}{=}
    \lim_{\lambda\to 0_+} \frac{\lambda^2}{d} \Tr{\bSigma_{\btheta/\bP}\LL(\frac{1}{n}\bar{\bX}^\top\bar{\bX} + \lambda\bI_d \RR)^{-1}\bSigma_{\bX\bP}\LL(\frac{1}{n}\bar{\bX}^\top\bar{\bX} + \lambda\bI_d \RR)^{-1}} \\
&\overset{(iii)}{=}
    \lim_{\lambda\to 0_+} \frac{\lambda^2}{d} \Tr{\LL(\frac{1}{n}\hat{\bX}^\top\hat{\bX} + \lambda\bSigma_{\btheta/\bP}^{-1}\RR)^{-2}\bSigma_{\btheta/\bP}^{-1/2}\bSigma_{\bX\bP}\bSigma_{\btheta/\bP}^{-1/2}},
    % \numberthis
    % \label{eq:bias_derivation}
\end{align*}  
where we utilized (A3) and defined $\bar{\bX} = \bX\bP^{1/2}$, $\bSigma_{\bX\bP} = \bP^{1/2}\bSigma_\bX \bP^{1/2}$, $\bSigma_{\btheta/\bP} = \bP^{-1/2}\bSigma_\btheta \bP^{-1/2}$ in (i), applied the equality $(\boldA\boldA^\top)^\dagger\boldA = \lim_{\lambda\to 0} (\boldA^\top\boldA+\lambda\bI)^{-1}\boldA$ in (ii), and defined $\hat{\bX} = \bX\bP^{1/2}\bSigma_{\theta}^{-1/2}$ in (iii). To proceed, we first assume that $\bSigma_{\btheta/\bP}$ is invertible (i.e.~$\lambda_{\min}(\bSigma_{\btheta/\bP})$ is bounded away from 0) and observe the following relation via a leave-one-out argument similar to that in \cite{xu2019many},  
\begin{align*}
    &\frac{1}{d} \Tr{\frac{1}{n}\hat{\bX}^\top\hat{\bX}\LL(\frac{1}{n}\hat{\bX}^\top\hat{\bX} + \lambda\bSigma_{\btheta/\bP}^{-1}\RR)^{-2}}
    \numberthis
    \label{eq:intermediate}\\
\overset{(i)}{=}&
    \frac{1}{d}\sum_{i=1}^{n} \frac{\frac{1}{n}\hat{\bx}_i^\top \LL(\frac{1}{n}\hat{\bX}^\top\hat{\bX} + \lambda\bSigma_{\btheta/\bP}^{-1}\RR)_{\neg i}^{-2}\hat{\bx}_i}{\LL(1 + \frac{1}{n}\hat{\bx}_i^\top\LL(\frac{1}{n}\hat{\bX}^\top\hat{\bX} + \lambda\bSigma_{\btheta/\bP}^{-1}\RR)_{\neg i}^{-1}\hat{\bx}_i\RR)^2} \\
\overset{(ii)}{\underset{p}{\to}}&
    \frac{\frac{1}{d}\Tr{\LL(\frac{1}{n}\hat{\bX}^\top\hat{\bX} + \lambda\bSigma_{\btheta/\bP}^{-1}\RR)^{-2}\bSigma_{\btheta/\bP}^{-1/2}\bSigma_{\bX\bP}\bSigma_{\btheta/\bP}^{-1/2}}}{\LL(1 + \frac{1}{n}\Tr{\LL(\frac{1}{n}\bar{\bX}^\top\bar{\bX} + \lambda\bI_d\RR)^{-1}\bSigma_{\bX\bP}}\RR)^2},
\label{eq:leave-one-out}
\numberthis
\end{align*} 
where (i) is due to the Woodbury identity and we defined  $\LL(\frac{1}{n}\hat{\bX}^\top\hat{\bX} + \lambda\bSigma_{\btheta/\bP}^{-1}\RR)_{\neg i} = \frac{1}{n}\hat{\bX}^\top\hat{\bX} - \frac{1}{n}\hat{\bx}_i\hat{\bx}_i^\top + \lambda\bSigma_{\btheta/\bP}^{-1}$ which is independent to $\hat{\bx}_i$ (see \cite[Eq. 58]{xu2019many} for details), and in (ii) we used (A3), the convergence to trace \cite[Lemma 2.1]{ledoit2011eigenvectors} and its stability under low-rank perturbation (e.g., see \cite[Eq. 18]{ledoit2011eigenvectors}) which we elaborate below. In particular, denote $\hat{\bSigma} = \frac{1}{n}\hat{\bX}^\top\hat{\bX} + \lambda\bSigma_{\btheta/\bP}^{-1}$, for the denominator we have
\begin{align*}
    &\sup_i \LL|\frac{\lambda}{n}\Tr{\hat{\bSigma}^{-1}\bSigma_{\btheta/\bP}^{-1/2}\bSigma_{\bX\bP}\bSigma_{\btheta/\bP}^{-1/2}} - \frac{\lambda}{n}\Tr{\hat{\bSigma}_{\neg i}^{-1}\bSigma_{\btheta/\bP}^{-1/2}\bSigma_{\bX\bP}\bSigma_{\btheta/\bP}^{-1/2}}\RR| \\
\le&
    \frac{\lambda}{n}\norm{\bSigma_{\btheta/\bP}^{-1/2}\bSigma_{\bX\bP}\bSigma_{\btheta/\bP}^{-1/2}}_2 \sup_i\LL|\Tr{\hat{\bSigma}^{-1}\LL(\hat{\bSigma} - \hat{\bSigma}_{\neg i}\RR)\hat{\bSigma}_{\neg i}^{-1}}\RR| \\
\le&
    \frac{\lambda}{n}\norm{\bSigma_{\btheta/\bP}^{-1/2}\bSigma_{\bX\bP}\bSigma_{\btheta/\bP}^{-1/2}}_2 \norm{\hat{\bSigma}^{-1}}_2 \sup_i \norm{\hat{\bSigma}_{\neg i}^{-1}}_2 \Tr{\hat{\bSigma} - \hat{\bSigma}_{\neg i}}
\overset{(i)}{\to} O_p\LL(\frac{1}{n}\RR),
\end{align*}
where (i) is due to the definition of $\hat{\bSigma}_{\neg i}$ and (A1)(A3).
The result on the numerator can be obtained via a similar calculation, the details of which we omit.

Note that the denominator can be evaluated by previous results (e.g. \cite[Theorem 2.1]{dobriban2018high}) as follows,
\begin{align*}
    \frac{1}{n}\Tr{\LL(\frac{1}{n}\bar{\bX}^\top\bar{\bX} + \lambda\bI_d\RR)^{-1}\bSigma_{\bX\bP}}
\overset{a.s.}{\to}
    \frac{1}{\lambda m(-\lambda)} - 1.
\label{eq:dobriban}
\numberthis
\end{align*}

On the other hand, following the same derivation as \cite{dobriban2018high,hastie2019surprises}, \eqref{eq:intermediate} can be decomposed as
\begin{align*}
    &\frac{1}{d}\Tr{\frac{1}{n}\hat{\bX}^\top\hat{\bX}\LL(\frac{1}{n}\hat{\bX}^\top\hat{\bX} + \lambda\bSigma_{\btheta/\bP}^{-1}\RR)^{-2}} \\
=&
    \frac{1}{d}\Tr{\LL(\frac{1}{n}\bar{\bX}^\top\bar{\bX} + \lambda\bI_d\RR)^{-1}\bSigma_{\btheta/\bP}} - \frac{\lambda}{d} \Tr{\LL(\frac{1}{n}\bar{\bX}^\top\bar{\bX} + \lambda\bI_d\RR)^{-2}\bSigma_{\btheta/\bP}} \\
=&
    \frac{1}{d}\Tr{\LL(\frac{1}{n}\bar{\bX}^\top\bar{\bX} + \lambda\bI_d\RR)^{-1}\bSigma_{\btheta/\bP}} + \frac{\lambda}{d}\frac{\mathrm{d}}{\mathrm{d}\lambda}\Tr{\LL(\frac{1}{n}\bar{\bX}^\top\bar{\bX} + \lambda\bI_d\RR)^{-1}\bSigma_{\btheta/\bP}}. 
% \overset{(i)}{\to}&
%     \int \frac{g(\tau)}{1 + \tau m(-\lambda)}\mathrm{d}F_{\bX\bP} +  , 
\label{eq:intermetidate2}
\numberthis
\end{align*}
% \DW{TBD} 

We employ \cite[Theorem 1]{rubio2011spectral} to characterize \eqref{eq:intermetidate2}. In particular, For any deterministic sequence of matrices $\boldTheta_n\in\R^{d\times d}$ with finite trace norm, as $n,d\to\infty$ we have
\begin{align*}
    \Tr{\boldTheta_n\LL(\frac{1}{n}\bar{\bX}^\top\bar{\bX} -z\bI_d\RR)^{-1} - \boldTheta_n \LL(c_n(z) \bSigma_{\bX\bP} - z\bI_d\RR)^{-1}}\overset{a.s.}{\to} 0,
\end{align*}
in which $c_n(z)\to -z m(z)$ for $z\in\mathbb{C} \backslash \R^+$ and $m(z)$ is defined in Theorem~\ref{theo:variance} due to the dominated convergence theorem. By (A3) we are allowed to take $\boldTheta_n=\frac{1}{d}\bSigma_{\btheta/\bP}$. Thus we have
\begin{align*}
    \frac{\lambda}{d}\Tr{\bSigma_{\btheta/\bP}\LL(\frac{1}{n}\bar{\bX}^\top\bar{\bX} + \lambda\bI_d\RR)^{-1}} 
\to& 
    \frac{\lambda}{d}\Tr{\bSigma_{\btheta/\bP} \LL(\lambda m(-\lambda) \bSigma_{\bX\bP} +\lambda\bI_d\RR)^{-1}} \\
\overset{(i)}{=}&
    \E\LL[\frac{\upsilon_x\upsilon_\theta\upsilon_{xp}^{-1}}{1 + m(-\lambda)\upsilon_{xp}}\RR], \quad \forall \lambda > -c_l,
\numberthis
\label{eq:bias-term1}
\end{align*}
in which (i) is due to (A3), the fact that the LHS is almost surely bounded for $\lambda > -c_l$, where $c_l$ is the lowest non-zero eigenvalue of $\frac{1}{n}\bar{\bX}^\top\bar{\bX}$, and the application of the dominated convergence theorem. 
Differentiating \eqref{eq:bias-term1} (note that the derivative is also bounded on $\lambda > -c_l$) yields
\begin{align*}
    \frac{\lambda}{d}\frac{\mathrm{d}}{\mathrm{d}\lambda}\Tr{\LL(\frac{1}{n}\bar{\bX}^\top\bar{\bX} + \lambda\bI_d\RR)^{-1}\bSigma_{\btheta/\bP}}
\to
    \E\LL[\frac{\upsilon_x\upsilon_\theta\upsilon_{xp}^{-1}}{\lambda(1 + m(-\lambda)\upsilon_{xp})}-\frac{m'(-\lambda)\upsilon_x\upsilon_{\theta}}{(1+m(-\lambda)\upsilon_{xp})^2}\RR].
\numberthis
\label{eq:bias-term2} 
\end{align*}

Note that the numerator of \eqref{eq:leave-one-out} is the quantity of interest.
Combining \eqref{eq:intermediate} \eqref{eq:leave-one-out} \eqref{eq:dobriban} \eqref{eq:intermetidate2} \eqref{eq:bias-term1} \eqref{eq:bias-term2} and taking $\lambda\to 0$ yields the formula of the bias term. 
Finally, when $\bSigma_{\btheta/\bP}$ is not invertible, observe that if we increment all eigenvalues by some small $\epsilon>0$ to ensure invertibility $\bSigma_{\btheta/\bP}^\epsilon = \bSigma_{\btheta/\bP} + \epsilon\boldI_d$, \eqref{eq:bias} is bounded and also decreasing w.r.t.~$\epsilon$. Thus by the dominated convergence theorem we take $\epsilon\to 0$ and obtain the desired result.  
We remark that similar (but less general) characterization can also be derived based on \cite[Theorem 1.2]{ledoit2011eigenvectors} when the eigenvalues of $\bSigma_{\bX\bP}$ and $\bSigma_{\btheta/\bP}$ exhibit certain relations.

To show that $\bP = \boldU\diag{(\bolde_{\theta})}\boldU^\top$ achieves the lowest bias, first note that under the definition of random variables in (A3), our claimed optimal preconditioner is equivalent to $\upsilon_{xp}\overset{a.s.}{=}\upsilon_x\upsilon_\theta$. We therefore define an interpolation $\upsilon_\alpha = \alpha\upsilon_x\upsilon_\theta  + (1-\alpha)\bar{\upsilon}$ for some $\bar{\upsilon}$ and write the corresponding Stieltjes transform as $m_\alpha(-\lambda)$ and the bias term as $B_\alpha$. We aim to show that $\argmin_{\alpha\in[0,1]} B_\alpha = 1$. 

For notational convenience define $g_\alpha\triangleq m_\alpha(0) \upsilon_x\upsilon_\theta$ and $h_\alpha\triangleq m_\alpha(0) \upsilon_\alpha$. One can check that 
\begin{align*}
    B_\alpha = \E\LL[\frac{\upsilon_x\upsilon_{\theta}}{(1+h_\alpha)^2}\RR]\E\LL[\frac{h_\alpha}{(1+h_\alpha)^2}\RR]^{-1}; \quad 
    \frac{\mathrm{d}m_\alpha(-\lambda)}{\mathrm{d}\alpha} \Big|_{\lambda\to0} = \frac{m_\alpha(0)\E\LL[\frac{h_\alpha-g_\alpha}{(1+h_\alpha)^2}\RR]}{(1-\alpha)\E\LL[\frac{h_\alpha}{(1+h_\alpha)^2}\RR]}.
\end{align*}
We now verify that the derivative of $B_\alpha$ w.r.t.~$\alpha$ is non-positive for $\alpha\in[0,1]$. A standard simplification of the derivative yields
\begin{align*}
    \frac{\mathrm{d}B_\alpha}{\mathrm{d}\alpha}
\propto&
    -2\E\LL[\frac{(g_\alpha-h_\alpha)^2}{(1+h_\alpha)^3}\RR]\LL(\E\LL[\frac{h_\alpha}{(1+h_\alpha)^2}\RR]\RR)^2
    -2\LL(\E\LL[\frac{g_\alpha-h_\alpha}{(1+h_\alpha)^2}\RR]\RR)^2\E\LL[\frac{h_\alpha^2}{(1+h_\alpha)^3}\RR] \\
&+
    4\E\LL[\frac{h_\alpha(g_\alpha-h_\alpha)}{(1+h_\alpha)^3}\RR]\E\LL[\frac{g_\alpha-h_\alpha}{(1+h_\alpha)^2}\RR]\E\LL[\frac{h_\alpha}{(1+h_\alpha)^2}\RR] \\
\overset{(i)}{\le}& 
    -4\sqrt{\E\LL[\frac{(g_\alpha-h_\alpha)^2}{(1+h_\alpha)^3}\RR]\E\LL[\frac{h_\alpha^2}{(1+h_\alpha)^3}\RR]\LL(\E\LL[\frac{g_\alpha-h_\alpha}{(1+h_\alpha)^2}\RR]\RR)^2\LL(\E\LL[\frac{h_\alpha}{(1+h_\alpha)^2}\RR]\RR)^2} \\
&+
    4\E\LL[\frac{h_\alpha(g_\alpha-h_\alpha)}{(1+h_\alpha)^3}\RR]\E\LL[\frac{g_\alpha-h_\alpha}{(1+h_\alpha)^2}\RR]\E\LL[\frac{h_\alpha}{(1+h_\alpha)^2}\RR] 
\overset{(ii)}{\le} 0, 
\end{align*}
where (i) is due to AM-GM and (ii) due to Cauchy-Schwarz on the first term. Note that the two equalities hold when $g_\alpha=h_\alpha$, from which one can easily deduce that the optimum is achieved when $\upsilon_{xp}\overset{a.s.}{=}\upsilon_{x}\upsilon_{\theta}$, and thus we know that the proposed $\bP$ is the optimal preconditioner that is codiagonazable with $\bSigma_\bX$.
\end{proof}

\subsection{Proof of Proposition~\ref{prop:source_condition}}
\label{subsec:proof_source_condition}
\begin{proof}
Since $\bSigma_{\bt} = \bSigma_\bX^{r}$, we can simplify the expressions by defining $\upsilon_x\triangleq h$ and thus $\upsilon_\theta = h^{r}$. From Theorem~\ref{theo:bias} we have the following derivation of the GD bias under (A1)(A3),
\begin{align*}
    B(\hbt_\bI) \to \frac{m_1'}{m_1^2}\E\frac{h\cdot h^{r}}{(1 + h\cdot m_1)^2}
=
    \frac{\E\frac{h^{1+r}}{(1+h\cdot m_1)^2}}{1 - \gamma\E\frac{(h\cdot m_1)^2}{(1 + h\cdot m_1)^2}},
    \numberthis
    \label{eq:bias_GD}
\end{align*}
where $m_1 = \lim_{\lambda\to 0_+}m(-\lambda)$, and $m$ satisfies 
\[
    \frac{1}{m(-\lambda)} = \lambda + \gamma\E\left[\frac{h}{1 + h\cdot m(-\lambda)}\right].
\]
Similarly, for NGD ($\bP = \bSigma_\bX^{-1}$) we have
\begin{align*}
    B(\hbt_{\boldF^{-1}}) \to \frac{m_2'}{m_2^2}\E\frac{h\cdot h^{r}}{(1 + m_2)^2}
=
    \frac{\E\frac{h^{1+r}}{(1 + m_2)^2}}{1 - \gamma\E\frac{m_2^2}{(1 + m_2)^2}} 
=
    \frac{\E h^{1+r}}{(1+m_2)^2 - \gamma m_2^2},
    \numberthis
    \label{eq:bias_NGD}
\end{align*}
where standard calculation yields $m_2 = (\gamma - 1)^{-1}$, and thus $B(\hbt_{\boldF^{-1}}) \to (1 - \gamma^{-1})\E h^{1+r}$.

To compare the magnitude of \eqref{eq:bias_GD} and \eqref{eq:bias_NGD}, observe the following equivalence.
\begin{align*}
    B(\hbt_\bI) &\lessgtr B(\hbt_{\boldF^{-1}}) \\
\Leftrightarrow\,\,
    \E\frac{h^{1+r}}{(1+h\cdot m_1)^2}\cdot\frac{\gamma}{\gamma - 1} &\lessgtr
    \left(1 - \gamma\E\frac{(h\cdot m_1)^2}{(1 + h\cdot m_1)^2}\right)\E h^{1+r}. \\
\overset{(i)}{\Leftrightarrow}\,\,
    \E\frac{\zeta^{1+r}}{(1+\zeta)^2}\E\frac{\zeta}{1 + \zeta} &\lessgtr
    \E\frac{\zeta}{(1 + \zeta)^2}\E \zeta^{1+r}\E\frac{1}{1 + \zeta}.
    \numberthis
    \label{eq:source_condition_comparison}
\end{align*}
where (i) follows from the definition of $m_1$ and we defined $\zeta \triangleq h\cdot m_1$. Note that when $r\le-1$ and $h$ is not a point mass, we have
\begin{align*}
    \E\frac{\zeta^{1+r}}{(1+\zeta)^2}\E\frac{\zeta}{1 + \zeta} 
> 
    \E\frac{\zeta}{(1+\zeta)^2}\E\frac{\zeta^{1+r}}{1 + \zeta}
>
    \E\frac{\zeta}{(1 + \zeta)^2}\E \zeta^{1+r}\E\frac{1}{1 + \zeta}.
\end{align*} 
On the other hand, when $r\ge 0$, following the exact same procedure we get
\begin{align*}
    \E\frac{\zeta^{1+r}}{(1+\zeta)^2}\E\frac{\zeta}{1 + \zeta} 
< 
    \E\frac{\zeta}{(1 + \zeta)^2}\E \zeta^{1+r}\E\frac{1}{1 + \zeta}.
\end{align*} 
Combining the two cases completes the proof. 

\end{proof}

\subsection{Proof of Proposition~\ref{prop:interpolate}}
{\label{subsec:interpolate-proof}
\begin{proof}  
We first outline a more general setup where $\bP_{\alpha}=f(\bSigma_{\bX};\alpha)$ for continuous and differentiable function of $\alpha$ and $f$ applied to eigenvalues of $\bSigma_{\bx}$. For any interval $\mathcal{I}\subseteq [0,1]$, we claim 
\begin{itemize}[topsep=0.5mm,itemsep=0.5mm]
    \item[(a)] Suppose all four functions $\frac{1}{xf(x;\alpha)}$, $f(x;\alpha)$, $\frac{\partial f(x;\alpha)}{\partial \alpha} / f(x;\alpha)$ and $x\frac{\partial f(x;\alpha)}{\partial \alpha}$ are decreasing functions of $x$ on the support of $v_x$ for all $\alpha \in \mathcal{I}$. In addition, $\frac{\partial f(x;\alpha)}{\partial \alpha}\geq 0$ on the support of $v_x$ for all $\alpha\in\mathcal{I}$. Then the stationary bias is an increasing function of $\alpha$ on $\mathcal{I}$.
    \item[(b)] For all $\alpha\in\mathcal{I}$, suppose $xf(x;\alpha)$ is a monotonic function of $x$ on the support of $v_x$ and $\frac{\partial f(x;\alpha)}{\partial \alpha} / f(x;\alpha)$ is a decreasing function of $x$ on the support of $v_x$. Then the stationary variance is a decreasing function of $\alpha$ on $\mathcal{I}$.
\end{itemize}
Let us verify the three choices of $\bP_{\alpha}$ in Proposition \ref{prop:interpolate} one by one.
\begin{itemize}[leftmargin=*,topsep=0.5mm,itemsep=0.5mm]
\item When $\bP_{\alpha}=(1-\alpha)\bI_d+\alpha \left(\bSigma_{\bX}\right)^{-1}$, the corresponding $f(x;\alpha)$ is $(1-\alpha)+\alpha x$. It is clear that it satisfies all conditions in (a) and (b) for all $\alpha \in [0,1]$. Hence, the stationary variance is a decreasing function and the stationary bias is an increasing function of $\alpha \in [0,1]$.

\item When $\bP_{\alpha}=\left(\bSigma_{\bX}\right)^{-\alpha}$, the corresponding $f(x;\alpha)$ is $x^{-\alpha}$. It is clear that it satisfies all conditions in (a) and (b) for all $\alpha \in [0,1]$ except for the condition that $x\frac{\partial f(x;\alpha)}{\partial \alpha}=-x^{1-\alpha}\ln x$ is a decreasing function of $x$. Note that $x\frac{\partial f(x;\alpha)}{\partial \alpha}=-x^{1-\alpha}\ln x$ is a decreasing function of $x$ on the support of $v_x$ only for $\alpha\geq \frac{\ln(\kappa)-1}{\ln(\kappa)}$ where $\kappa=\sup v_x/\inf v_x$. Hence, the stationary variance is a decreasing function of $\alpha\in [0,1]$ and the stationary bias is an increasing function of $\alpha \in [\max(0,\frac{\ln(\kappa)-1}{\ln(\kappa)}),1]$.

\item When $\bP_{\alpha}=\LL(\alpha\bSigma_\bX \!+\! (1\!-\!\alpha)\bI_d\RR)^{-1}$, the corresponding $f(x;\alpha)$ is $1/(\alpha x+(1-\alpha))$. It is clear that it satisfies all conditions in (a) and (b) for all $\alpha \in [0,1]$ except for the condition that $x\frac{\partial f(x;\alpha)}{\partial \alpha}=\frac{x(1-x)}{(\alpha x+(1-\alpha))^2}$ is a decreasing function of $x$. Note that $x\frac{\partial f(x;\alpha)}{\partial \alpha}=\frac{x(1-x)}{(\alpha x+(1-\alpha))^2}$ is a decreasing function of $x$ on the support of $v_x$ only for $\alpha\geq \frac{\kappa-2}{\kappa-1}$. Hence, the stationary variance is a decreasing function of $\alpha\in [0,1]$ and the stationary bias is an increasing function of $\alpha \in [\max(0,\frac{\kappa-2}{\kappa-1}),1]$. 
\end{itemize}

To show (a) and (b), note that under the conditions on $\bSigma_{\bx}$ and $\bSigma_{\btheta}$ assumed in Proposition~\ref{prop:interpolate}, the stationary bias $B(\hat{\btheta}_{\bP_{\alpha}})$ and the stationary variance $V(\hat{\btheta}_{\bP_{\alpha}})$ can be simplified to
\[
  B(\hat{\btheta}_{\bP_{\alpha}}) \ = \ \frac{m_{\alpha}'(0)}{m_{\alpha}^2(0)} \E \frac{v_x}{(1+v_xf(v_x;\alpha)m_{\alpha}(0))^2} \quad \text{and} \quad V(\hat{\btheta}_{\bP_{\alpha}}) \ = \ 
\sigma^2 \cdot\left(\frac{m_{\alpha}'(0)}{m_{\alpha}^2(0)}-1\right),
\]
where $m_{\alpha}(z)$ and $m_{\alpha}'(z)$ satisfy
\begin{align}
1&=-zm_{\alpha}(z)+\gamma \E \frac{v_xf(v_x;\alpha)m_{\alpha}(z)}{1+v_xf(v_x;\alpha)m_{\alpha}(z)}\label{eq:prop4_eq1}\\
\frac{m_{\alpha}'(z)}{m_{\alpha}^2(z)}&=\frac{1}{1-\gamma \E\left(\frac{f(v_x;\alpha)m_{\alpha}(z)}{1+f(v_x;\alpha)m_{\alpha}(z)}\right)^2}. \label{eq:prop4_eq2}
\end{align}
For notation convenience, let $f_{\alpha} :=  v_xf(v_x;\alpha)$.
From \eqref{eq:prop4_eq2}, we have the following equivalent expressions.
\begin{align}
  B(\hat{\btheta}_{\bP_{\alpha}}) &= \frac{\E \frac{v_x}{(1+f_{\alpha}m_{\alpha}(0))^2}}{1-\gamma\E \left(\frac{f_{\alpha}m_{\alpha}(0)}{1+f_{\alpha}m_{\alpha}(0)}\right)^2}, \label{eq:prop4_B}\\
V(\hat{\btheta}_{\bP_{\alpha}})&=\sigma^2\left(\frac{1}{1-\gamma\E\left( \frac{f_{\alpha}m_{\alpha}(0)}{1+f_{\alpha}m_{\alpha}(0)}\right)^2}-1\right).\label{eq:prop4_V}
\end{align}
We first show that (b) holds. Note that from \eqref{eq:prop4_V}, we have 
\begin{align}
\frac{\partial V(\hat{\btheta}_{\bP_{\alpha}})}{\partial \alpha} \!=\! \gamma\sigma^2\left(\frac{1}{1\!-\!\gamma\E\left( \frac{f_{\alpha}m_{\alpha}(0)}{1\!+\!f_{\alpha}m_{\alpha}(0)}\right)^2}\right)^2
\!\E\left[\frac{2f_{\alpha}m_{\alpha}(0)}{\left(1\!+\!f_{\alpha}m_{\alpha}(0)\right)^3} \left(f_{\alpha}\!\frac{\partial m_{\alpha}(z)}{\partial \alpha}\Big|_{z=0} \!+\! \frac{\partial f_{\alpha}}{\partial \alpha} \!m_{\alpha}(0)\right) \right]. 
\label{eq:prop4_Vdif}
\end{align}
To calculate $\frac{\partial m_{\alpha}(z)}{\partial \alpha}\Big|_{z=0}$, we take derivatives with respect to $\alpha$ on both sides of \eqref{eq:prop4_eq1},
\begin{align}
0&=\gamma\E\left[\frac{1}{(1+f_{\alpha}m_{\alpha}(0))^2}\cdot \left(f_{\alpha}\frac{\partial m_{\alpha}(z)}{\partial \alpha}\Big|_{z=0}+\frac{\partial f_{\alpha}}{\partial \alpha} m_{\alpha}(0)\right)\right].\label{eq:prop4_eq3}
\end{align}
Therefore, plugging \eqref{eq:prop4_eq3} into \eqref{eq:prop4_Vdif} yields
\begin{align}
\frac{\partial V(\hat{\btheta}_{\bP_{\alpha}})}{\partial \alpha}=&2\gamma\sigma^2\left(\frac{m_{\alpha}(0)}{1-\gamma\E\left(  \frac{f_{\alpha}m_{\alpha}(0)}{1+f_{\alpha}m_{\alpha}(0)}\right)^2}\right)^2\left(\E \frac{f_{\alpha}}{(1+f_{\alpha}m_{\alpha}(0))^2}\right)^{-1}\nonumber\\
&\times\left(\E \frac{f_{\alpha}\frac{\partial f_{\alpha}}{\partial \alpha}}{\left(1+f_{\alpha}m_{\alpha}(0)\right)^3}\E\frac{f_{\alpha}}{(1+f_{\alpha}m_{\alpha}(0))^2}-\E\frac{f_{\alpha}^2}{\left(1+f_{\alpha}m_{\alpha}(0)\right)^3}\E\frac{\frac{\partial f_{\alpha}}{\partial \alpha}}{(1+f_{\alpha}m_{\alpha}(0))^2}\right)\nonumber
\end{align}
Thus showing $V(\hat{\btheta}_{\bP_{\alpha}})$ is a decreasing function of $\alpha$ is equivalent to showing that 
\begin{align}
\E\frac{f_{\alpha}^2}{\left(1+f_{\alpha}m_{\alpha}(0)\right)^3}\E\frac{\frac{\partial f_{\alpha}}{\partial \alpha}}{(1+f_{\alpha}m_{\alpha}(0))^2}\geq \E \frac{f_{\alpha}\frac{\partial f_{\alpha}}{\partial \alpha}}{\left(1+f_{\alpha}m_{\alpha}(0)\right)^3}\E\frac{f_{\alpha}}{(1+f_{\alpha}m_{\alpha}(0))^2}.\label{eq:prop4_Vcondition}
\end{align}
Let $\mu_x$ be the probability measure of $v_x$. We define a new measure $\tilde{\mu}_x=\frac{f_{\alpha}\mu_x}{(1+f_{\alpha}m_{\alpha}(0))^2}$, and let $\tilde{v}_x$ follow the new measure. 
Since $\frac{\partial f(x;\alpha)}{\partial \alpha} / f(x;\alpha)$ is a decreasing function of $x$ and $xf(x;\alpha)$ is a monotonic function of $x$,
\[
\E \frac{\tilde{v}_xf(\tilde{v}_x;\alpha)}{1+\tilde{v}_xf(\tilde{v}_x;\alpha)m_{\alpha}(0)} \E \frac{\frac{\partial \tilde{v}_xf(\tilde{v}_x;\alpha)}{\partial \alpha}}{\tilde{v}_xf(\tilde{v}_x;\alpha)} \geq \E \frac{\frac{\partial \tilde{v}_xf(\tilde{v}_x;\alpha)}{\partial \alpha}}{1+\tilde{v}_xf(\tilde{v}_x;\alpha)m_{\alpha}(0)}.
\]
Changing $\tilde{v}_x$ back to $v_x$, we arrive at \eqref{eq:prop4_Vcondition} and thus (b).

For the bias term $  B(\hat{\btheta}_{\bP_{\alpha}})$, note that from \eqref{eq:prop4_eq1} and \eqref{eq:prop4_B}, we have
\begin{align}
\frac{\partial   B(\hat{\btheta}_{\bP_{\alpha}})}{\partial \alpha}=&\frac{1}{\gamma}\left(\frac{1}{\gamma}-\E \left(\frac{f_{\alpha}m_{\alpha}(0)}{1+f_{\alpha}m_{\alpha}(0)}\right)^2\right)^{-2}\nonumber\\
&\times \left(-\E\left[2\frac{v_x}{\left(1+f_{\alpha}m_{\alpha}(0)\right)^3} \cdot\left(f_{\alpha}\frac{\partial m_{\alpha}(z)}{\partial \alpha}\Big|_{z=0}+\frac{\partial f_{\alpha}}{\partial \alpha} m_{\alpha}(0)\right)\right]\E \frac{f_{\alpha}m_{\alpha}(0)}{(1+f_{\alpha}m_{\alpha}(0))^2}\right.\nonumber\\
&+\left.\E\frac{v_x}{\left(1+f_{\alpha}m_{\alpha}(0)\right)^2}\E\left[2\frac{f_{\alpha}m_{\alpha}(0)}{\left(1+f_{\alpha}m_{\alpha}(0)\right)^3} \cdot\left(f_{\alpha}\frac{\partial m_{\alpha}(z)}{\partial \alpha}\Big|_{z=0}+\frac{\partial f_{\alpha}}{\partial \alpha} m_{\alpha}(0)\right) \right]\right).
\label{eq:prop4_Bdif}
\end{align}
Similarly, we combine \eqref{eq:prop4_eq3} and \eqref{eq:prop4_Bdif} and simplify the expression. To verify $B(\hat{\btheta}_{\bP_{\alpha}})$ is an increasing function of $\alpha$, we need to show that
\begin{align}
0\leq&\left(\E \frac{v_x f_{\alpha}m_{\alpha}(0)}{(1+f_{\alpha}m_{\alpha}(0))^3}\E\frac{\frac{\partial f_{\alpha}}{\partial \alpha}}{(1+f_{\alpha}m_{\alpha}(0))^2} - \E\frac{v_x\frac{\partial f_{\alpha}}{\partial \alpha}}{(1+f_{\alpha}m_{\alpha}(0))^3}\E\frac{f_{\alpha}m_{\alpha}(0)}{(1+f_{\alpha}m_{\alpha}(0))^2}\right)\E\frac{f_{\alpha}m_{\alpha}(0)}{(1+f_{\alpha}m_{\alpha}(0))^2}\nonumber\\
&-\E\frac{v_x}{(1+f_{\alpha}m_{\alpha}(0))^2}\left(\E\frac{(f_{\alpha}m_{\alpha}(0))^2}{(1+f_{\alpha}m_{\alpha}(0))^3}\E \frac{\frac{\partial f_{\alpha}}{\partial \alpha}}{(1+f_{\alpha}m_{\alpha}(0))^2}
-\E\frac{f_{\alpha}m_{\alpha}(0)\frac{\partial f_{\alpha}}{\partial \alpha}}{(1+f_{\alpha}m_{\alpha}(0))^3}\E\frac{f_{\alpha}m_{\alpha}(0)}{(1+f_{\alpha}m_{\alpha}(0))^2}\right),\nonumber\\
\label{eq:prop4_eq4}
\end{align}
Let $h_{\alpha}\triangleq f_{\alpha}m_{\alpha}(0)=v_xf(v_x;\alpha)m_{\alpha}(0)$ and $g_{\alpha}\triangleq \frac{\partial f_{\alpha}}{\partial \alpha}=v_x\frac{\partial f(v_x;\alpha)}{\partial \alpha}$. Then \eqref{eq:prop4_eq4} can be further simplified to the following equation
\begin{align}
0\leq&\underbrace{\E \frac{v_xh_{\alpha}}{(1+h_{\alpha})^3}\E \frac{g_{\alpha}}{(1+h_{\alpha})^3}\E \frac{h_{\alpha}}{(1+h_{\alpha})^3} -\E \frac{v_x}{(1+h_{\alpha})^3}\E \frac{g_{\alpha}}{(1+h_{\alpha})^3}\E \frac{h_{\alpha}^2}{(1+h_{\alpha})^3}}_{\text{part 1}} \nonumber\\
&+\underbrace{\E \frac{v_x}{(1+h_{\alpha})^3}\E \frac{g_{\alpha}h_{\alpha}}{(1+h_{\alpha})^3}\E \frac{h_{\alpha}}{(1+h_{\alpha})^3}-\E \frac{v_xg_{\alpha}}{(1+h_{\alpha})^3}\E \frac{h_{\alpha}}{(1+h_{\alpha})^3}\E \frac{h_{\alpha}}{(1+h_{\alpha})^3}}_{\text{part 2}}\nonumber\\
&+\underbrace{2\E \frac{v_xh_{\alpha}}{(1+h_{\alpha})^3}\E \frac{g_{\alpha}h_{\alpha}}{(1+h_{\alpha})^3}\E \frac{h_{\alpha}}{(1+h_{\alpha})^3}-2\E \frac{v_xg_{\alpha}}{(1+h_{\alpha})^3}\E \frac{h_{\alpha}^2}{(1+h_{\alpha})^3}\E \frac{h_{\alpha}}{(1+h_{\alpha})^3}}_{\text{part 3}}\nonumber\\
&+\underbrace{\E \frac{v_xh_{\alpha}}{(1+h_{\alpha})^3}\E \frac{g_{\alpha}h_{\alpha}}{(1+h_{\alpha})^3}\E \frac{h_{\alpha}^2}{(1+h_{\alpha})^3}-\E \frac{v_xg_{\alpha}}{(1+h_{\alpha})^3}\E \frac{h_{\alpha}^2}{(1+h_{\alpha})^3}\E \frac{h_{\alpha}^2}{(1+h_{\alpha})^3}}_{\text{part 4}}.
\label{eq:prop4_eq5}
\end{align}
Note that under condition of (a), we know that both $h_{\alpha}$ and $v_x/h_{\alpha}$ are increasing functions of $v_x$; and both  $g_{\alpha}/h_{\alpha}$ and $g_{\alpha}$ are decreasing functions of $v_x$. 
Hence, with calculation similar to \eqref{eq:prop4_Vcondition}, we know part 1,2,3,4 in \eqref{eq:prop4_eq5} are all non-negative, and therefore \eqref{eq:prop4_eq5} holds.
\end{proof}

\vspace{-2.5mm} 
\begin{remark}
The above characterization provides sufficient but not necessary conditions for the monotonicity of the bias term. In general, the expression of the bias is rather opaque, and determining the sign of its derivative can be tedious, except for certain special cases (e.g.~$\gamma=2$ and the eigenvalues of $\bSigma_\bX$ are two equally weighted point masses, for which $m_\alpha$ has a simple form and one may analytically check the monotonicity). 
We conjecture that the bias is monotone for $\alpha\in[0,1]$ for a much wider class of $\bSigma_\bX$, as shown in Figure~\ref{fig:interpolate_bias}. 
\end{remark} 

\begin{figure}[!htb]
% \vspace{-2mm} 
\centering
\begin{minipage}[t]{0.4\linewidth}
\centering
{\includegraphics[width=0.88\textwidth]{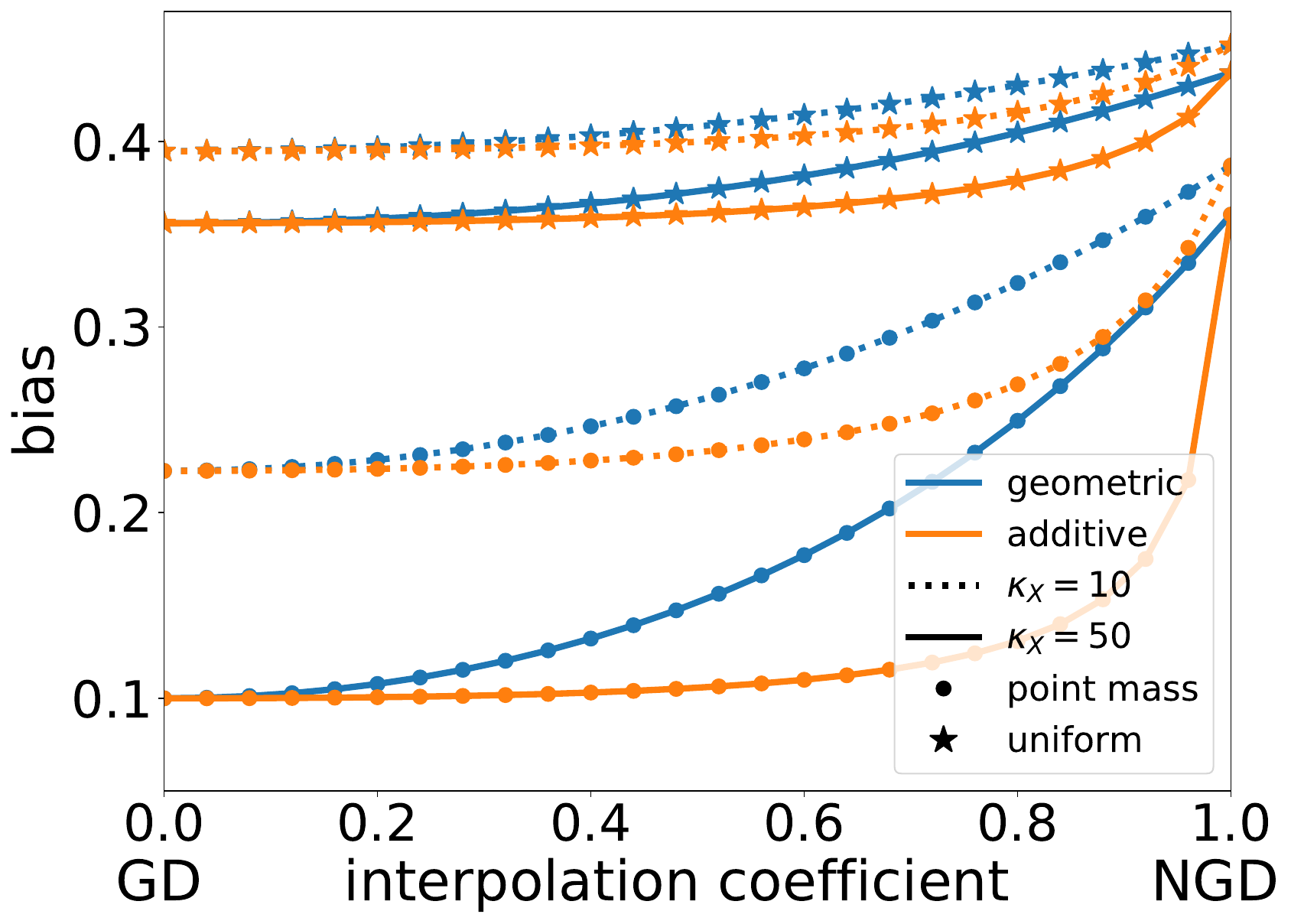}}  

\small (a) $\gamma=2$.
\end{minipage}
\begin{minipage}[t]{0.4\linewidth}
\centering
{\includegraphics[width=0.9\textwidth]{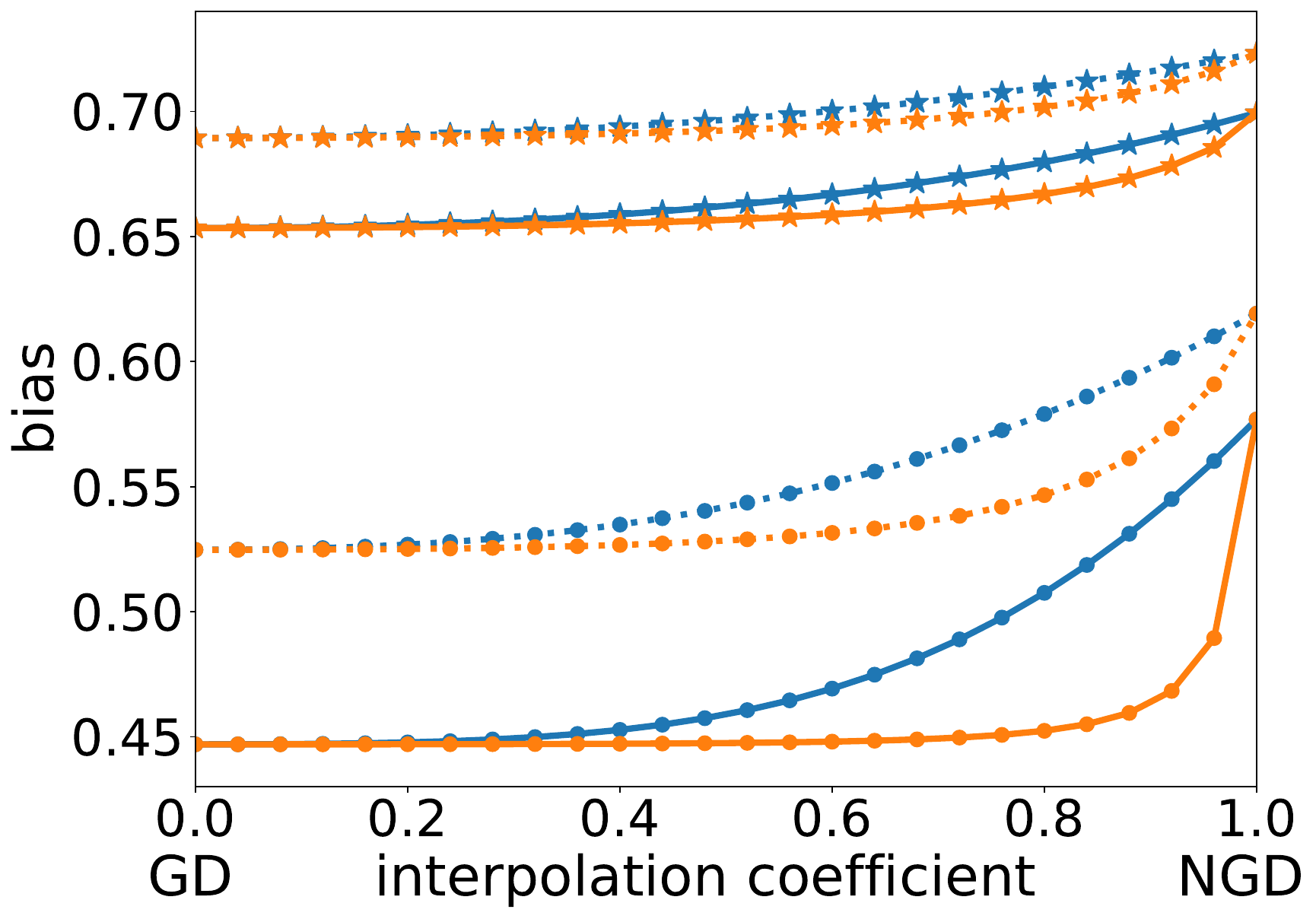}} 

\small (b) $\gamma=5$.
\end{minipage}
\vspace{-0.15cm}
\caption{\small Illustration of the monotonicity of the bias term under $\bSigma_{\btheta}=\bI_d$. We consider two distributions of eigenvalues for $\bSigma_\bX$: two equally weighted point masses (circle) and a uniform distribution (star), and vary the condition number $\kappa_X$ and overparameterization level $\gamma$. In all cases the bias in monotone in $\alpha\in[0,1]$.}
\label{fig:interpolate_bias}
% \vspace{-0.15cm}
\end{figure}
}
 
\subsection{Proof of Proposition~\ref{prop:monotone}}
\label{subsec:proof_variance_ES}
\begin{proof}
Taking the derivative of $V(\bt_\boldP(t))$ w.r.t.~time yields (omitting the scalar $\sigma^2$),
\begin{align*}
    \frac{\mathrm{d}V(\bt_\bP(t))}{\dt}
=&
    \frac{\mathrm{d}}{\dt} \norm{\bSigma_\bX^{1/2}\bP\bX^\top\LL(\bI_n - \exp\LL(-\frac{t}{n}\bX\bP\bX^\top\RR)\RR)\LL(\bX\bP\bX^\top\RR)^{-1}}_F^2 \\
\overset{(i)}{=}& 
    \frac{1}{n}\Tr{\bSigma_{\bX\bP}\underbrace{\bar{\bX}^\top\boldS_\bP\exp\LL(-\frac{t}{n}\boldS_\bP\RR)\boldS_\bP^{-2}\LL(\bI_n - \exp\LL(-\frac{t}{n}\boldS_\bP\RR)\RR)\bar{\bX}}_{p.s.d.}} 
\overset{(ii)}{>} 0,
\end{align*}
where we defined $\bar{\bX} = \bX\bP^{1/2}$ and $\boldS_\bP = \bX\bP\bX^\top$ in (i), and (ii) is due to (A2-3) the inequality $\Tr{\boldA\boldB}\ge\lambda_{\min}(\boldA)\Tr{\boldB}$ for positive semi-definite $\boldA$ and $\boldB$.
\end{proof}

\subsection{Proof of Proposition~\ref{prop:early-stop}}
\label{subsec:proof_bias_ES}
\begin{proof}
Recall the definition of the bias (well-specified) of $\hbt_\bP(t)$,
\begin{align*}
    B(\bt_\bP(t))
\overset{(i)}{=} & 
    \frac{1}{d}\Tr{\bSigma_{\btheta}\LL(\bI_d - \bP\bX^\top\bW_\bP(t)\boldS_\bP^{-1}\bX\RR)^\top\bSigma_\bX\LL(\bI_d - \bP\bX^\top\bW_\bP(t)\boldS_\bP^{-1}\bX\RR)} \\
\overset{(ii)}{=} &
    \frac{1}{d}\Tr{\bSigma_{\btheta/\bP}\LL(\bI_d - \bar{\bX}^\top\bW_\bP(t)\boldS_\bP^{-1}\bar{\bX}\RR)^\top\bSigma_{\bX\bP}\LL(\bI_d - \bar{\bX}^\top\bW_\bP(t)\boldS_\bP^{-1}\bar{\bX}\RR)} \\
\overset{(iii)}{\ge} & 
    \frac{1}{d}\Tr{\LL(\bSigma_{\bX\bP}^{1/2}\LL(\bI_d - \bar{\bX}^\top\bW_\bP(t)\boldS_\bP^{-1}\bar{\bX}\RR)\bSigma_{\btheta/\bP}^{1/2}\RR)^2},
    \numberthis
    \label{eq:early-stop-bias}
\end{align*}
where we defined $\boldS_\bP = \bX\bP\bX^\top$, $\bW_\bP(t) = \bI_n - \exp\LL(-\frac{t}{n}\boldS_\bP\RR)$ in (i), $\bar{\bX} = \bX\bP^{1/2}$ in (ii), and (iii) is due to the inequality $\Tr{\boldA^\top\boldA}\ge\Tr{\boldA^2}$.

When $\bSigma_\bX = \bSigma_{\btheta}^{-1}$, i.e.~NGD achieves lowest stationary bias, \eqref{eq:early-stop-bias} simplifies to
\begin{align*}
    \!\!\!\!\!\!\!\!\!\!B(\bt_\bP(t)) \ge \frac{1}{d}\Tr{\LL(\bI_d - \bar{\bX}^\top\bW_\bP(t)\boldS_\bP^{-1}\bar{\bX}\RR)^2}
=
    \LL(1 - \frac{1}{\gamma}\RR) + \frac{1}{d}\sum_{i=1}^n  \exp\LL(-\frac{t}{n}\bar{\lambda}_i\RR)^2,
    \numberthis
    \label{eq:early-stop-PGD-bias}
\end{align*}
where $\bar{\lambda}$ is the eigenvalue of $\boldS_\bP$. On the other hand, since $\boldF=\bSigma_\bX$, for the NGD iterate $\hbt_{\boldF^{-1}}(t)$ we have
\begin{align*}
    \!\!\!\!\!\!\!\!\!\!\!\!B\LL(\bt_{\boldF^{-1}}(t)\RR) = \frac{1}{d}\Tr{\LL(\bI_d - \hat{\bX}^\top \bW_{\boldF^{-1}}(t) \boldS_{\boldF^{-1}}^{-1} \hat{\bX}\RR)^2} =  \LL(1 - \frac{1}{\gamma}\RR) + \frac{1}{d}\sum_{i=1}^n  \exp\LL(-\frac{t}{n}\hat{\lambda}_i\RR)^2,
    \numberthis
    \label{eq:early-stop-NGD-bias}
\end{align*}
where $\hat{\bX} = \bX\bSigma_\bX^{-1/2}$ and $\bar{\lambda}$ is the eigenvalue of $\boldS_{\boldF^{-1}} = \hat{\bX}\hat{\bX}^\top$. Comparing \eqref{eq:early-stop-PGD-bias}\eqref{eq:early-stop-NGD-bias}, we see that given $\hbt_\bP(t)$ at a fixed t, if we run NGD for time $T > \frac{\bar{\lambda}_{\max}}{\hat{\lambda}_{\min}}t$ (note that $T/t = O(1)$ by (A2-3)), then we have $B(\bt_\bP(t)) \ge B\LL(\bt_{\boldF^{-1}}(T)\RR)$ for any $\bP$ satisfying (A3). This thus implies that $B^{\mathrm{opt}}(\bt_\bP)\ge B^{\mathrm{opt}}\LL(\bt_{\boldF^{-1}}\RR)$.
 
On the other hand, when $\bSigma_{\bt}=\bI_d$, we can show that the bias term of GD is monotonically decreasing through time by taking its derivative,
\begin{align*}
    &\frac{\mathrm{d}}{\dt}B(\bt_\bI(t))
=
    \frac{1}{d}\frac{\mathrm{d}}{\dt}\Tr{\LL(\bI_d - \bX^\top\bW_\bI(t)\boldS_\bI^{-1}\bX\RR)^\top\bSigma_\bX\LL(\bI_d - \bX^\top\bW_\bI(t)\boldS_\bI^{-1}\bX\RR)} \\
&=
    -\frac{1}{nd}\Tr{\bSigma_\bX\underbrace{\bX^\top\boldS\exp\LL(-\frac{t}{n}\boldS\RR)\boldS^{-1}\bX\LL(\bI_d - \bX^\top\bW_\bI(t)\boldS_\bI^{-1}\bX\RR)}_{p.s.d.}} < 0.
    \numberthis
    \label{eq:GD-bias-derivative}
\end{align*} 

Similarly, one can verify that the expected bias of NGD is monotonically decreasing for all choices of $\bSigma_\bX$ and $\bSigma_{\bt}$ satisfying (A2-4), 
\begin{align*}
    &\frac{\mathrm{d}}{\dt}\Tr{\bSigma_{\btheta}\LL(\bI_d - \boldF^{-1}\bX^\top\bW_{\boldF^{-1}}(t)\boldS_{\boldF^{-1}}^{-1}\bX\RR)^\top\bSigma_\bX\LL(\bI_d - \boldF^{-1}\bX^\top\bW_{\boldF^{-1}}(t)\boldS_{\boldF^{-1}}^{-1}\bX\RR)} \\
=&
    \frac{\mathrm{d}}{\dt}\Tr{\bSigma_{\bX\btheta}\LL(\bI_d - \hat{\bX}^\top\bW_{\boldF^{-1}}(t)\boldS_{\boldF^{-1}}^{-1}\hat{\bX}\RR)^\top\LL(\bI_d - \hat{\bX}^\top\bW_{\boldF^{-1}}(t)\boldS_{\boldF^{-1}}^{-1}\hat{\bX}\RR)} 
\overset{(i)}{<} 0, 
\end{align*} 
where (i) follows from calculation similar to \eqref{eq:GD-bias-derivative}.
Since the expected bias is decreasing through time for both GD and NGD when $\bSigma_{\bt}=\bI_d$, and from Theorem~\ref{theo:bias} we know that $B(\hbt_\bI)\le B(\hbt_{\boldF^{-1}})$, we conclude that $B^{\mathrm{opt}}(\bt_\bI)\le B^{\mathrm{opt}}\LL(\bt_{\boldF^{-1}}\RR)$.
\end{proof}

\allowdisplaybreaks

\subsection{Proof of Theorem~\ref{theo:RKHS}}
\label{subsec:proof_RKHS}
\subsubsection{Setup and Main Result}
\label{subsec:RKHS_setup}
We restate the setting and assumptions.
$\calH$ is an RKHS included in $\LPiPx$ equipped with a bounded kernel function $k$ satisfying $\sup_{\supp{P_X}}k(\bx,\bx) \leq 1$.
$K_\bx \in \calH$ is the Riesz representation of the kernel function $k(x,\cdot)$, that is, $k(\bx,\by) = \langle K_\bx, K_\by\rangle_{\calH}$.
$S$ is the canonical embedding operator from $\calH$ to $\LPiPx$.
We write $\Sigma = S^*S: \calH \to \calH$ and $L = SS^*$.
Note that the boundedness of the kernel gives $\|S f\|_{\LPiPx} \leq \sup_{\bx} |f(\bx)| = \sup_{\bx} |\langle K_\bx, f\rangle|\leq \|K_\bx\|_{\calH}\|f\|_{\calH}\leq \|f\|_{\calH}$. Hence we know $\|\Sigma\| \leq 1$ and $\|L\| \leq 1$. Our analysis will be made under the following standard regularity assumptions.
\begin{itemize}[leftmargin=*,itemsep=0.5mm,topsep=0.5mm] 
    \item There exist $r \!\in \!(0,\infty)$ and $M \!>\! 0$ such that $f^* \!=\! L^r h^*$ for some $h^* \!\in\! L_2(P_X)$ and $\norm{f^*}_{\infty} \!\leq\! M$.
    \item There exists $s > 1$ s.t. $\Tr{\Sigma^{1/s}} < \infty$ and $2r + s^{-1} > 1$.
    \item There exist $\mu\in[s^{-1},1]$ and $C_\mu>0$ such that $\sup_{\bx\in\supp{P_X}}\norm{\Sigma^{1/2-1/\mu}K_\bx}_\calH \le C_\mu$. 
\end{itemize} 
% (A5)(A6) are standard regularity assumptions in the literature that provide capacity control of the RKHS (e.g., see \cite{caponnetto2007optimal,pillaud2018statistical}). 
% It is worth noting that previous works mostly consider $r\ge 1/2$ which implies that $f^*\!\in\!\calH$.

The training data is generated as $y_i = f^*(\bx_i) + \varepsilon_i$, where $\varepsilon_i$ is an i.i.d. noise satisfying $|\varepsilon_i| \leq \sigma$ almost surely. Let $\by\in\R^n$ be the label vector. We identify $\R^n$ with $L_2(P_n)$ and define
\begin{align*}
\hat{\Sigma} = \frac{1}{n} \sum_{i=1}^n K_{\bx_i} \otimes K_{\bx_i}: \calH \to \calH, \quad \hat{S}^*Y = \frac{1}{n}\sum_{i=1}^n Y_i K_{\bx_i}, ~
(Y \in L_2(P_n)).
\end{align*}

We consider the following preconditioned update on $f_t \in \calH$:
\begin{align*}
f_t = f_{t-1}  - \eta (\Sigma + \lambda I)^{-1}(\hat{\Sigma} f_{t-1} - \hat{S}^* Y), \quad f_0 = 0.
\end{align*}

We briefly comment on how our analysis differs from \cite{rudi2017falkon}, which showed that a preconditioned update (the FALKON algorithm) for kernel ridge regression can also achieve accelerated convergence in the population risk. We emphasize the following differences.
\begin{itemize}[leftmargin=*,itemsep=1pt,topsep=0pt]
    \item The two algorithms optimize different objectives, as highlighted by the different role of the “ridge” coefficient $\lambda$. In FALKON, $\lambda$ turns the objective into kernel ridge regression; whereas in our \eqref{eq:RKHS-update}, $\lambda$ controls the interpolation between GD and NGD. As we aim to study how the preconditioner affects generalization, it is important that we look at the objective in its original (instead of regularized) form. 
    \item To elaborate on the first point, since FALKON minimizes a regularized objective, it would not overfit even after large number of gradient steps, but it is unclear how preconditioning impacts generalization (i.e., any preconditioner may generalize well with proper regularization).  
    In contrast, we consider the unregularized objective, and thus early stopping plays a crucial role -- this differs from most standard analysis of GD.   
    \item Algorithm-wise, the two updates employ different preconditioners. FALKON involves inverting the kernel matrix $K$ defined on the training points, whereas we consider the population covariance operator $\Sigma$, which is consistent with our earlier discussion on the population Fisher in Section~\ref{sec:risk}.
    \item In terms of the theoretical setup, our analysis allows for $r<1/2$, whereas \cite{rudi2017falkon} and many other previous works assumed $r\in[1/2,1]$, as commented in Section~\ref{subsec:RKHS}.
\end{itemize}

We aim to show the following theorem:
\begin{theo*}[Formal Statement of Theorem~\ref{theo:RKHS}]
Given (A4-6), if the sample size $n$ is sufficiently large so that $1/(n\lambda) \ll 1$, then for $\eta < \|\Sigma\|$ with $\eta t \geq 1$ and $0 < \delta < 1$ and $0 < \lambda < 1$, it holds that
\begin{align*}
& \|S f_t - f^* \|_{\LPiPx}^2 \leq C (B(t) + V(t)),
\end{align*}
with probability $1-3 \delta$, where $C$ is a constant and 
\begin{align*}
B(t) & := \exp(-\eta t) \vee \left(\frac{\lambda}{\eta t} \right)^{2r}, \\
V(t) & := 
V_1(t)
 + (1+\eta t)\left( \frac{   \lambda^{-1} B(t)+ 
\sigma^2 \Tr{\Sigma^{\frac{1}{s}}}\lambda^{-\frac{1}{s}} }{n}  + \frac{\lambda^{-1}(\sigma +M + (1+t\eta) \lambda^{-(\frac{1}{2}-r)_+})^2}{n^2}\right) \log(1/\delta)^2,
\end{align*}
in which 
\begin{align*}
V_1(t)  := \left[\exp(-\eta t) \vee  \left(\frac{\lambda}{\eta t}\right)^{2r} + (t \eta)^2 
\left(\frac{\beta'(1\vee \lambda^{2r - \mu})\Tr{\Sigma^{\frac{1}{s}}}\lambda^{-\frac{1}{s}}}{n} +
 \frac{\beta'^2( 1 +\lambda^{-\mu}(1\vee \lambda^{2r-\mu})}{n^2}\right) 
\right](1+t\eta)^2,
\end{align*}
for 
$\beta' = \log\LL(\frac{28C_\mu^2 (2^{2r-\mu} \vee \lambda^{-\mu + 2r}) \Tr{\Sigma^{1/s}}\lambda^{-1/s} }{\delta}\RR)$.
When $r \geq 1/2$, if we set $\lambda = n^{-\frac{s}{2rs + 1}} =:\lambda^*$ and $t = \Theta(\log(n))$, then the overall convergence rate becomes
\begin{align*}
\|S g_t - f^* \|_{\LPiPx}^2  = \widetilde{O}_p \left( n^{- \frac{2rs}{2rs + 1}}\right),
\end{align*}
which is the minimax optimal rate ($\tilde{O}_p(\cdot)$ hides a poly-$\log(n)$ factor). 
On the other hand, when $r < 1/2$, the bound is also $\widetilde{O}_p \left( n^{- \frac{2rs}{2rs + 1}}\right)$ except the term $V_1(t)$. 
In this case, if $2r \geq \mu$ holds additionally, we have $V_t(t) = \widetilde{O}_p \left( n^{- \frac{2rs}{2rs + 1}}\right)$, which again recovers the optimal rate.
\end{theo*}

Note that if the GD (with iterates $\tilde{f}_t$) is employed, from previous work \cite{lin2017optimal}
we know that the {\it bias term} $\left(\frac{\lambda}{\eta t} \right)^{2r}$ is replaced by $\left(\frac{1}{\eta t} \right)^{2r}$,
and therefore the upper bound translates to
\begin{align*}
\|S \tilde{f}_t - f^* \|_{\LPiPx}^2 \leq 
C \left\{
(\eta t)^{-2r} + \frac{1 }{n} \left(\Tr{\Sigma^{1/s}}(\eta t)^{1/s} + \frac{\eta t}{n}  \right)
\left( \sigma^2 + \left(\frac{1}{\eta t} \right)^{2r} + \frac{M^2 + (\eta t)^{-(2r-1)}}{n}\right)
\right\},
\end{align*}
with high probability.
In other words, by the condition $\eta = O(1)$, we need $t = \Theta(n^{\frac{2rs}{2rs + 1}})$ steps to sufficiently diminish the bias term. In contrast, the preconditioned update that interpolates between GD and NGD \eqref{eq:RKHS-update} only require $t = O(\log(n))$ steps to make the bias term negligible. This is because the NGD amplifies the high frequency component and rapidly captures the detailed ``shape'' of the target function $f^*$.  

\subsubsection{Proof of Main Result}
\label{subsub:RKHS_proof_main}
\begin{proof} 
We follow the proof strategy of \cite{lin2017optimal}.
First we define a reference optimization problem with iterates $\bar{f}_t$ that directly minimize the population risk:
\begin{align}\label{eq:ftUpdate}
\bar{f}_t = \bar{f}_{t-1}  - \eta (\Sigma + \lambda I)^{-1}(\Sigma \bar{f}_{t-1} - S^* f^*), \quad \bar{f}_0 = 0.
\end{align}
Note that $\E[f_t] = \bar{f}_t$.
In addition, we define the degrees of freedom and its related quantity as 
\begin{align*}
\calN_\infty(\lambda) := \E_\bx[\langle K_\bx, \Sigma_\lambda^{-1}K_\bx\rangle_{\calH}] = \Tr{\Sigma \Sigma_\lambda^{-1}}, \quad
\calF_\infty(\lambda) := \sup_{\bx\in \supp{P_X}} \|\Sigma_\lambda^{-1/2}K_\bx \|_{\calH}^2.
\end{align*}

We can see that the risk admits the following bias-variance decomposition
\begin{align*}
\|Sf_t - f^*\|_{\LPiPx}^2 \leq 2 (\underbrace{\|Sf_t - S\bar{f}_t\|_{\LPiPx}^2}_{V(t)\text{, variance}} + \underbrace{\|\bar{f}_t - f^*\|_{\LPiPx}^2}_{B(t)\text{, bias}}).
\end{align*}
We upper bound the bias and variance separately.

\paragraph{Bounding the bias term $B(t)$:}
Note that by the update rule \eqref{eq:RKHS-update}, it holds that 
\begin{align*}
& S \bar{f}_t - f^* = S \bar{f}_{t-1} - f^*  - \eta S (\Sigma + \lambda I)^{-1} (\Sigma \bar{f}_{t-1} - S^* f^*) \\
\Leftrightarrow~&
S \bar{f}_t - f^*  =  (I -  \eta S  (\Sigma + \lambda I)^{-1} S^* ) (S \bar{f}_{t-1} - f^*).
\end{align*}
Therefore, unrolling the recursion gives $S \bar{f}_t - f^* = (I - \eta S (\Sigma + \lambda I)^{-1} S^* )^t  (S \bar{f}_0 - f^*) = (I - \eta S (\Sigma + \lambda I)^{-1} S^*)^t (- f^*) = - (I -  \eta S  (\Sigma + \lambda I)^{-1} S^* )^t  L^r h^*$. 
Write the spectral decomposition of $L$ as $L = \sum_{j=1}^\infty \sigma_j \phi_j \phi_j^*$ for $\phi_j \in \LPiPx$ for $\sigma_j \geq 0$.
We have $\| (I -  \eta S  (\Sigma + \lambda I)^{-1} S^* )^t  L^r h^*\|_{\LPiPx} = \sum_{j=1}^\infty (1-\eta \frac{\sigma_j}{\sigma_j + \lambda})^{2t} \sigma^{2r}_j h_j^2$,
where $h = \sum_{j=1}^\infty h_j \phi_j$.
We then apply Lemma \ref{lem:PolyFiltIneq} to obtain 
\begin{align*}
B(t) \leq \exp(-\eta t) \sum_{j: \sigma_j \geq \lambda} h^2_j
+ \left( \frac{2 r}{e} \frac{\lambda}{\eta t}\right)^{2r} \sum_{j: \sigma_j < \lambda} h^2_j
\leq C \left[ \exp(-\eta t) \vee \left(\frac{\lambda}{\eta t} \right)^{2r} \right]\|h^*\|_{\LPiPx}^2,
\end{align*}
where $C$ is a constant depending only on $r$.

\paragraph{Bounding the variance term $V(t)$:}

We now handle the variance term $V(t)$. For notational convenience, we write $A_\lambda := A + \lambda I$ for a linear operator $A$ from a Hilbert space $H$ to $H$. By the definition of $f_t$, we know
\begin{align*}
f_t & = (I - \eta (\Sigma + \lambda I)^{-1}\hat{\Sigma} )f_{t-1}  + \eta (\Sigma + \lambda I)^{-1}\hat{S}^* Y \\
& = \sum_{j=0}^{t-1} (I - \eta (\Sigma + \lambda I)^{-1}\hat{\Sigma} )^j \eta (\Sigma + \lambda I)^{-1}\hat{S}^* Y \\
& =  \Sigma_\lambda^{-1/2} 
\eta \left[ \sum_{j=0}^{t-1} (I - \eta \Sigma_\lambda^{-1/2} \hat{\Sigma}  \Sigma_\lambda^{-1/2})^j   \right] \Sigma_\lambda^{-1/2} \hat{S}^* Y =:  \Sigma_\lambda^{-1/2}  G_t  \Sigma_\lambda^{-1/2}  \hat{S}^* Y,
\end{align*}
where we defined $G_t :=  \eta %\Sigma_\lambda^{-1/2} 
\left[ \sum_{j=0}^{t-1} (I - \eta \Sigma_\lambda^{-1/2} \hat{\Sigma}  \Sigma_\lambda^{-1/2})^j   \right] %\Sigma_\lambda^{-1/2} 
$.
Accordingly, we decompose $V(t)$ as 
\begin{align*}
\|Sf_t - S\bar{f}_t\|_{\LPiPx}^2
\leq & 2 ( \underbrace{\|S (f_t -  \Sigma_\lambda^{-1/2}  G_t  \Sigma_\lambda^{-1/2} \hat{\Sigma} \bar{f}_t) \|_{\LPiPx}^2}_{(a)} \\
& ~~+ \underbrace{\|S ( \Sigma_\lambda^{-1/2}  G_t  \Sigma_\lambda^{-1/2}  \hat{\Sigma} \bar{f}_t - \bar{f}_t) \|_{\LPiPx}^2}_{(b)}).
\end{align*}
We bound $(a)$ and $(b)$ separately.

\textbf{Step 1. Bounding $(a)$.}
Decompose $(a)$ as 
\begin{align*}
&\|S (f_t -  \Sigma_\lambda^{-1/2}  G_t  \Sigma_\lambda^{-1/2}  \hat{\Sigma} \bar{f}_t) \|_{\LPiPx}^2 
= 
 \|S  \Sigma_\lambda^{-1/2}  G_t  \Sigma_\lambda^{-1/2}  (\hat{S}^* Y  -\hat{\Sigma}  \bar{f}_t) \|_{\LPiPx}^2 \\
\leq
& 
\|S  \Sigma_\lambda^{-1/2} \|^2  \| G_t  \Sigma_\lambda^{-1/2}\hat{\Sigma}_{\lambda} \Sigma_\lambda^{-1/2} \|^2 
\| \Sigma_\lambda^{1/2} \hat{\Sigma}_{\lambda}^{-1}\Sigma_\lambda^{1/2}\|^2 \|\Sigma_\lambda^{-1/2} (\hat{S}^* Y  -\hat{\Sigma}  \bar{f}_t) \|_{\calH}^2.
\end{align*}
We bound the terms in the RHS individually.

\textbf{(i)} 
$\|S \Sigma_\lambda^{-1/2} \|^2 = \|\Sigma_\lambda^{-1/2} \Sigma \Sigma_\lambda^{-1/2} \| \leq 1.$

\textbf{(ii)} Note that
$
\Sigma_\lambda^{-1/2} \hat{\Sigma}_{\lambda} \Sigma_\lambda^{-1/2}
= I - \Sigma_\lambda^{-1/2}(\Sigma - \hat{\Sigma})\Sigma_\lambda^{-1/2}
\succeq (1 - \|\Sigma_\lambda^{-1/2}(\Sigma - \hat{\Sigma})\Sigma_\lambda^{-1/2}\|) I.
$

Proposition 6 of \cite{rudi2017generalization} and its proof implies that for $\lambda \leq \|\Sigma\|$ and $0  < \delta < 1$, 
it holds that 
\begin{align}
\|\Sigma_\lambda^{-1/2}(\Sigma- \hat{\Sigma})\Sigma_\lambda^{-1/2}\|
\leq \sqrt{\frac{2\beta \calF_\infty(\lambda)}{n}} + \frac{2\beta( 1 + \calF_\infty(\lambda))}{3n}
=: 
\Xi_n,
\label{eq:MatrixBernstein}
\end{align}
with probability $1-\delta$, where $\beta = \log\LL(\frac{4\Tr{\Sigma \Sigma_\lambda^{-1}}}{\delta}\RR) = \log\LL(\frac{4 \calN_\infty(\lambda)}{\delta}\RR)$.
By Lemma \ref{lem:DoFBound}, $\beta \leq \log\LL(\frac{4 \Tr{\Sigma^{1/s}}\lambda^{-1/s} }{\delta}\RR)$ and  $\calF_\infty(\lambda) \leq \lambda^{-1}$.
Therefore, if $\lambda = o(n^{-1} \log(n))$ and $\lambda = \Omega(n^{-1/s})$,
the RHS can be smaller than $1/2$ for sufficiently large $n$, i.e.~$\Xi_n = O(\sqrt{\log(n)/(n\lambda)}) \leq 1/2$.
In this case we have,
\begin{align*}
\Sigma_\lambda^{-1/2} \hat{\Sigma}_{\lambda} \Sigma_\lambda^{-1/2}
\succeq \frac{1}{2} I. 
\end{align*}
We denote this event as $\calE_1$.

\textbf{(iii)}
Note that
\begin{align*}
G_t  \Sigma_\lambda^{-1/2} \hat{\Sigma}_{\lambda} \Sigma_\lambda^{-1/2} 
& =  \eta 
\left[ \sum_{j=0}^{t-1} (I - \eta \Sigma_\lambda^{-1/2} \hat{\Sigma}  \Sigma_\lambda^{-1/2})^j   \right] \Sigma_\lambda^{-1/2} \hat{\Sigma}_{\lambda} \Sigma_\lambda^{-1/2} \\
& =  \eta 
\left[ \sum_{j=0}^{t-1} (I - \eta \Sigma_\lambda^{-1/2} \hat{\Sigma}  \Sigma_\lambda^{-1/2})^j   \right] 
(\Sigma_\lambda^{-1/2} \hat{\Sigma} \Sigma_\lambda^{-1/2} + \lambda \Sigma_\lambda^{-1}).
\end{align*}

Thus, by Lemma \ref{lem:PolySumBound} we have
\begin{align*}
&\| G_t  \Sigma_\lambda^{-1/2} \hat{\Sigma}_{\lambda} \Sigma_\lambda^{-1/2}\| \\
\leq& 
\underbrace{\left\|\eta 
\left[ \sum_{j=0}^{t-1} (I - \eta \Sigma_\lambda^{-1/2} \hat{\Sigma}  \Sigma_\lambda^{-1/2})^j   \right] 
\Sigma_\lambda^{-1/2} \hat{\Sigma} \Sigma_\lambda^{-1/2} \right\|}_{\leq 1 \text{~~(due to Lemma \ref{lem:PolySumBound})}} 
+ 
\left\|\eta 
\left[ \sum_{j=0}^{t-1} (I - \eta \Sigma_\lambda^{-1/2} \hat{\Sigma}  \Sigma_\lambda^{-1/2})^j   \right] 
\lambda \Sigma_\lambda^{-1} \right\| \\
\leq& 
1 + \eta \sum_{j=0}^{t-1} \| (I - \eta \Sigma_\lambda^{-1/2} \hat{\Sigma}  \Sigma_\lambda^{-1/2})^j \|
\|\lambda \Sigma_\lambda^{-1} \| 
\leq 1 + \eta t.
\end{align*} 

(iv)
Note that 
\begin{align*}
\|\Sigma_\lambda^{-1/2} (\hat{S}^* Y  -\hat{\Sigma}  \bar{f}_t) \|_{\calH}^2
\leq  
2 (\|\Sigma_\lambda^{-1/2} [(\hat{S}^* Y  -\hat{\Sigma}  \bar{f}_t) - (S^* f^* - \Sigma \bar{f}_t) ]\|_{\calH}^2
+
\|\Sigma_\lambda^{-1/2}  (S^* f^* - \Sigma \bar{f}_t) \|_{\calH}^2).
\end{align*}
First we bound the first term of the RHS.
Let $\xi_i = \Sigma_\lambda^{-1/2}[K_{\bx_i} y_i - K_{\bx_i} \bar{f}_t(\bx_i) - (S^* f^* - \Sigma \bar{f}_t)]$. Then, $\{\xi_i\}_{i=1}^n$ is an i.i.d. sequence of zero-centered random variables taking value in $\calH$ and thus we have 
$$
\|\Sigma_\lambda^{-1/2} [(\hat{S}^* Y  -\hat{\Sigma}  \bar{f}_t) - (S^* f^* - \Sigma \bar{f}_t) ]\|_{\calH}^2
= \left\|\frac{1}{n} \sum_{i=1}^n \xi_i \right\|_{\calH}^2.
$$
The RHS can be bounded by using Bernstein's inequality in Hilbert space \cite{caponnetto2007optimal}.
To apply the inequality, we need to bound the variance and sup-norm of the random variable. The variance can be bounded as 
\begin{align*}
\E[\|\xi_i\|_{\calH}^2] 
& \leq \E_{(\bx,y)}\left[ \|\Sigma_\lambda^{-1/2}(K_\bx (f^*(\bx) - \bar{f}_t(\bx)) + K_\xi \epsilon)\|_{\calH}^2\right] \\
& \leq 2 \left\{ \E_{(\bx,y)}\left[ \|\Sigma_\lambda^{-1/2}(K_\bx (f^*(\bx) - \bar{f}_t(x))\|_{\calH}^2+ 
\|\Sigma_\lambda^{-1/2}(K_\bx \epsilon)\|_{\calH}^2\right]\right\} \\
& \leq 2 \left\{ \sup_{\bx \in \supp{P_X}}\|\Sigma_\lambda^{-1/2} K_\bx\|^2 \|f^*  - S \bar{f}_t\|_{\LPiPx}^2 
+ 
\sigma^2 \Tr{\Sigma_\lambda^{-1} \Sigma}\right\} \\
& \leq 2  %\left\{   B(t)+ 
\left\{ \calF_\infty(\lambda) B(t) + \sigma^2 \Tr{\Sigma_\lambda^{-1} \Sigma}  \right\} \\
& \leq 
2 \left\{ \lambda^{-1} B(t) + \sigma^2 \Tr{\Sigma_\lambda^{-1} \Sigma}  \right\},
\end{align*} 
The sup-norm can be bounded as follows. Observe that $\|\bar{f}_t\|_{\infty} \leq \|\bar{f}_t\|_{\calH}$, and thus by Lemma \ref{lemm:ftRKHSnormBound},
\begin{align*}
\|\xi_i\|_{\calH}
& \leq  2 \sup_{\bx \in \supp{P_X}}\|\Sigma_{\lambda}^{-1/2} K_\bx\|_{\calH}(\sigma + \|f^*\|_\infty + \|\bar{f}_t\|_{\infty}) \\
& \lesssim \calF_\infty^{1/2}(\lambda)(\sigma + M + (1+t\eta) \lambda^{-(1/2-r)_+}) \\
& \lesssim \lambda^{-1/2}(\sigma + M + (1+t\eta) \lambda^{-(1/2-r)_+}).
\end{align*}
Therefore, for $0 < \delta < 1$, Bernstein's inequality (see Proposition 2 of \cite{caponnetto2007optimal}) yields that 
\begin{align*}
\left\|\frac{1}{n} \sum_{i=1}^n \xi_i \right\|_{\calH}^2
\leq C \left( \sqrt{\frac{
\lambda^{-1} B(t) + \sigma^2 \Tr{\Sigma_\lambda^{-1} \Sigma} 
}{n}}  + \frac{\lambda^{-1/2}(\sigma +M + (1+t\eta) \lambda^{-(1/2-r)_+})}{n}\right)^2 \log(1/\delta)^2
\end{align*}
with probability $1-\delta$ where $C$ is a universal constant. 
We define this event as $\calE_2$. 

For the second term $\|\Sigma_\lambda^{-1/2}  (S^* f^* - \Sigma \bar{f}_t) \|_{\calH}^2$ we have
\begin{align*}
\|\Sigma_\lambda^{-1/2}  (S^* f^* - \Sigma \bar{f}_t) \|_{\calH}^2
\leq 
\|\Sigma_\lambda^{1/2}  (f^* - S f_t) \|_{\calH}^2 
=\|f^* - S \bar{f}_t \|_{\LPiPx}^2 \leq B(t). 
\end{align*}
Combining these evaluations, on the event $\calE_2$ where $P(\calE_2) \geq 1-\delta$ for $0 < \delta < 1$ we have
\begin{align*}
&  \|\Sigma_\lambda^{-1/2} (\hat{S}^* Y  -\hat{\Sigma}  \bar{f}_t) \|_{\calH}^2 \\
\overset{(i)}{\leq} & 
C \left(\sqrt{\frac{\lambda^{-1} B(t) + \sigma^2 \Tr{\Sigma_\lambda^{-1} \Sigma} }{n}}  + \frac{\lambda^{-1/2}(\sigma +M + (1+t\eta) \lambda^{-(1/2-r)_+})}{n}\right)^2 \log(1/\delta)^2
+ B(t).
\end{align*}
where we used Lemma \ref{lem:DoFBound} in (i).

\textbf{Step 2. Bounding $(b)$.}
On the event $\calE_1$, the term $(b)$ can be evaluated as 
\begin{align}
& \|S ( \Sigma_\lambda^{-1/2}  G_t  \Sigma_\lambda^{-1/2}  \hat{\Sigma} \bar{f}_t - \bar{f}_t) \|_{\LPiPx}^2 \notag\\
\leq& 
 \|\Sigma^{1/2} ( \Sigma_\lambda^{-1/2}  G_t  \Sigma_\lambda^{-1/2}  \hat{\Sigma} \bar{f}_t - \bar{f}_t) \|_{\calH}^2 \notag\\
\leq& 
 \|\Sigma^{1/2} \Sigma_\lambda^{-1/2}  (  G_t  \Sigma_\lambda^{-1/2}  \hat{\Sigma} \Sigma_\lambda^{-1/2}  - I ) \Sigma_\lambda^{1/2}  \bar{f}_t\|_{\calH}^2 \notag\\
\leq& 
 \|\Sigma^{1/2} \Sigma_\lambda^{-1/2}\|\| (  G_t  \Sigma_\lambda^{-1/2}  \hat{\Sigma} \Sigma_\lambda^{-1/2}  - I ) \Sigma_\lambda^{1/2}  \bar{f}_t\|_{\calH}^2 \notag \\ 
\leq& 
\| (  G_t  \Sigma_\lambda^{-1/2}  \hat{\Sigma} \Sigma_\lambda^{-1/2}  - I ) \Sigma_\lambda^{1/2}  \bar{f}_t\|_{\calH}^2. 
\label{eq:Gtftbound}
\end{align}
where we used Lemma \ref{lemm:ftRKHSnormBound} in the last inequality. 
The term
$\| (G_t  \Sigma_\lambda^{-1/2}  \hat{\Sigma} \Sigma_\lambda^{-1/2}  - I )  \Sigma_\lambda^{1/2}f_t \|_{\calH}$ 
can be bounded as follows.
First, note that 
\begin{align*}
(G_t  \Sigma_\lambda^{-1/2} \hat{\Sigma} \Sigma_\lambda^{-1/2} - I )\Sigma_\lambda^{1/2}
& = \left\{  \eta 
\left[ \sum_{j=0}^{t-1} (I - \eta \Sigma_\lambda^{-1/2} \hat{\Sigma}  \Sigma_\lambda^{-1/2})^j   \right] 
\Sigma_\lambda^{-1/2} \hat{\Sigma} \Sigma_\lambda^{-1/2} - I \right\}
\Sigma_\lambda^{1/2} \\
& = (I - \eta \Sigma_\lambda^{-1/2} \hat{\Sigma}  \Sigma_\lambda^{-1/2})^t \Sigma_\lambda^{1/2}.
\end{align*}
Therefore, the RHS of \eqref{eq:Gtftbound} can be further bounded by 
\begin{align}
& \|(I - \eta \Sigma_\lambda^{-1/2} \hat{\Sigma}  \Sigma_\lambda^{-1/2})^t \Sigma_\lambda^{1/2}\bar{f}_t\|_{\calH} \notag \\
= &
\|(I - \eta \Sigma_\lambda^{-1/2} \Sigma  \Sigma_\lambda^{-1/2} + \eta \Sigma_\lambda^{-1/2}( \Sigma - \hat{\Sigma}) \Sigma_\lambda^{-1/2})^t \Sigma_\lambda^{1/2}\bar{f}_t\|_{\calH} \notag\\
= & 
\|\sum_{k=0}^{t-1} (I - \eta \Sigma_\lambda^{-1/2} \hat{\Sigma}\Sigma_\lambda^{-1/2})^{k} 
(\eta \Sigma_\lambda^{-1/2}( \Sigma - \hat{\Sigma}) \Sigma_\lambda^{-1/2})
(I - \eta \Sigma_\lambda^{-1} \Sigma)^{t - k - 1} 
 \Sigma_\lambda^{1/2}\bar{f}_t
- (I - \eta \Sigma_\lambda^{-1} \Sigma  )^t \Sigma_\lambda^{1/2}\bar{f}_t \|_{\calH} \notag\\
%%%%%%%%%%%%%%%%%%%%%%
\overset{(i)}{\leq} 
&
\|(I - \eta \Sigma_\lambda^{-1} \Sigma  )^t \Sigma_\lambda^{1/2}\bar{f}_t\|_{\calH} \notag \\
&+  \eta 
 \sum_{k=0}^{t-1}
\| (I - \eta \Sigma_\lambda^{-1/2} \hat{\Sigma}\Sigma_\lambda^{-1/2})^{k}  \Sigma_\lambda^{-1/2}( \Sigma - \hat{\Sigma}) \Sigma_\lambda^{-1/2+r} 
(I - \eta \Sigma_\lambda^{-1} \Sigma)^{t - k - 1} 
\Sigma_\lambda^{1/2-r} \bar{f}_t \|_{\calH}  \notag \\
\leq 
&
\|(I - \eta \Sigma_\lambda^{-1} \Sigma  )^t \Sigma_\lambda^{1/2}\bar{f}_t\|_{\calH}
+ t \eta \|\Sigma_\lambda^{-1/2}( \Sigma - \hat{\Sigma}) \Sigma_\lambda^{-1/2+r}  \|\| \Sigma_\lambda^{1/2-r} \bar{f}_t \|_{\calH} \notag\\
=
&
\|(I - \eta \Sigma_\lambda^{-1} \Sigma  )^t \Sigma_\lambda^{r}\|  \|\Sigma_\lambda^{1/2-r} \bar{f}_t\|_{\calH}
+  t \eta \|\Sigma_\lambda^{-1/2}( \Sigma - \hat{\Sigma})  \Sigma_\lambda^{-1/2+r} \| \|\Sigma_\lambda^{1/2-r} \bar{f}_t \|_{\calH} \notag\\
\lesssim
&
\|(I - \eta \Sigma_\lambda^{-1} \Sigma  )^t \Sigma_\lambda^{r}\| 
+  t \eta  \|\Sigma_\lambda^{-1/2}( \Sigma - \hat{\Sigma})  \Sigma_\lambda^{-1/2+r} \|
 (1 + t \eta ) \|h^*\|_{\LPiPx},
\label{eq:ftinterateSigmahstarbound}
\end{align}
where (i) is due to exchangeability of $\Sigma_\lambda$ and $\Sigma$. By Lemma \ref{lem:PolyFiltIneq}, for the RHS we have
\begin{align*}
\|(I - \eta \Sigma_\lambda^{-1} \Sigma  )^t \Sigma_\lambda^{r}\| 
\leq \exp\left(- \eta t\right/2)  \vee \left( \frac{1}{e} \frac{\lambda}{\eta t}\right)^{r}.
\end{align*}

Next,
as in the \eqref{eq:MatrixBernstein}, by applying the Bernstein inequality for asymmetric operators (Corollary 3.1 of \cite{MINSKER2017111} with the argument in its Section 3.2), it holds that 
\begin{align*}
& \|\Sigma_\lambda^{-1/2}( \Sigma - \hat{\Sigma})  \Sigma_\lambda^{-1/2+r} \| \\ 
\leq&
C' \left( \sqrt{\frac{\beta' 
C_\mu^2 (2^{2r - \mu} \vee \lambda^{2 r-\mu})   \calN_\infty(\lambda)
}{n}} 
+ 
\frac{\beta'( (1+\lambda)^r +C_\mu^2\lambda^{-\mu/2}(2^{2r - \mu}\vee  \lambda^{r-\mu/2})}{n}\right)
=: \Xi'_n,
\end{align*}
with probability $1-\delta$, where $C'$ is a universal constant
and 
$\beta' \leq \log\LL(\frac{28C_\mu^2 (2^{2r-\mu} \vee \lambda^{-\mu + 2r}) \Tr{\Sigma^{1/s}}\lambda^{-1/s} }{\delta}\RR)$.
We also used the following bounds on the sup-norm and the second order moments:
\begin{align*}
(\text{sup-norm})~~~~ & \|\Sigma_\lambda^{-1/2}(K_{\bx}K_{\bx}^* - \Sigma) \Sigma_\lambda^{-1/2+r} \| \\
&\leq 
\|\Sigma_\lambda^{-1/2}K_{\bx}K_{\bx}^*\Sigma_\lambda^{-1/2+r}\|+ \|\Sigma_\lambda^r\| \\
& \leq  
\|\Sigma_\lambda^{-\mu/2}\Sigma_\lambda^{\mu/2-1/2}K_{\bx}K_{\bx}^*\Sigma_\lambda^{-1/2+\mu/2}
\Sigma_\lambda^{r - \mu/2} \|+ \|\Sigma_\lambda^r\| \\
& \leq    C_\mu^2 \lambda^{- \mu/2} (2^{r - \mu/2} \vee \lambda^{r - \mu/2}) + (1 + \lambda)^r~~~(\text{a.s.}), \\
(\text{2nd order moment 1}) ~~~~ &  \| \E_\bx[ \Sigma_\lambda^{-1/2}(K_{\bx}K_{\bx}^* - \Sigma)\Sigma_\lambda^{-1+2r}  (K_{\bx}K_{\bx}^* - \Sigma) \Sigma_\lambda^{-1/2}] \| \\
& \leq
\| \Sigma_\lambda^{-1/2}\Sigma \Sigma_{\lambda}^{-1/2}\| \sup_{\bx \in \mathrm{supp}(P_X)}[K_\bx^* \Sigma_\lambda^{-1/2+\mu/2} \Sigma_\lambda^{-\mu+2r}\Sigma_\lambda^{-1/2+\mu/2} K_\bx]  \\
& \leq
C_\mu^2 (2^{2r - \mu} \vee \lambda^{2r-\mu}),  \\
(\text{2nd order moment 2}) ~~~~ & \| \E_\bx[ \Sigma_\lambda^{-1/2+r} (K_{\bx}K_{\bx}^* - \Sigma) \Sigma_\lambda^{-1/2}\Sigma_\lambda^{-1/2}(K_{\bx}K_{\bx}^* - \Sigma)\Sigma_\lambda^{-1/2+r}  ] \|
 \\
& \leq
\| \E_\bx[ \Sigma_\lambda^{-1/2+r} K_\bx K_\bx^* \Sigma_\lambda^{-1} K_\bx K_\bx^* \Sigma_\lambda^{-1/2+r}]  \| \\
& \leq
C_\mu^2 (2^{2r - \mu} \vee \lambda^{2 r-\mu})  \E_\bx[  K_\bx^* \Sigma_\lambda^{-1} K_\bx ]  \\
& = 
C_\mu^2 (2^{2r - \mu} \vee \lambda^{2 r-\mu})   \Tr{\Sigma \Sigma_\lambda^{-1}}  \\
& =
C_\mu^2 (2^{2r - \mu} \vee \lambda^{2 r-\mu})   \calN_\infty(\lambda).
\end{align*}
We define this event as $\calE_3$.
Therefore, the RHS of \eqref{eq:ftinterateSigmahstarbound}
can be further bounded by 
\begin{align*}
& [ \|(I - \eta \Sigma_\lambda^{-1} \Sigma  )^t \Sigma_\lambda^{r}\| 
+ C t \eta  \|\Sigma_\lambda^{-1/2}( \Sigma - \hat{\Sigma})  \Sigma_\lambda^{-1/2+r} \|]
 (1 + t \eta ) \|h^*\|_{\LPiPx} \\
\leq & 
\left[\exp\left(- \eta t\right/2)  \vee \left( \frac{1}{e} \frac{\lambda}{\eta t}\right)^{r} 
+  t \eta \Xi'_n \right](1 + t\eta) \|h^*\|_{\LPiPx}.
\end{align*}

Finally, note that when $\lambda = \lambda^*$ and $2r \geq \mu$, 
\begin{align*} 
\Xi_n'^{2}
= \tilde{O}\LL(\frac{{\lambda^*}^{2r-\mu -1/s}}{n}
 + \frac{{\lambda^*}^{2(r-\mu)}}{n^2} \RR) 
\leq \tilde{O}(n^{-\frac{s(4r-\mu)}{2rs + 1}} + n^{-\frac{s(4r-2\mu) + 2}{2rs + 1}})
\leq \tilde{O}(n^{-\frac{2rs}{2rs + 1}}). 
\end{align*}

\textbf{Step 3.}
Combining the calculations in Step 1 and 2 leads to the desired result.
\end{proof}

\subsubsection{Auxiliary lemmas}
\label{subsec:proof_RKHS_auxiliary}
\begin{lemm}\label{lem:PolyFiltIneq}
For $t \in \mathbb{N}$, $0 < \eta < 1$, $0 < \sigma \leq 1$ and $0 \leq \lambda$, it holds that
\begin{align*}
\left(1-\eta \frac{\sigma}{\sigma + \lambda}\right)^t \sigma^r 
\leq 
\begin{cases}
\exp\left(- \eta t\right/2) 
& (\sigma \geq \lambda) \\
\left( \frac{2 r}{e} \frac{\lambda}{\eta t}\right)^r &  (\sigma < \lambda)
\end{cases}.
\end{align*}
\end{lemm}
\begin{proof}
When $\sigma \geq \lambda$, we have 
$$
\left(1-\eta \frac{\sigma}{\sigma + \lambda}\right)^t \sigma^r 
\leq 
\left(1-\eta \frac{\sigma}{2 \sigma}\right)^t \sigma^r 
= 
\left(1-\eta/2 \right)^t \sigma^r 
\leq \exp(- t \eta/2)  \sigma^r \leq \exp(- t \eta/2) 
$$
due to $\sigma \leq 1$.
On the other hand, note that
\begin{align*}
\left(1-\eta \frac{\sigma}{\sigma + \lambda}\right)^t \sigma^r 
& 
\leq \exp\left(-\eta t \frac{\sigma}{\sigma+ \lambda}\right) \times\left(\frac{ \sigma \eta t}{\sigma + \lambda}\right)^r 
\left(\frac{\sigma+ \lambda}{\eta t}\right)^r  \\ 
& \leq \sup_{x > 0} \exp(-x)x^r \left(\frac{\sigma+ \lambda}{\eta t}\right)^r 
\leq \left(\frac{ (\sigma + \lambda) r }{\eta t e}\right)^r,
\end{align*}
where we used $\sup_{x > 0} \exp(-x)x^r = (r/e)^r$.
\end{proof}

\begin{lemm}\label{lem:PolySumBound}
For $t = \mathbb{N}$, $0 < \eta $ and $0 \leq \sigma$ such that $\eta \sigma < 1$, it holds that $\eta \sum_{j=0}^{t-1} \left(1-\eta \sigma \right)^j \sigma 
\leq 1$.
\end{lemm}

\begin{proof}
If $\sigma=0$, then the statement is obvious.
Assume that $\sigma > 0$, then 
\begin{align*}
\sum_{j=0}^{t-1} \left(1-\eta \sigma \right)^j \sigma
= \frac{1 - (1-\eta \sigma)^t}{1-(1-\eta \sigma)} \sigma
= \frac{1}{\eta} [1 - (1-\eta \sigma)^t] \leq \eta^{-1}.
\end{align*}
This yields the desired claim.
\end{proof}

\begin{lemm}\label{lemm:ftRKHSnormBound} 
Under (A5-7), for any $0 < \lambda < 1$ and $q \leq r$, it holds that 
\begin{align*}
\| \Sigma_\lambda^{-s} \bar{f}_t\|_{\calH} \lesssim 
 (1 + \lambda^{-(1/2+(q-r))_+}+\lambda t \eta \lambda^{-(3/2+(q-r))_+})  \|h^*\|_{\LPiPx}.
\end{align*}
\end{lemm}

\begin{proof}
Recall that 
\begin{align*}
\bar{f}_t = (I - \eta (\Sigma + \lambda I)^{-1}\Sigma )\bar{f}_{t-1}  + \eta (\Sigma + \lambda I)^{-1}S^* f^* 
= \sum_{j=0}^{t-1} (I - \eta (\Sigma + \lambda I)^{-1}\Sigma )^j \eta (\Sigma + \lambda I)^{-1}S^* f^*.
\end{align*}

Therefore, we obtain the following
\begin{align*}
&\| \Sigma_\lambda^{-q} \bar{f}_t\|_{\calH}
 = \eta \| \sum_{j=0}^{t-1} (I - \eta \Sigma_\lambda^{-1}\Sigma )^j  \Sigma_\lambda^{-1-q} S^* L^r h^*\|_{\calH} \\
=& \eta \| \sum_{j=0}^{t-1} (I - \eta \Sigma_\lambda^{-1}\Sigma )^j  \Sigma_\lambda^{-1} (\Sigma + \lambda I)\Sigma_\lambda^{-q-1} S^* L^r h^*\|_{\calH} \\
\leq& 
\eta \| \sum_{j=0}^{t-1} (I - \eta \Sigma_\lambda^{-1}\Sigma )^j  \Sigma_\lambda^{-1} \Sigma \Sigma_\lambda^{-q-1} S^* L^r h^*\|_{\calH} +
\lambda \eta \| \sum_{j=0}^{t-1} (I - \eta \Sigma_\lambda^{-1}\Sigma )^j  \Sigma_\lambda^{-1}  \Sigma_\lambda^{-q-1} S^* L^r h^*\|_{\calH} \\
\leq& 
\eta \| \sum_{j=0}^{t-1} (I - \eta \Sigma_\lambda^{-1}\Sigma )^j  \Sigma_\lambda^{-1} \Sigma \| \|\Sigma_\lambda^{-q-1} S^* L^r h^*\|_{\calH} 
+
\lambda \eta \| \sum_{j=0}^{t-1} (I - \eta \Sigma_\lambda^{-1}\Sigma )^j  \Sigma_\lambda^{-1}  \Sigma_\lambda^{-q-1} S^* L^r h^*\|_{\calH} \\
\leq& 
\|\Sigma_\lambda^{-q-1} S^* L^r h^*\|_{\calH} +
\lambda t \eta \|   \Sigma_\lambda^{-1}  \Sigma_\lambda^{-q-1} S^* L^r h^*\|_{\calH} \\
\leq& 
\| S^* L_\lambda^{-q-1+r} h^*\|_{\calH} 
+
\lambda t \eta \| S^*  L_\lambda^{-q-2+r} h^*\|_{\calH} \\
\leq& 
\sqrt{ \langle h^*,  L_\lambda^{-q-1+r} S S^* L_\lambda^{-q-1+r} h^*\rangle_{\LPiPx} }
+
\lambda t \eta \sqrt{ \langle h^*,   L_\lambda^{-q-2+r} SS^*   L_\lambda^{-q-2+r} h^*\rangle_{\LPiPx}} \\ \\
=&
\sqrt{ \langle h^*,  L_\lambda^{-q-1+r} L L_\lambda^{-q-1+r} h^*\rangle_{\LPiPx} }
+
\lambda t \eta \sqrt{ \langle h^*,   L_\lambda^{-q-2+r} L  L_\lambda^{-q-2+r} h^*\rangle_{\LPiPx}} \\
\leq&
(\lambda^{-1/2-(q-r)} + \lambda t \eta \lambda^{-3/2-(q-r)} ) \|h^*\|_{\LPiPx}
\leq
(1+ t \eta) \lambda^{-1/2-(q-r)}  \|h^*\|_{\LPiPx}.
\end{align*}
\end{proof}

\begin{lemm}\label{lem:DoFBound}
Under (A5-7) and for $\lambda\in(0,1)$, it holds that $\calN_\infty(\lambda) \leq \Tr{\Sigma^{1/s}} \lambda^{-1/s}$, and
$\calF_\infty(\lambda) \leq 1/\lambda$.
\end{lemm}
\begin{proof}
For the first inequality, we have
\begin{align*}
\calN_\infty(\lambda) 
=& \Tr{\Sigma \Sigma^{-1}_\lambda} = \Tr{\Sigma^{1/s}\Sigma^{1-1/s}\Sigma^{-(1-1/s)}_\lambda\Sigma^{-1/s}_\lambda} \\
\leq& \Tr{\Sigma^{1/s}\Sigma^{1-1/s}\Sigma^{-(1-1/s)}_\lambda} \lambda^{-1/s} 
\leq \Tr{\Sigma^{1/s}}\lambda^{-1/s}.
\end{align*}
As for the second inequality, note that 
\begin{align*}
\calF_\infty(\lambda) = \sup_\bx \langle K_\bx, \Sigma_\lambda^{-1}K_\bx\rangle_{\calH}
\leq \sup_\bx \lambda^{-1}\langle K_\bx,K_\bx \rangle_{\calH} 
\leq \lambda^{-1} \sup_\bx k(\bx,\bx) \leq \lambda^{-1}.
\end{align*}
  
\end{proof}

\subsection{Proof of Proposition~\ref{prop:approximate_fisher}}
\label{subsec:proof_approximate_fisher}

\begin{proof}
Part (a) is a simple combination of \cite[Theorem 2]{bai2008limit} and assumption (A3), which implies $\|\bSigma_\bX\|_2$ and $\|\bSigma_\bX^{-1}\|_2$ are both finite. For part (b), the substitution error for the variance term (ignoring the scalar $\sigma^2$) can be bounded as
\begin{align*}
    |V^* - \hat{V}| 
=&
    \left|\Tr{\boldF^{-1}\bX^\top(\bX \boldF^{-1}\bX^\top)^{-2}\bX \boldF^{-1}\bSigma_\bX}
    - \Tr{\hat{\boldF}^{-1}\bX^\top(\bX \hat{\boldF}^{-1}\bX^\top)^{-2}\bX \hat{\boldF}^{-1}\bSigma_\bX}\right| \\
\overset{(i)}{\le}&
    O(1)\norm{\boldF^{-1}-\hat{\boldF}^{-1}}_2\left(\Tr{\bX^\top\boldS^{-2}\bX \boldF\bSigma_\bX} + \sqrt{d}\norm{\bX^\top\boldS^{-2}\bX}_2\norm{\bSigma_\bX\hat{\boldF}^{-1}}_F\right) \\
&+
    \Tr{\bX\hat{\boldF}^{-1}\bSigma_\bX\boldF^{-1}\bX^\top}\norm{\boldS^{-2}-\hat{\boldS}^{-2}}_2
\overset{(ii)}{=}
    O(\epsilon).
\end{align*}
where we defined $\boldS=\bX\boldF^{-1}\bX^\top$ and $\hat{\boldS}=\bX\hat{\boldF}^{-1}\bX^\top$ in (i) and applied $\Tr{\boldA\boldB}\le\lambda_{\max}(\boldA+\boldA^\top)\Tr{\boldB}$ for positive semi-definite $\boldB$, as well as $\Tr{\boldA\boldB}\le\sqrt{d}\norm{\boldA}_2\norm{\boldB}_F$, and (ii) is due to (A3), $\psi>1$, \cite[Theorem 4.1]{wedin1973perturbation} and the following estimate,
\begin{align*}
    n_u^2\norm{\boldS^{-2}-\hat{\boldS}^{-2}}_2 
\le&
    \norm{n_u\boldS^{-1}-n_u\hat{\boldS}^{-1}}_2\left(n_u\norm{\boldS^{-1}}_2 + n_u\norm{\hat{\boldS}^{-1}}_2\right) \\
\overset{(i)}{=}&
    O(1)\norm{n_u\boldS^{-1}}_2\norm{n_u\hat{\boldS}^{-1}}_2\norm{\boldS/n_u - \hat{\boldS}/n_u}_2
\overset{(ii)}{=} O(\epsilon),
\end{align*}
where (i) again follows from \cite[Theorem 4.1]{wedin1973perturbation}, and (ii) is due to (A1)(A3) and $\psi>1$ (which implies $\|n_u\boldS^{-1}\|_2$ and $\|n_u\hat{\boldS}^{-1}\|_2$ are bounded a.s.). 
Finally, from part (a) we know that $\psi=\Theta(\epsilon^{-2})$ suffices to achieve $\epsilon$-accurate approximation of $\boldF$ in spectral norm. 
The substitution error for the bias term can be derived from similar calculation, the details of which we omit.
\end{proof}

\subsection{Proof of Corollary~\ref{coro:source_condition_special_case}}
\label{subsec:proof_special_case}

\begin{proof}
Note that in this setting $\upsilon_x$ takes value of $\frac{2}{1+\kappa}$ and $\frac{2\kappa}{1+\kappa}$ with probability 1/2 each. From \eqref{eq:bias_NGD} in the proof of Proposition~\ref{prop:source_condition} one can easily verify that for NGD,
\[
    B(\hbt_{\boldF^{-1}}) \to \frac{2^{r}(1+\kappa^{1+r})}{(1+\kappa)^{1+r}}\left(1 - \frac{1}{\gamma}\right). 
\]

For GD, the bias formula \eqref{eq:bias_GD} can be simplified as
\begin{align*}
    B(\hbt_\bI) 
\to& 
    \frac{1}{\gamma}\cdot\left(\frac{\left(\frac{2}{1+\kappa}\right)^{r}}{(1+\kappa+2 m_1)^2} + \frac{\kappa\left(\frac{2\kappa}{1+\kappa}\right)^{r}}{(1+\kappa+2\kappa m_1)^2}\right)\cdot\left(\frac{m_1}{(1 + \kappa + 2 m_1)^2} + \frac{\kappa m_1}{( 1 + \kappa + 2\kappa m_1)^2}\right)^{-1},  
\end{align*}
where $m_1$ is the Stieltjes transform defined after Equation~\eqref{eq:bias_GD}.
From standard numerical calculation one can show that when $\gamma>1$, $\kappa\ge 1$,  
\begin{align*} 
    m_1 
= 
    \frac{(\kappa+1) \sqrt{\gamma^2 (\kappa + 1)^2 + 4 (1 - \gamma)(\kappa - 1)^2}+ (2 - \gamma)(\kappa+1)^2}{8 (\gamma - 1) \kappa}. 
\end{align*}

Setting $B(\hbt_\bI) = B(\hbt_{\boldF^{-1}})$ and solve for $r$, we have
\begin{align*}
    r^* = -\ln{c_{\kappa,\gamma}} \big/ \ln{\kappa}, \quad c_{\kappa,\gamma} = \frac{c_4 - c_2}{c_1 - c_3},
\numberthis
\label{eq:transition_r} 
\end{align*}
where
\begin{align*}
    c_1 &= \left(1-\frac{1}{\gamma}\right)\frac{1}{\kappa+1}, \quad
    c_2 = \left(1-\frac{1}{\gamma}\right)\frac{\kappa}{\kappa+1}, \\
    c_3 &= \frac{1}{\gamma}\cdot\frac{(1+\kappa+2\kappa m_1)^2}{m_1 (1+\kappa+2\kappa m_1)^2 + \kappa m_1 (1+\kappa+2 m_1)^2}, \\
    c_4 &= \frac{1}{\gamma}\cdot\frac{\kappa(1+\kappa+2 m_1)^2}{m_1 (1+\kappa+2\kappa m_1)^2 + \kappa m_1 (1+\kappa+2 m_1)^2}.
\end{align*}

Hence, Proposition~\ref{prop:source_condition} (from which we know $r^*\in (-1,0)$) and the uniqueness of \eqref{eq:transition_r} implies that when $r\ge r^*$, $B(\hbt_\bI)\le B(\hbt_{\boldF^{-1}})$, and vice versa. 
Finally, observe that in the special case of $\gamma=2$, $m_1 = \frac{\kappa+1}{2\sqrt{\kappa}}$. Therefore, one can check that constants in \eqref{eq:transition_r} simplify to
\begin{align*}
    c_1 - c_3 = \frac{1 - \sqrt{\kappa}}{2(\kappa+1)}, \quad
    c_2 - c_4 = \frac{\sqrt{\kappa}(\sqrt{\kappa} - 1)}{2(\kappa + 1)},
\end{align*}
which implies that $r^* = -1/2$.

\end{proof}

\bigskip

}
% \newpage
\section{Experiment Setup}
\label{sec:experiment_setup}
{

\subsection{Processing the Datasets}\label{app:dataset}
To obtain extra unlabeled data to estimate the Fisher, we zero pad pixels on the boarders of each image before randomly cropping; a random horizontal flip is also applied for CIFAR10 images~\cite{krizhevsky2009learning}. 
We preprocess all images by dividing pixel values by $255$ before centering them to be located within $[-0.5, 0.5]$ with the subtraction by $1/2$. 
For experiments on CIFAR10, we downsample the original images using a max pooling layer with kernel size 2 and stride 2. 
 
\subsection{Setup and Implementation for Optimizers}\label{app:optimizers}
In all settings, GD uses a learning rate of $0.01$ that is exponentially decayed every 1k updates with the parameter value $0.999$. For NGD, we use a fixed learning rate of $0.03$. Since inverting a parameter-by-parameter-sized Fisher estimate per iteration would be costly, we adopt the Hessian free approach~\cite{martens2010deep} which computes approximate matrix-inverse-vector products using the conjugate gradient (CG) method~\cite{nocedal2006numerical, boyd2004convex}. 
For each approximate inversion, we run CG for $200$ iterations starting from the solution returned by the previous CG run. 
The precise number of CG iterations and the initialization heuristic roughly follow~\cite{martens2012training}.
For the first run of CG, we initialize the vector from a standard Gaussian, and run CG for 5k iterations.
To ensure invertibility, we apply a very small amount of damping ($0.00001$) in most scenarios.
For geometric interpolation experiments between GD and NGD, we use the singular value decomposition to compute the minus $\alpha$ power of the Fisher, as CG is not applicable in this scenario.

\subsection{Other Details}\label{app:others}
For experiments in the label noise and misspecification sections, we pretrain the teacher using the Adam optimizer~\cite{kingma2014adam} with its default hyperparameters and a learning rate of $0.001$. 

For experiments in the misalignment section, we downsample all images twice using max pooling with kernel size 2 and stride 2. Moreover, only for experiments in this section, we implement natural gradient descent by exactly computing the Fisher on a large batch of unlabeled data and inverting the matrix by calling PyTorch's \texttt{torch.inverse} before right multiplying the gradient.

}

\end{document}